\documentclass{article}

\usepackage{fullpage}
\usepackage[utf8]{inputenc}
\usepackage[T1]{fontenc}
\usepackage{textcomp}
\usepackage{bm}
\usepackage{dsfont}
\usepackage{newtxtext}

\usepackage{amsmath}
\usepackage{amssymb}
\usepackage{amsthm}
\usepackage{amsfonts}       

\usepackage{mathtools}
\DeclarePairedDelimiter{\prn}{(}{)}
\DeclarePairedDelimiter{\set}{\{}{\}}
\DeclarePairedDelimiterX{\Set}[2]{\{}{\}}{\,{#1}\,:\,{#2}\,}
\DeclarePairedDelimiter{\abs}{|}{|}
\DeclarePairedDelimiter{\norm}{\|}{\|}
\DeclarePairedDelimiter{\inpr}{\langle}{\rangle}

\DeclarePairedDelimiter{\brc}{[}{]}
\DeclarePairedDelimiterX{\Brc}[2]{[}{]}{\,{#1}\,\middle|\,{#2}\,}

\DeclareFontFamily{U}{mathx}{}
\DeclareFontShape{U}{mathx}{m}{n}{<-> mathx10}{}
\DeclareSymbolFont{mathx}{U}{mathx}{m}{n}
\DeclareMathAccent{\widecheck}{0}{mathx}{"71}

\usepackage{graphicx}
\usepackage[svgnames]{xcolor}

\usepackage{booktabs}
\usepackage{multirow}
\usepackage[framemethod=default]{mdframed}

\usepackage{etoolbox}
\let\etbforlistloop\forlistloop
\makeatletter
\let\blx@noerroretextools\@empty
\makeatother
\usepackage[date=year,eprint=false,doi=false,isbn=false,backend=biber,giveninits=true,maxcitenames=2,maxbibnames=99,natbib=true,url=false,sorting=ynt,style=apa,citestyle=authoryear-comp,backref=true,uniquelist=false]{biblatex}

\DeclareSourcemap{
  \maps[datatype=bibtex]{
    \map{
      \step[fieldset=editor, null]
      \step[fieldset=series, null]
      \step[fieldset=language, null]
      \step[fieldset=address, null]
      \step[fieldset=location, null]
      \step[fieldset=month, null]
      \step[fieldset=annote, null]
    }
  }
}
\DefineBibliographyStrings{english}{%
  backrefpage  = {cited on page},
  backrefpages = {cited on pages},
}
\DefineBibliographyStrings{american}{%
  backrefpage  = {cited on page},
  backrefpages = {cited on pages},
}
\DefineBibliographyStrings{american-apa}{%
  backrefpage  = {cited on page},
  backrefpages = {cited on pages},
}
\renewbibmacro*{pageref}{%
  \iflistundef{pageref}
    {}
    {\printtext[parens]{%
       \midsentence%
       \ifnumgreater{\value{pageref}}{1}
         {\bibstring{backrefpages}\ppspace}
         {\bibstring{backrefpage}\ppspace}%
       \printlist[pageref][-\value{listtotal}]{pageref}}}}

\usepackage{algorithm}
\usepackage{ifthen}

\usepackage[colorlinks=true,citecolor=Navy,linkcolor=Maroon,urlcolor=Orchid,bookmarksnumbered,hypertexnames=false,pdfdisplaydoctitle,pdfusetitle,unicode]{hyperref}

\usepackage[noend]{algpseudocode}

\algnewcommand{\algorithmicinput}{\textbf{Input:}}
\algnewcommand{\Input}{\item[\algorithmicinput]}
\algnewcommand{\algorithmicoutput}{\textbf{Input:}}
\algnewcommand{\Output}{\item[\algorithmicoutput]}
\algnewcommand{\Break}{\textbf{break}}
\makeatletter
\renewcommand{\ALG@step}{%
  \refstepcounter{ALG@line}%
  \addtocounter{ALG@rem}{1}%
  \ifthenelse{\equal{\arabic{ALG@rem}}{\ALG@numberfreq}}%
    {\setcounter{ALG@rem}{0}\alglinenumber{\arabic{ALG@line}}}%
    {}%
}
\makeatother

\usepackage[algo2e,ruled,vlined,noend]{algorithm2e}
\SetKwInput{Input}{Input}
\SetKwInput{Output}{Output}
\SetKw{Break}{break}

\usepackage{url}
\usepackage{upref}
\usepackage[capitalize,noabbrev]{cleveref}

\crefname{step}{Step}{Steps}
\Crefname{step}{Step}{Steps}
\crefname{appendix}{Appendix}{Appendices}
\Crefname{appendix}{Appendix}{Appendices}
\crefname{assumption}{Assumption}{Assumptions}
\Crefname{assumption}{Assumption}{Assumptions}
\crefname{algorithm}{Algorithm}{Algorithms}
\Crefname{algorithm}{Algorithm}{Algorithms}

\usepackage{autonum}
\let\forlistloop\etbforlistloop

\makeatletter
\pretocmd{\appendix}{%
  \crefalias{section}{appendix}%
}{}{}
\makeatother

\usepackage{nicefrac}       
\usepackage{microtype}      
\usepackage{xspace}
\usepackage{subcaption}
\usepackage{wrapfig}
\usepackage{siunitx}
\usepackage{thmtools}
\usepackage{thm-restate}
\allowdisplaybreaks
\usepackage{enumitem}

\usepackage[disable,color={red!100!green!33},colorinlistoftodos,prependcaption,textsize=small]{todonotes}

\newcommand{\rr}[1]{}

\usepackage{pgfplots}
\usepackage{tikz, tikz-3dplot}

\usetikzlibrary{positioning}
\usetikzlibrary{arrows}

\newtheorem{theorem}{Theorem}[section]
\newtheorem{lemma}[theorem]{Lemma}
\newtheorem{proposition}[theorem]{Proposition}
\newtheorem{corollary}[theorem]{Corollary}
\newtheorem{assumption}[theorem]{Assumption}
\theoremstyle{definition}
\newtheorem{definition}[theorem]{Definition}

\newtheorem{remark}[theorem]{Remark}
\newcommand{\savednormallabel}{}
\AtBeginDocument{\global\let\savednormallabel\label}
\newcommand{\restorenormallabel}{\global\let\label\savednormallabel}
\makeatletter
\newcommand{\wraprestatablewithlabelrestore}[1]{%
  \csletcs{restatable@orig@#1}{#1}%
  \csdef{#1}{%
    \@ifstar{%
      \csuse{restatable@orig@#1}*%
      \restorenormallabel
    }{%
      \csuse{restatable@orig@#1}%
    }%
  }%
}
\newcommand{\mainonlyhypertarget}[2]{%
  \ifcsname ifthmt@thisistheone\endcsname
    \ifthmt@thisistheone
      \hypertarget{#1}{#2}%
    \else
      #2%
    \fi
  \else
    \hypertarget{#1}{#2}%
  \fi
}
\newcommand{\decomptarget}[2]{\text{\normalfont\scriptsize\mainonlyhypertarget{#1}{\textcolor{DarkGreen}{#2}}}}
\newcommand{\decomplink}[2]{\hyperlink{#1}{\textcolor{DarkGreen}{#2}}}
\makeatother

\newcommand{\Z}{\mathbb{Z}}
\newcommand{\R}{\mathbb{R}}

\newcommand{\ones}{\mathbf{1}}

\newcommand{\E}{\mathop{\mathbb{E}}}

\newcommand{\OPT}{\mathrm{OPT}}
\newcommand{\Reg}{\mathcal{R}}

\title{From Average Sensitivity to Small-Loss Regret Bounds\\ under Random-Order Model}

\author{%
Shinsaku Sakaue%
\thanks{CyberAgent, Tokyo, Japan. National Institute of Informatics, Tokyo, Japan. Center for Advanced Intelligence Project, RIKEN, Tokyo, Japan. Email: shinsaku.sakaue@gmail.com.}%
\and
Yuichi Yoshida%
\thanks{National Institute of Informatics, Tokyo, Japan. Email: yyoshida@nii.ac.jp.}%
}

\date{}

\begin{document}
\maketitle
\begin{abstract}
We study online learning in the random-order model, where the multiset of loss functions is chosen adversarially but revealed in a uniformly random order. By extending the \emph{batch-to-online} transformation of Dong and Yoshida (2023), we show that if an offline algorithm enjoys a $(1+\varepsilon)$-approximation guarantee, an \emph{average sensitivity} bound controlled by a function $\varphi(\varepsilon)$, and stability with respect to $\varepsilon$, then we can obtain a \emph{small-loss} regret bound typically of order $\tilde O(\varphi^{\star}(\mathrm{OPT}_T))$, where $\varphi^{\star}$ is the concave conjugate of $\varphi$, $\mathrm{OPT}_T$ is the offline optimum over $T$ rounds, and $\tilde O$ hides polylogarithmic factors in $T$.
Our result refines their original $(1+\varepsilon)$-approximate regret guarantee and applies to a broad class of problems, including online $k$-means clustering and online low-rank approximation.
We further apply our approach to online submodular function minimization using $(1\pm\varepsilon)$-cut sparsifiers of submodular hypergraphs, obtaining a small-loss regret bound of $\smash{\tilde O(n^3 + n^{3/4}\mathrm{OPT}_T^{3/4})}$, where $n$ is the ground-set size; we also demonstrate its applicability to online $\ell_1$ regression.
Our work sheds light on the power of sparsification and related algorithmic techniques in achieving small-loss regret bounds in the random-order model, without requiring structural assumptions on loss functions, such as linearity or smoothness.\looseness=-1
\end{abstract}

\section{Introduction}
The online learning literature has been shaped by two canonical regimes: the stochastic and adversarial regimes. However, the i.i.d.\ assumption of the stochastic regime is often restrictive, while guarantees in the adversarial regime can be overly conservative. The \emph{random-order model} provides a compelling bridge between these extremes by allowing an adversary to pick the multiset of inputs while requiring that their order be drawn uniformly at random. This recovers the stochastic regime when the inputs themselves are generated i.i.d., while still allowing us to leverage random ordering to sidestep the pessimism of fully adversarial sequences, yielding better guarantees in many settings. The random-order model has a long history in competitive analysis of online algorithms 
\citep{Dynkin1963-sj,Karp1990-tw,Babaioff2018-hb,Gupta2021-zv} and has recently received increasing attention in regret analysis \citep{Garber2020-oh,Sherman2021-wb,Bernasconi2025-qb}.

In the stochastic setting, Follow-the-Leader achieves logarithmic regret bounds, but it fails in the adversarial setting (see, e.g., \citet[Sections~3.2 and~4.3]{Cesa-Bianchi2006-oa}).
Competing under adversarially selected inputs requires stability, as reflected in the design of Follow-the-Regularized-Leader \citep{Agarwal2005-qp,Shalev-Shwartz2007-qm} and in the connections to differential privacy \citep{Abernethy2019-cn,Gonen2019-ec,Alon2022-os}.
For the random-order model,
\citet{Dong2023-yh} identified the \emph{average sensitivity}---the change in an algorithm's output under uniformly random deletions of its input---as the key stability measure.
This quantity captures the robustness of algorithms to random perturbations and has been gaining attention across diverse contexts, including graph algorithms \citep{Varma2021-fl,Varma2023-fy}, dynamic programming \citep{Kumabe2022-hu}, and clustering \citep{Yoshida2022-pb}. \citet{Dong2023-yh} showed that if a $(1+\varepsilon)$-approximate offline algorithm with a bounded average sensitivity is available, then simply feeding its output obtained on the data observed so far into the next prediction yields a $(1+\varepsilon)$-approximate regret bound controlled by the average sensitivity; they termed this procedure the \emph{batch-to-online} transformation. Informally, their $(1+\varepsilon)$-approximate regret bound takes the form ``$\text{Learner's cumulative loss} - (1+\varepsilon)\OPT_T =  O(\varphi(\varepsilon)\log T),$'' where $\OPT_T$ is the offline optimal cumulative loss over $T$ rounds.
The function $\varphi$ represents how the average sensitivity of the offline algorithm depends on the approximation accuracy $\varepsilon$, and it grows as $\varepsilon$ approaches zero, e.g., $\varphi(\varepsilon) \asymp 1/\varepsilon^3$.\looseness=-1

While their result offers a bridge from average sensitivity to online learning under the random-order model, it only applies to the $(1+\varepsilon)$-approximate regret. The approximate regret has traditionally served as a reasonable performance measure when even the offline optimum cannot be attained efficiently, as in the case of online submodular function maximization \citep{Streeter2008-qu,Niazadeh2020-id}. By contrast, \citet{Dong2023-yh} use the approximate regret to balance stability against accuracy, giving the learner an extra advantage of $\varepsilon \OPT_T$ that is introduced only to trade off approximation and stability.
This observation raises a natural question: What form should a \emph{non-approximate} (standard) regret bound take in the random-order model, and how does the average sensitivity of the offline algorithm shape that bound?\looseness=-1

\subsection{Our contribution}
We extend the fixed-$\varepsilon$, $(1+\varepsilon)$-approximate-regret framework of \citet{Dong2023-yh} to obtain a regret bound typically governed by $\tilde O(\varphi^{\star}(\OPT_T))$, where $\varphi^{\star}$ is the \emph{concave conjugate} of $\varphi$ and $\tilde O$ hides polylogarithmic factors in $T$ (and problem-dependent parameters).
For example, when $\varphi(\varepsilon) \asymp 1/\varepsilon^q$ for some $q>0$, this yields a bound of $\tilde O(\OPT_T^{q/(q+1)})$.
Regret bounds that depend on $\OPT_T$ (up to polylogarithmic factors) are called \emph{small-loss} regret bounds (also known as first-order or $L^{\star}$ bounds), which are highly desirable as they automatically exploit benign input sequences with small $\OPT_T$.\looseness=-1

Our core idea is to set the approximation parameter $\varepsilon_t$ separately for each round $t$ to optimize the accuracy--stability trade-off exposed by our regret decomposition (see \cref{sec:bto}).
Two quantities to be balanced are the approximation term $\varepsilon_t\OPT_t/t$ and the average-sensitivity term $\varphi(\varepsilon_t)/t$.
In \cref{sec:adaptive-control}, we show that per-round optimization of $\varepsilon_t$ is sufficient for obtaining a regret bound of $\tilde O(\varphi^{\star}(\OPT_T))$, even though $\OPT_T$ is unknown in advance.\footnote{A standard way to adapt to such an unknown scale is a doubling trick.  In the random-order model, however, extending the fixed-$\varepsilon$ framework of \citet{Dong2023-yh} via doubling is not straightforward. We discuss this obstruction in \cref{sec:doubling-trick}.}
This adaptive choice of $\varepsilon_t$ introduces one additional requirement on the offline algorithm: in addition to the fixed-$\varepsilon$ average-sensitivity bound used by \citet{Dong2023-yh}, the output must be stable when the approximation parameter changes according to our rule for choosing $\varepsilon_t$ (\cref{assump:adaptive-pc-stability}).
Still, this requirement is naturally satisfied by the offline algorithms used in our applications and those of \citet{Dong2023-yh}.\looseness=-1

Importantly, our framework does not rely on linearity or smoothness assumptions on the loss functions, unlike most prior small-loss regret bounds (see \cref{sec:related-work}).
It therefore yields small-loss regret bounds for a broad class of problems, including specific examples studied by \citet{Dong2023-yh}, such as online $k$-means clustering and online low-rank matrix approximation.
Moreover, in \cref{sec:submodular}, we demonstrate its applicability to online submodular function minimization \citep{Hazan2012-hc,Ito2022-oy}, which was not covered by \citet{Dong2023-yh}.
We obtain an $\tilde O\prn[\big]{n^3+n^{3/4}\OPT_T^{3/4}}$ regret bound, where $n$ is the ground-set size.
This showcases the advantage of our framework, as the continuous extension of submodular functions, known as the \emph{Lovász extension}, is not smooth (nor self-bounding).
At a technical level, we leverage $(1\pm\varepsilon)$-cut sparsifiers for \emph{submodular hypergraphs}~\citep{Kenneth2024-pf} to design an offline algorithm with low average sensitivity.
As another example with non-smooth losses, \cref{sec:online-regression} shows an instantiation for online $\ell_1$ regression based on \emph{Lewis-weight sampling}~\citep{Parulekar2021-bm}.
More broadly, \cref{subsec:guiding-idea} outlines a general recipe for deriving small-loss regret bounds from offline approximation algorithms based on widely studied algorithmic techniques, such as coresets, sketching, sampling, or sparsifiers, provided that these algorithms have bounded average sensitivity and parameter-change stability.\looseness=-1

\subsection{Related work}\label{sec:related-work}
\textbf{Small-loss regret bounds.\;}
In prediction with $N$ experts, an $O(\sqrt{\OPT_T\log N}+\log N)$ regret bound is achievable \citep[Section 2.4]{Cesa-Bianchi2006-oa}, where $\OPT_T$ here is the cumulative loss of the best expert.
Small-loss regret bounds for limited-feedback settings (e.g., bandit or partial information) with explicitly enumerated action sets or linear loss functions have been widely studied \citep{Neu2015-tu,Agarwal2017-ve,Allen-Zhu2018-jq,Lykouris2018-bn,Ito2020-fx,Lee2020-zw,Lee2020-vr,Foster2021-lp,Olkhovskaya2023-vy,Wang2023-kz}.
In online convex optimization, small-loss regret bounds have been obtained under smoothness---more precisely, under the \emph{self-bounding} property \citep[Definition~4.25]{orabona2023modern}---of loss functions \citep{Srebro2010-lp,Abernethy2019-cn,Tsai2023-jr,Yang2024-oq,Zhao2024-pi}.\linebreak
By contrast, we impose minimal assumptions on the decision space and loss functions, and focus on a general connection between average sensitivity of offline approximation algorithms and small-loss regret bounds under the random-order model.
Among prior works, the closest to ours is \citet{Abernethy2019-cn}, who derived small-loss regret bounds in the adversarial environment from a stability notion closely connected to differential privacy. Related subsequent works \citep{Wang2022-gw,Block2025-mo} develop oracle-efficient approaches that obtain small-loss regret bounds by leveraging privacy-style stability notions. Our approach can be viewed as a random-order analogue that replaces such privacy-style stability requirements with average sensitivity. In particular, differential privacy is a worst-case stability notion, whereas average sensitivity is an average-case requirement and hence generally less restrictive \citep[Section~1.5]{Varma2023-fy}.
Another relevant work is \citet{Lykouris2018-bn}, who obtained small-loss regret bounds in graph bandits via an approximate-regret analysis using a doubling trick.
This shares a high-level analogy with our idea of extending the approximate-regret analysis of \citet{Dong2023-yh}.\looseness=-1


\textbf{Random-order model.\;}
In competitive analysis, the random-order model has a long history, from the classical secretary problem to subsequent studies across various settings \citep{Dynkin1963-sj,Karp1990-tw,Babaioff2018-hb}.
We refer the reader to \citet{Gupta2021-zv} for an overview.
Recently, there has been growing interest in the random-order model in regret analysis.
\citet{Garber2020-oh} established polylogarithmic regret bounds under cumulative strong convexity of loss functions; later, \citet{Sherman2021-wb} achieved optimal logarithmic regret bounds.
\citet{Bernasconi2025-qb} introduced a general framework that lifts algorithms designed for the stochastic setting to the random-order model, at the cost of an additive $\tilde O(\sqrt{T})$ regret overhead.\looseness=-1


\section{Preliminaries}
Let $[n] = \set*{1,\ldots,n}$ for any $n \in \Z_{>0}$.
We consider an online learning problem over $T$ rounds.
Let~$\mathcal{X}$ be the input space, $\Theta$ the decision space, and $\ell\colon \Theta \times \mathcal{X} \to [0,1]$ the loss function.
At each round $t \in [T]$, the learner selects $\theta_t \in \Theta$, observes $x_t \in \mathcal{X}$, and incurs $\ell\prn*{\theta_t,x_t}$.
We assume the random-order model:
the multiset of $T$ data points $\set*{x_1,\ldots,x_T}\in\mathcal{X}^T$ is chosen adversarially in advance,\footnote{
  We use the abuse of treating multiset occurrences as labeled whenever a multiset is indexed, sampled from, or modified.
}
but it is ordered uniformly at random.
Let $\ell(\theta,X) = \sum_{x\in X} \ell(\theta,x)$ for any multiset $X$ of data points in~$\mathcal{X}$ and any $\theta \in \Theta$.
For each $t \in [T]$, define the offline optimum as
$
\OPT_t = \min_{\theta\in\Theta} \sum_{i=1}^t \ell\prn*{\theta,x_i}
$.
Note that $\OPT_t$ for $t < T$ are random variables depending on the random order of $\set*{x_t}_{t=1}^T$, while $\OPT_T$ is deterministic once the multiset $\set*{x_t}_{t=1}^T$ is fixed.
The following lemma, which also appears in \citet{Dong2023-yh}, follows from the random-order assumption; see \cref{sec:rom-lem-proof} for a proof.
\begin{restatable}[Bound on sum of average optimal values]{lemma}{romlemma}\label[lemma]{lem:rom-lem}
It holds that $\E\brc[\big]{\sum_{t=1}^T {\OPT_t}/{t}} \le \OPT_T$.
\end{restatable}
\wraprestatablewithlabelrestore{romlemma}
Let $\mathsf{Alg}$ be the learner's algorithm for computing $\theta_1,\dots,\theta_T \in \Theta$, which may be randomized.
Denoting by $\mathop{\E}_{\mathsf{Alg},\set*{x_t}}$ the expectation over the randomness of $\mathsf{Alg}$ and the random order of $\set{x_t}_{t=1}^T$, the learner's performance is measured by the regret defined as
$
\Reg_T \coloneqq \mathop{\E}_{\mathsf{Alg},\set*{x_t}}\brc{\sum_{t=1}^T \ell\prn*{\theta_t,x_t}} - \OPT_T
$.\looseness=-1

In this paper, polynomial-time statements are with respect to the problem size, including the horizon~$T$, under an exact-real oracle model in which arithmetic operations, comparisons, calls to the relevant value oracles, and the specified random sampling primitives have unit cost.

\textbf{Average sensitivity.\;}
To define average sensitivity, we first introduce the total variation distance.
\begin{definition}[Total variation]\label[definition]{def:tv}
For probability measures $P,Q$ defined on a measurable space $(\Omega,\mathcal{F})$, we define the total variation distance between $P$ and $Q$ as
$
\mathrm{TV}(P,Q) \coloneqq \mathop{\sup}_{{A \in \mathcal{F}}}\abs*{P(A) - Q(A)}
$.
\end{definition}
The average sensitivity of a randomized algorithm is measured by the mean total-variation distance of output distributions on two datasets that differ by a single point deleted uniformly at random.
Formally, we use the following definition adopted in \citet{Yoshida2022-pb} and \citet{Dong2023-yh}.\looseness=-1
\begin{definition}[Average sensitivity]\label[definition]{def:avg-sens}
Let $t \in\Z_{>0}$.
Consider a randomized algorithm $\mathcal{A}$ that takes $X\in \mathcal{X}^t$ as input and returns $\mathcal{A}(X) \in \Theta$.
By abusing notation, we identify $\mathcal{A}(X)$ with the distribution over $\Theta$ that the algorithm's output follows.
The average sensitivity of $\mathcal{A}$ on $X$ is defined as
\[
\frac{1}{t}\sum_{x\in X} \mathrm{TV}\prn*{\mathcal{A}(X), \mathcal{A}(X\setminus x)},
\]
where $X\setminus x$ denotes the multiset obtained by removing one occurrence of $x$ from $X$.
For a function $\beta\colon \Z_{>0} \to \R_{\ge0}\cup\set*{+\infty}$, we say that the average sensitivity of $\mathcal{A}$ is upper bounded by $\beta$ if the above quantity is at most $\beta(t)$ for every $t \in \Z_{>0}$ and $X\in\mathcal{X}^t$.
\end{definition}

\textbf{Approximation accuracy and average sensitivity.}\;
Keeping an algorithm's average sensitivity low typically comes at the expense of approximation accuracy.
In this study, we focus on algorithms that exhibit the following trade-off between approximation accuracy and average sensitivity.
\begin{assumption}[Offline average-stable approximation]\label[assumption]{assump:approx-avesen}
  Let $t \in [T]$ and $X \in \mathcal{X}^t$.
  For any fixed $\varepsilon\in(0,1]$, the offline randomized algorithm, denoted by $\mathcal A_\varepsilon$, returns a $\prn*{1+\varepsilon}$-approximate solution in expectation, up to an additive $O(t/T)$ term (where the hidden constant may depend on problem-dependent parameters but not on $t$, $T$, $\varepsilon$, or $X$), and its average sensitivity is at most $\varphi(\varepsilon)/(2t)$, i.e.,\looseness=-1
  \[
    \E_{\mathcal A_\varepsilon}\brc*{\ell(\mathcal A_\varepsilon(X),X)} \le (1+\varepsilon) \min_{\theta\in\Theta} \ell(\theta,X) + O\prn*{\frac{t}{T}}
    \;\,
    \text{and}
    \;\,
    \frac{1}{t}\sum_{x\in X} \mathrm{TV}\prn*{\mathcal{A}_\varepsilon(X), \mathcal{A}_\varepsilon(X\setminus x)}\le
    \frac{\varphi\prn*{\varepsilon}}{2t},
  \]
  where $\varphi\colon \R_{\ge0} \to \R_{\ge0}\cup\set*{+\infty}$ is a proper lower-semicontinuous function with effective domain $\mathrm{dom}(\varphi) \coloneqq \Set*{\varepsilon \in \R_{\ge0}}{\varphi(\varepsilon)<+\infty}$ contained in $(0,1]$, and such that $\lim_{\varepsilon\downarrow0}\varphi(\varepsilon) = +\infty$.
\end{assumption}
Typically, $\varphi(\varepsilon)$ is a decreasing function on $(0,1]$, e.g., $\varphi(\varepsilon) \asymp 1/\varepsilon^3$, and thus a better approximation (smaller~$\varepsilon$) increases the average sensitivity bound $\varphi(\varepsilon)/(2t)$, reflecting the trade-off between approximation accuracy and average sensitivity.
The $O(t/T)$ term in the approximation guarantee is non-essential and is included only for consistency with high-probability guarantees of offline algorithms.
The role of \cref{assump:approx-avesen} is to isolate how~$\varepsilon$ and $t$ affect the average sensitivity bound, which is indeed satisfied in all examples discussed later.\footnote{
  Even when an average sensitivity bound $\beta(t)$ is given as $\tilde\varphi(t,\varepsilon)/(2t)$ for some $\tilde\varphi$ increasing in $t$, setting $\varphi(\varepsilon) = \tilde\varphi(T,\varepsilon)$ yields a valid upper bound. Typically, such $\tilde\varphi(t,\varepsilon)$ is logarithmic in $t$, and this modification affects only logarithmically.
}
In \cref{subsec:parameter-change-control}, we also introduce an additional condition (\cref{assump:adaptive-pc-stability}) on offline algorithms to handle time-varying approximation accuracy.\looseness=-1

\section{Batch-to-online transformation with time-varying approximation accuracy}\label{sec:bto}
\begin{algorithm2e}[tb]
{\SetAlCapHSkip{0pt}\caption{Batch-to-online transformation with time-varying approximation accuracy}\label[algorithm]{alg:bto}}
\Input{Offline algorithm family $\{\mathcal{A}_\varepsilon\}_{\varepsilon\in(0,1]}$ and parameter $\lambda>0$\;}
Initialize $\theta_1\in\Theta$ arbitrarily\;
\For{$t = 1,\dots,T$}{
  Play $\theta_t$, receive $x_t$, and incur $\ell\prn*{\theta_t, x_t}$\;
  Compute $\varepsilon_t\in(0,1]$ from $X_t=\set{x_1,\dots,x_t}$ observed so far and $\lambda$ by the rule in \eqref{eq:rule}\;
  Obtain $\theta_{t+1} \in \Theta$ by running $\mathcal{A}_{\varepsilon_t}$ on $X_t$\;
}
\end{algorithm2e}

We introduce the batch-to-online transformation.
As described in \cref{alg:bto}, at each round $t$, the offline algorithm~$\mathcal A_{\varepsilon_t}$ is applied to the dataset~$X_t$ observed so far to obtain $\theta_{t+1}$.
A crucial difference from \citet{Dong2023-yh} is that we allow the approximation parameter $\varepsilon_t$ to vary over rounds.
This time-varying view yields a new regret decomposition, exposing the per-round accuracy--stability trade-off and the cost of changing $\varepsilon_t$ over time.
\begin{restatable}[Regret decomposition]{proposition}{regretdecomprop}
\label[proposition]{prop:regret-decomp}
Suppose that the parameters $\varepsilon_t$ used by \cref{alg:bto} are specified by a measurable rule $Y\mapsto\varepsilon_Y$ for any finite multiset $Y$, so that $\varepsilon_t=\varepsilon_{X_t}$ holds for every~$t$; let $\varepsilon_\emptyset \in (0,1]$ be an arbitrary fixed value.
For $\theta_1,\dots,\theta_T\in\Theta$ obtained by \cref{alg:bto}, the regret~$\Reg_T$ is upper bounded by the following two sums up to an additive constant:\footnote{For convenience, we extend each $\mathcal A_\varepsilon$ to the empty multiset by an arbitrary fixed output distribution. All terms involving the empty input affect the regret by at most an additive $O(1)$.}
\[
\E\brc[\Bigg]{\sum_{t=1}^T
  \underbrace{
    \prn*{
    \frac{\varepsilon_t}{t}\OPT_t + \frac{\varphi\prn*{\varepsilon_t}}{t}
  }
  }_{\decomptarget{summand:approx-sens}{accuracy--stability trade-off}}
}
+
\E\brc[\Bigg]{\sum_{t=1}^T
  \underbrace{
  \frac{2}{t}\sum_{x\in X_t}
  \mathrm{TV}\prn*{
    \mathcal A_{\varepsilon_{X_t}}\prn*{X_t\setminus x},
    \mathcal A_{\varepsilon_{X_t\setminus x}}\prn*{X_t\setminus x}
  }}_{\decomptarget{summand:param-change}{parameter-change cost}}
}
+O(1).
\]
\end{restatable}
\wraprestatablewithlabelrestore{regretdecomprop}
We defer the proof to \cref{asec:proofs-bto}, but the intuition is as follows.
The output $\theta_{t+1}$ returned by the $(1+\varepsilon_t)$-approximation algorithm satisfies $\ell(\theta_{t+1}, X_t)/t \lesssim (1+\varepsilon_t)\OPT_t/t$, and evaluating it on the fresh loss $\ell(\cdot, x_{t+1})$ introduces an additive error governed by the average sensitivity plus the cost of using $\varepsilon_t = \varepsilon_{X_t}$ on $X_t\setminus x$ instead of $\varepsilon_{X_t\setminus x}$, hence $\E[\ell(\theta_{t+1}, x_{t+1})] \lesssim \E[(1+\varepsilon_t)\OPT_t/t + \varphi(\varepsilon_t)/t] + \text{\decomplink{summand:param-change}{parameter-change cost}}$.
Summing over $t$ and applying \cref{lem:rom-lem} yield the above regret decomposition.\looseness=-1

\textbf{Observation from the time-varying view.\;}
Allowing the approximation parameter $\varepsilon_t$ to vary over rounds reveals a more informative structure than the fixed-$\varepsilon$ presentation of \citet{Dong2023-yh}: the \decomplink{summand:approx-sens}{accuracy--stability trade-off} term takes the form of a \emph{per-round} trade-off between the two terms,
$\frac{\OPT_t}{t}\varepsilon_t$ and $\frac{\varphi(\varepsilon_t)}{t}$.
Importantly, $\OPT_t$ can be computed from $X_t$ and hence available when choosing~$\varepsilon_t$ for obtaining $\theta_{t+1}$, although the target quantity $\OPT_T$ is not.
In addition, as we will see later, this trade-off term typically dominates the \decomplink{summand:param-change}{parameter-change cost}.
Therefore, a natural strategy is: choose~$\varepsilon_t$ by optimizing the per-round trade-off, which is the perspective underlying the next section.\looseness=-1

\section{Small-loss regret bound via time-varying approximation accuracy}\label{sec:adaptive-control}
We present how to obtain a small-loss regret bound by bounding the two sums in the regret decomposition in \cref{prop:regret-decomp}.
The following concave conjugate of $\varphi$ plays a central role in our analysis.\looseness=-1
\begin{definition}[Concave conjugate]\label[definition]{def:concave-conjugate}
  Let $\varphi\colon \R_{\ge0} \to \R_{\ge0}\cup\set*{+\infty}$ be a proper lower-semicontinuous function with $\mathrm{dom}(\varphi) \subseteq (0,1]$ and $\lim_{\varepsilon\downarrow0}\varphi(\varepsilon)=+\infty$.
  The concave conjugate $\varphi^{\star}\colon \R_{>0} \to \R_{\ge0}$ of~$\varphi$ is defined as
  $\varphi^{\star}(u) = \inf_{\varepsilon\ge0}\set*{u \varepsilon+\varphi\prn*{\varepsilon}}$.
  For $u > 0$, the infimum is attained;
  define $\varepsilon_{\min}(u)$ as one minimizer chosen by a fixed tie-breaking rule, i.e.,
  $\varepsilon_{\min}(u) \in \mathop{\arg\min}_{\varepsilon\ge0}\set{ u \varepsilon + \varphi\prn*{\varepsilon}}$.\footnote{Since the divergence at zero excludes a sufficiently small interval $(0,\delta)$ from the relevant sublevel set of $u\varepsilon+\varphi(\varepsilon)$, the minimization reduces to the compact interval $[\delta,1]$; therefore, lower-semicontinuity gives a minimizer. We also require $\varepsilon_{\min}$ to be Borel-measurable; in the applications, this is immediate as the minimizer is unique except possibly at the endpoint.}
\end{definition}
For example, when $\varphi(\varepsilon) \asymp 1/\varepsilon^3$, we have $\varepsilon_{\min}(u) \asymp u^{-1/4}$ and $\varphi^{\star}(u)\asymp u^{3/4}$.
We will use the following basic properties of $\varphi^{\star}$ and $\varepsilon_{\min}$; see \cref{asec:proof-concave-conjugate-prop} for the proof.
\begin{restatable}{lemma}{concaveconjugateproplemma}
\label[lemma]{lem:concave-conjugate-prop}
The function
$\varphi^{\star}(u)$ is non-decreasing and concave;
$\varepsilon_{\min}(u)$ is contained in $(0,1]$, non-increasing, and a supergradient of $\varphi^\star$ at $u > 0$, i.e., $\varepsilon_{\min}(u)(u - v) \le \varphi^\star(u) - \varphi^\star(v)$ for all $v > 0$.\looseness=-1
\end{restatable}
\wraprestatablewithlabelrestore{concaveconjugateproplemma}

\subsection{Choosing approximation accuracy to optimize per-round trade-off}\label{sec:adaptive-eps}
We show how to choose $\varepsilon_t$.
The core idea is to set $\varepsilon_t$ to optimize the \decomplink{summand:approx-sens}{accuracy--stability trade-off}, or $\OPT_t \varepsilon_t + \varphi(\varepsilon_t)$.
We derive a useful bound on the optimized trade-off terms by leveraging the properties of $\varphi^{\star}$.
Note that, while the following choice of~$\varepsilon_t$ involves $\OPT_t$, which varies with the random order, the bound holds for every realization of the randomness.
Also, the realized value of $\OPT_t$ is available when choosing~$\varepsilon_t$, since we can use the observed data points $X_t=\set*{x_s}_{s=1}^t$ to compute $\OPT_t$.\footnote{
This use of $\OPT_t$ is information-theoretic, and our main bound (\cref{thm:general-separable}) does not claim efficient implementability.
In the applications below, $\OPT_t$ is polynomial-time computable for submodular minimization, $\ell_1$ regression, and low-rank approximation; for clustering, we assume access to an exact solver or a sufficiently accurate offline approximation scheme~\citep{Cohen-Addad2019-nc}.\looseness=-1
}\looseness=-1
\begin{lemma}[Per-round trade-off optimization]
\label[lemma]{lem:local-opt-eps}
Let $\lambda > 0$ and consider the following choice of $\varepsilon_t$:
\begin{equation}
  \label{eq:rule}
  \varepsilon_{t} = \varepsilon_{\min}\prn*{\OPT_t + \lambda} \in \mathop{\arg\min} \Set*{ (\OPT_t + \lambda) \varepsilon + \varphi\prn*{\varepsilon} }{{\varepsilon\ge0}}
  \qquad\text{for } t = 1,\dots,T.
\end{equation}
Writing $H_T = \sum_{s=1}^{T}\frac{1}{s}$, the \decomplink{summand:approx-sens}{accuracy--stability trade-off} part in \cref{prop:regret-decomp} is bounded as follows:\looseness=-1
\[
  \sum_{t=1}^{T}\prn*{ \frac{\OPT_{t}}{t} \varepsilon_{t}+\frac{\varphi\prn*{\varepsilon_{t}}}{t} }
  \le
  H_T \varphi^{\star}\prn*{\frac{\sum_{t=1}^{T} \OPT_t/t}{H_T} + \lambda}.
\]
\end{lemma}
\begin{proof}
By the definition of $\varepsilon_t = \varepsilon_{\min}(\OPT_t + \lambda)$, we have $\prn*{\OPT_t + \lambda}\varepsilon_t + \varphi\prn*{\varepsilon_t}
= \varphi^{\star}\prn*{\OPT_t + \lambda}$.
Rearranging the terms and using $\lambda > 0$ and $\varepsilon_t \ge 0$ yield
$
\OPT_t\varepsilon_t + \varphi\prn*{\varepsilon_t}
= \varphi^{\star}\prn*{\OPT_t + \lambda} - \lambda\varepsilon_t
\le \varphi^{\star}\prn*{\OPT_t + \lambda}
$.
Dividing by $t$ and summing over $t=1,\dots,T$ give
\[
\sum_{t=1}^{T}\prn*{\frac{\OPT_t}{t}\varepsilon_t + \frac{\varphi\prn*{\varepsilon_t}}{t}}
\le
\sum_{t=1}^{T} \frac{\varphi^{\star}\prn*{\OPT_t + \lambda}}{t}
=
H_T\sum_{t=1}^{T} \frac{t^{-1}}{H_T}\varphi^{\star}\prn*{\OPT_t + \lambda}.
\]
Since $\varphi^{\star}$ is concave, Jensen's inequality applied to the probability weights $\prn{t^{-1}/H_T}_{t=1}^T$ gives
\[
H_T\sum_{t=1}^{T} \frac{t^{-1}}{H_T}\varphi^{\star}\prn*{\OPT_t + \lambda}
\le
H_T\varphi^{\star}\prn*{\sum_{t=1}^{T} \frac{t^{-1}}{H_T}\prn*{\OPT_t + \lambda}}
= H_T\varphi^{\star}\prn*{\frac{\sum_{t=1}^{T} \OPT_t/t}{H_T} + \lambda}.
\]
Combining the above inequalities completes the proof.
\end{proof}
We introduce $\lambda>0$ as a small shift to ensure, together with $\mathrm{dom}(\varphi) \subseteq (0,1]$ and $\lim_{\varepsilon\downarrow0}\varphi(\varepsilon)=+\infty$, that~$\varepsilon_t$ is well defined in $(0,1]$ even when $\OPT_t = 0$; below we take $\lambda\le H_T^{-(q+1)/q}$.

\subsection{Bounding parameter-change cost}\label{subsec:parameter-change-control}
To control the \decomplink{summand:param-change}{parameter-change cost} part in \cref{prop:regret-decomp}, we need an additional condition tailored to the parameter selection rule \eqref{eq:rule}.
This is indeed satisfied in all applications considered in this paper.\looseness=-1
\begin{assumption}[Parameter-change stability]\label[assumption]{assump:adaptive-pc-stability}
  Fix $\lambda>0$ and consider the selection rule of $\varepsilon_{\min}$ in~\eqref{eq:rule}.
  For any multiset $Y$, we write $\OPT(Y)\coloneqq\min_{\theta\in\Theta}\ell(\theta,Y)$ and $\varepsilon_Y\coloneqq\varepsilon_{\min}(\OPT(Y)+\lambda)$, where $\OPT(\emptyset)=0$ and $\varepsilon_\emptyset=\varepsilon_{\min}(\lambda)$.
  Under this setting, we assume that the offline algorithm family $\{\mathcal A_\varepsilon\}_{\varepsilon\in(0,1]}$ satisfies the following condition with constant $C_{\rm pc}\ge0$ for any multiset $X$ with $|X|\ge2$:\looseness=-1
  \[
  \frac1{|X|}\sum_{x\in X}
  \mathrm{TV}\prn*{
  \mathcal A_{\varepsilon_X}(X\setminus x),
  \mathcal A_{\varepsilon_{X\setminus x}}(X\setminus x)
  }
  \le
  \frac{C_{\rm pc}\varepsilon_X}{2|X|}
  \sum_{x\in X}\prn*{\OPT(X)-\OPT(X\setminus x)}.
  \]
\end{assumption}
\Cref{assump:adaptive-pc-stability} is for controlling the extra total variation caused by changing only the accuracy parameter from $\varepsilon_X$ to~$\varepsilon_{X\setminus x}$ on the same input $X\setminus x$.
In the applications below, we can verify this condition through a common two-step argument. First, for a fixed input $Y$, the output distributions satisfy
$
\mathrm{TV}(\mathcal A_{\varepsilon}(Y),\mathcal A_{\eta}(Y))
\le
O(\varphi(\varepsilon)\frac{\eta-\varepsilon}{\varepsilon})
$
for $0<\varepsilon\le\eta\le1$.
Second, the rule $\varepsilon_Y=\varepsilon_{\min}(\OPT(Y)+\lambda)$, together with a regularity condition of $\varphi$ in \cref{sec:calculus-for-corollary}, converts the relative parameter change into $\varepsilon_X(\OPT(X)-\OPT(X\setminus x))$, and averaging over $x\in X$ ensures the requirement in \cref{assump:adaptive-pc-stability}.
The following lemma is a direct consequence of \cref{assump:adaptive-pc-stability}; see \cref{asec:proof-pc-bound} for the proof.
\begin{restatable}[Controlling parameter-change cost]{lemma}{pcboundlemma}\label[lemma]{lem:pc-bound}
  Let $\OPT_0=0$ and assume \cref{assump:adaptive-pc-stability}. Then, the \decomplink{summand:param-change}{parameter-change cost} part in \cref{prop:regret-decomp} is bounded as follows:
  \[
  \E\brc*{
  \sum_{t=1}^{T}\frac{2}{t}\sum_{x\in X_t}
  \mathrm{TV}\prn*{
    \mathcal A_{\varepsilon_{X_t}}\prn*{X_t\setminus x},
    \mathcal A_{\varepsilon_{X_t\setminus x}}\prn*{X_t\setminus x}
  }}
  \le
  C_{\rm pc}\E\brc*{\sum_{t=1}^{T}\varepsilon_t(\OPT_t\!-\!\OPT_{t-1})}
  \!+O(1).
  \]
\end{restatable}
\wraprestatablewithlabelrestore{pcboundlemma}

\subsection{Small-loss regret bound}\label{subsec:small-loss-bound}
\Cref{lem:local-opt-eps,lem:pc-bound} combined with \cref{prop:regret-decomp} yield the following small-loss regret bound.
\begin{theorem}[Small-loss regret bound]
\label{thm:general-separable}
Under the random-order model and \cref{assump:approx-avesen}, fix $\lambda>0$ and suppose that \cref{assump:adaptive-pc-stability} holds for this $\lambda$.
Then, \Cref{alg:bto} with $\varepsilon_t$ given by \eqref{eq:rule} achieves
\[
\Reg_T
\le{}
H_T \varphi^{\star}\prn*{ \frac{\OPT_T}{H_T} + \lambda}
+
C_{\rm pc}\prn*{\varphi^\star(\OPT_T+\lambda)-\varphi^\star(\lambda)}
+ O(1).
\]
\end{theorem}
\begin{proof}
  First, we bound the \decomplink{summand:approx-sens}{accuracy--stability trade-off} part in \cref{prop:regret-decomp} via \cref{lem:local-opt-eps}.
  Writing $u_T = \prn{\sum_{t=1}^{T} \OPT_t/t}/H_T + \lambda$, Jensen's inequality for concave $\varphi^{\star}$ gives $\E[\varphi^{\star}(u_T)]\le \varphi^{\star}(\E[u_T])$.
  Since $\E\brc{u_T} \le \OPT_T / H_T + \lambda$ follows from \cref{lem:rom-lem} and $\varphi^{\star}$ is non-decreasing, \cref{lem:local-opt-eps} implies that the \decomplink{summand:approx-sens}{accuracy--stability trade-off} part is at most $H_T\varphi^{\star}(\OPT_T/H_T+\lambda)$.
  We then upper bound the \decomplink{summand:param-change}{parameter-change cost} part via \cref{lem:pc-bound}.
  Let $a_t = \OPT_t + \lambda$ for $t\in[T]$ and $a_0 = \lambda$.
  It holds that $\varepsilon_t(\OPT_t-\OPT_{t-1}) = \varepsilon_{\min}(a_t)(a_t - a_{t-1}) \le \varphi^{\star}(a_t) - \varphi^{\star}(a_{t-1})$ thanks to the supergradient property of $\varepsilon_{\min}$ in \cref{lem:concave-conjugate-prop}.
  Telescoping the sum over $t$, taking expectation, and using \cref{lem:pc-bound} yield a bound of $C_{\rm pc}(\varphi^{\star}(\OPT_T+\lambda) - \varphi^{\star}(\lambda))$ on the \decomplink{summand:param-change}{parameter-change cost} part.
  Finally, summing the bounds on the two parts completes the proof.
\end{proof}
Below, we instantiate this theorem for a commonly used form of $\varphi$; see \cref{asec:proof-cor-simple-form} for the derivation.\looseness=-1
\begin{restatable}{corollary}{simpleformcorollary}\label[corollary]{cor:simple-form}
  Let $q>0$ be a constant.
  Assume $\varphi(\varepsilon)=C_1\varepsilon^{-q}\log(\mathrm{e} + C_2/\varepsilon)$ for $\varepsilon\in(0,1]$ and $\varphi(\varepsilon)=+\infty$ for $\varepsilon\notin(0,1]$ for some $C_1 \ge 1$ and $C_2 \ge 0$.
  Then, for any $\lambda \in (0, H_T^{-(q+1)/q}]$ such that \cref{assump:adaptive-pc-stability} holds, \cref{thm:general-separable} yields
    $
    \Reg_T
    =
    \tilde O
    \prn{
    (1+C_{\rm pc})\prn{C_1 + C_1^{{1}/{(q+1)}}\OPT_T^{{q}/{(q+1)}}}
    }
    $.
\end{restatable}
\wraprestatablewithlabelrestore{simpleformcorollary}
In the above bound, the dependence on $C_2$ is logarithmic and thus hidden by $\tilde O$; see \cref{asec:proof-cor-simple-form} for the explicit bound.
Note that $C_1$, $C_2$, and $C_{\rm pc}$ may depend on problem-specific parameters and $T$.

\textbf{Recovering the $(1+\varepsilon)$-approximate regret bound.\;}
Under $C_{\rm pc}=O(1)$, \cref{thm:general-separable} simplifies to $\Reg_T\le C_0H_T\varphi^{\star}(\OPT_T/H_T + \lambda)$ for some constant $C_0 \ge 1$, and the definition of $\varphi^{\star}$ implies $\varphi^{\star}(\OPT_T/H_T + \lambda) \le \varepsilon C_0^{-1}(\OPT_T/H_T + \lambda) + \varphi(\varepsilon/C_0)$ for $\varepsilon\in(0,1]$.
Combining them with $\lambda = O(C_0/(\varepsilon H_T))$ yields
$
  \E\brc{\sum_{t=1}^T \ell\prn*{\theta_t,x_t}} - (1+\varepsilon)\OPT_T
  \le
  \varphi(\varepsilon/C_0)\log T + O(1)
$,
recovering the $(1+\varepsilon)$-approximate regret bound of \citet{Dong2023-yh} up to $O(1)$ and~$C_0$ in the argument of $\varphi$.
Thus, our bound refines their fixed-$\varepsilon$ result by tuning $\varepsilon$ to optimize the regret bound without knowing $\OPT_T$;
note that a naive doubling does not apply in this random-order analysis (see \cref{sec:doubling-trick}).\looseness=-1

\textbf{Inconsistency bound.\;}
In addition to regret, \citet{Dong2023-yh} considered \emph{inconsistency} $\mathrm{Inc}_T$, the expected number of changes in the maintained outputs between consecutive rounds.
This notion is relevant when changing a deployed decision incurs switching, communication, or recomputation costs.
Using independent reruns of \cref{alg:bto} is not sufficient for controlling this quantity.
However, as shown in \cref{asec:small-loss-inconsistency}, if consecutive output distributions are implemented through probability transformation \citep[Definition~4.1 and Lemma~4.5]{Yoshida2022-pb}, we can prove
$
\mathrm{Inc}_T \lesssim
H_T\varphi^\star\prn*{\OPT_T/H_T+\lambda}
+
C_{\rm pc}\prn*{\varphi^\star(\OPT_T+\lambda)-\varphi^\star(\lambda)}
$ (see \cref{thm:small-loss-inconsistency}), while the marginal law at each $t$ remains $\mathcal A_{\varepsilon_t}(X_t)$ and hence the regret guarantee is unchanged.
For the function $\varphi$ in \cref{cor:simple-form}, this recovers the fixed-$\varepsilon$ inconsistency bound of $O(\varphi(\varepsilon)\log T)$ obtained by \citet{Dong2023-yh} (see \cref{cor:inconsistency-simple-form}).\looseness=-1

\section{Application to online submodular function minimization}\label{sec:submodular}
This section shows how our framework applies to online submodular function minimization \citep{Hazan2012-hc}.
Let $V$ be a finite ground set of size $n \in \Z_{>0}$ and $\Theta = 2^V$ the decision set.
We write $\ell_t(\cdot)$ for $\ell(\cdot, x_t)$ for convenience, and we assume that each $\ell_t\colon 2^V \to [0,1]$ is a submodular function, i.e., $\ell_t(S_1) + \ell_t(S_2) \ge \ell_t(S_1 \cup S_2) + \ell_t(S_1 \cap S_2)$ for any $S_1, S_2 \subseteq V$.
In the random-order model, the multiset of functions $\ell_t$ is chosen by an adversary but they arrive in a uniformly random order.
We sequentially select $S_t \in \Theta$ to minimize the regret against the optimal value $\OPT_T = \min_{S \in \Theta} \sum_{t=1}^T \ell_t(S)$.\looseness=-1

\subsection{Offline algorithm}
Our offline algorithm is based on a sparsification method of \citet{Kenneth2024-pf} for \emph{submodular hypergraphs};
see \cref{sec:submod-avg-sens-proof} for details.
Their algorithm constructs a $(1\pm\varepsilon)$-cut sparsifier in polynomial time for the following general setting, and we can show that it has low average sensitivity.\looseness=-1

Let $H=(U,E,\set{g_e}_{e\in E})$ be a submodular hypergraph, where $U$ is a finite vertex set, $E$ is a multiset of hyperedges on $U$, and each $e \in E$ has a non-negative submodular splitting function $g_e\colon 2^e\to\R_{\ge0}$ with $g_e(\emptyset)=0$.
We call $H'=(U,E',\set{g'_e}_{e\in E'})$ a \emph{reweighted subgraph} of $H$  if $E' \subseteq E$ and for each $e \in E'$, there exists $s_e\!>\!0$ such that $g'_e\!\equiv\!s_e g_e$.
We define the cut value as $\mathrm{cut}_H(A)\!=\!\sum_{e\in E} g_e(A\cap e)$ for $A\subseteq U$.
\Citet[Theorem~1.4]{Kenneth2024-pf} show that for any $\varepsilon > 0$, we can construct a reweighted subgraph $H'$ of $H$ with $|E'| = \tilde O(\varepsilon^{-2} |U|^3)$ hyperedges such that, with high probability,\looseness=-1
\begin{equation}\label{eq:cut-approx}
  (1 - \varepsilon)\cdot\mathrm{cut}_H(A)
  \le
  \mathrm{cut}_{H'}(A)
  \le (1 + \varepsilon)\cdot\mathrm{cut}_H(A)
  \qquad
  \text{for all \; $A \subseteq U$.}
\end{equation}
To apply this guarantee to online submodular function minimization, where $\ell_t(\emptyset)$ need not be zero, we use a one-vertex anchored lift: let $r\notin V$ be a dummy vertex, set $U=V^+\coloneqq V\cup\set{r}$, and, for each loss $\ell_t\colon 2^V\to[0,1]$, create a hyperedge $e_t=U$ with splitting function $g_{e_t}\colon 2^U\to\R_{\ge0}$ defined by $g_{e_t}(\emptyset)=0$ and $g_{e_t}(A)=\ell_t(A\cap V)$ for $A\neq\emptyset$.
This gives valid non-negative submodular splitting functions (see \cref{subsec:anchored-lift}).
With $E_t=\set*{e_1,\dots,e_t}$ and $H_t^+=(U,E_t,\set{g_e}_{e\in E_t})$, each original action $S\subseteq V$ is represented by the anchored cut $S\cup\set{r}$, for which $\mathrm{cut}_{H_t^+}(S\cup\set{r})=\sum_{s=1}^t \ell_s(S)$.
The offline algorithm constructs a $(1\pm\varepsilon')$-cut sparsifier $\tilde H_t^+$ of $H_t^+$ for appropriate $\varepsilon'=\Theta(\varepsilon)$ and returns a minimizer of $S\mapsto\mathrm{cut}_{\tilde H_t^+}(S\cup\set{r})$ over $S\subseteq V$, using a fixed deterministic tie-breaking rule.\looseness=-1

The main remaining task is to upper bound the average sensitivity of this offline algorithm.
For a submodular hypergraph $H = (U, E, \set{g_e}_{e \in E})$, let $H - e$ denote the hypergraph obtained by removing $e \in E$ and its splitting function $g_e$ from $H$.
Consider applying the offline algorithm to $H$ and $H - e$.
Thanks to the data processing inequality, the deterministic minimization step does not increase the total variation distance between the two output distributions.
Therefore, it suffices to analyze the average sensitivity of the sparsifier construction step.
In \cref{sec:submod-avg-sens-proof}, we prove the following bound.\looseness=-1
\begin{proposition}\label[proposition]{prop:submod-avg-sens}
  Let $\varepsilon \in (0,1]$.
  For a submodular hypergraph $H = (U, E, \set{g_e}_{e \in E})$, let $\mathcal{L}_\varepsilon(H)$ be the law of the $(1\pm\varepsilon)$-cut sparsifier constructed by the method of \citet{Kenneth2024-pf} (with weight perturbation in \cref{subsec:weight-perturbation}), which satisfies \eqref{eq:cut-approx} with probability $1-1/T$.
  The average sensitivity of~$\mathcal{L}_\varepsilon$ is bounded as
  $
    \sum_{e \in E} \mathrm{TV}\prn*{\mathcal{L}_\varepsilon(H), \mathcal{L}_\varepsilon(H - e)} / |E|
    =
    O\prn*{ \varepsilon^{-3}(|U|^{3} + |U|^2 \log T) / |E|}
  $.
\end{proposition}
Since the lifted hypergraph has $|U| = n+1$ vertices, this implies $\varphi(\varepsilon)=C_{\rm sub}\varepsilon^{-3}(n^3+n^2\log T)$ for some constant $C_{\rm sub} \ge 1$.
\Cref{prop:submod-same-hypergraph-pc} also shows that \cref{assump:adaptive-pc-stability} is satisfied with $C_{\rm pc} = O(1)$.\looseness=-1

\subsection{Regret analysis}
Given the above offline approximation algorithm equipped with the average sensitivity bound and the parameter-change stability, we can apply \cref{cor:simple-form} to obtain the following result for online submodular function minimization under the random-order model.\looseness=-1
\begin{theorem}\label{thm:submod-regret}
  For online submodular function minimization under the random-order model, \cref{alg:bto} with $\lambda=H_T^{-4/3}$ runs in polynomial time and achieves
    $\Reg_T
    =
    \tilde O\prn[\big]{n^3+n^{3/4}\OPT_T^{3/4}}$.
\end{theorem}
\begin{proof}
  Recall that the multiset of $t$ submodular functions $\ell_1,\dots,\ell_t$ corresponds to a submodular hypergraph $H_t^+$ on $n+1$ vertices with $|E_t| = t$ hyperedges.
  Then, \Cref{prop:submod-avg-sens} guarantees that the $(1 + \varepsilon)$-approximate offline algorithm satisfies an average sensitivity bound with $\varphi(\varepsilon) = C_{\rm sub}\varepsilon^{-3}(n^3 + n^2\log T)$.
  Moreover, \cref{prop:submod-same-hypergraph-pc} establishes \cref{assump:adaptive-pc-stability} with $C_{\rm pc}=O(1)$.
  Thus, \Cref{cor:simple-form} with $q=3$, $C_1 = C_{\rm sub}(n^3 + n^2\log T)$, $C_2 = 0$, and $C_{\rm pc}=O(1)$ proves the claim.\looseness=-1
\end{proof}
For online submodular function minimization in the adversarial regime, an $O(\sqrt{n \OPT_T} + n)$ regret bound is achievable by treating $2^n$ subsets as experts and applying the multiplicative weights update (see, e.g., \citet[Section~2.4]{Cesa-Bianchi2006-oa}).\footnote{An idealized continuous-domain exponential-weights argument over the Lovász extension also suggests a possible route to small-loss bounds. However, the efficient random-walk tracking result of \citet{Narayanan2010-vb} is proved for linear Lipschitz losses and does not by itself yield a polynomial-time small-loss guarantee for the non-smooth Lovász-extension losses.\looseness=-1}
However, this is prohibitively expensive, and therefore researchers have explored polynomial-time algorithms.
\Citet{Hazan2012-hc} obtained a regret bound of~$O(\sqrt{nT})$ in the adversarial regime with the online subgradient descent.
\Citet{Ito2022-oy} obtained a best-of-both-worlds bound---$O(n/\Delta)$ in the stochastic regime with the suboptimality gap $\Delta > 0$ and $O(\sqrt{nT})$ in the adversarial regime---and extended it to the stochastic setting with adversarial corruptions.
Unlike these prior results, we focus on the random-order model and obtain a small-loss regret bound.
Notably, when $\OPT_T = O(1)$, our result achieves a regret bound that is polylogarithmic in $T$ for fixed $n$ even under the random-order model, without any assumption on the suboptimality gap $\Delta$.

\section{\texorpdfstring{Application to online $\ell_1$ regression}{Application to online l1 regression}}\label{sec:online-regression}
We next instantiate the framework for online $\ell_1$ regression. At each round $t$, the learner incurs the absolute loss $\ell\prn*{\theta_t, x_t} = \abs*{\inpr{a_t, \theta_t} - b_t}$, where $x_t = (a_t, b_t) \in \R^d \times \R$ is a pair of a feature vector and a target value, and computes $\theta_{t+1} \in \Theta \subseteq \R^d$ based on $\set*{x_s}_{s=1}^t$ for the next round.
We assume that $\Theta$ is a bounded convex set over which offline $\ell_1$ regression is polynomial-time solvable and that $\ell(\theta,x_t)\in[0,1]$ holds for all $\theta\in\Theta$ and all adversarially chosen data points $x_t$, possibly after a known problem-dependent scaling.
In the random-order model, the data points arrive in a uniformly random order.\looseness=-1

\textbf{Offline algorithm.\;}
We use a $(1+\varepsilon)$-approximate polynomial-time offline $\ell_1$-regression algorithm based on \emph{Lewis-weight sampling}~\citep{Parulekar2021-bm}; see \cref{sec:offline-regression} for details.
Given data points $X_t$, it performs sampling based on Lewis weights of $\set*{(a_s,-b_s)}_{s\in [t]}$ to yield a weighted multiset of sampled data points, so that the weighted empirical loss multiplicatively approximates $\ell(\theta,X_t)$ for all $\theta\in\Theta$ with high probability; then, it returns a minimizer of the weighted empirical loss.
We can bound the average sensitivity of this algorithm by ${\varphi(\varepsilon)}/({2t})$ with
$
  \varphi(\varepsilon)
  =
  C_{\ell_1}
  {d}{\varepsilon^{-3}}\log (\mathrm e+{dT}/{\varepsilon})
$
for some constant $C_{\ell_1} \ge 1$ and verify \cref{assump:adaptive-pc-stability} with $C_{\rm pc}=O(1)$; see \cref{thm:avg-sensitivity-l1,prop:l1-same-input-pc}, respectively.\looseness=-1

\textbf{Regret analysis.\;}
Given the offline algorithm with $\varphi(\varepsilon) = C_{\ell_1} {d}{\varepsilon^{-3}}\log (\mathrm e+{dT}/{\varepsilon})$ and $C_{\rm pc}=O(1)$, we can apply \cref{cor:simple-form} to obtain the following small-loss regret bound.
\begin{theorem}\label{thm:regression-regret}
  For online $\ell_1$ regression under the random-order model, \cref{alg:bto} with $\lambda=H_T^{-4/3}$ runs in polynomial time and achieves
  $
    \Reg_T
    =
    \tilde O\prn[\big]{
      d+
      d^{1/4}\OPT_T^{{3}/{4}}
    }
  $.
\end{theorem}
\begin{proof}
  \Cref{cor:simple-form} with
  $q=3$, $C_1 = C_{\ell_1}d$, $C_2 = dT$, and $C_{\rm pc}=O(1)$ implies the claim.
\end{proof}
While online $\ell_1$ regression is a classical research topic \citep{Bernstein1992-eu,Bylander1997-ut}, we are not aware of prior work that provides a comparable small-loss regret bound.
One likely reason is that the absolute loss is not self-bounding, preventing a standard route to deriving small-loss regret bounds in online convex optimization (e.g., \citet[Section~4.2]{orabona2023modern}).
Our approach circumvents this issue under the random-order model by leveraging the average-stable offline algorithm based on Lewis-weight sampling.\looseness=-1

\section{Further applications and a general idea for achieving small-loss bounds}\label{subsec:guiding-idea}
\citet{Dong2023-yh} also applied the batch-to-online transformation to online $(k,z)$-clustering and online low-rank matrix approximation, thereby obtaining $(1+\varepsilon)$-approximate regret bounds.
Our \cref{cor:simple-form} applies to them, yielding small-loss regret bounds that refine their $(1+\varepsilon)$-approximate regret bounds, as discussed in \cref{subsec:small-loss-bound}.
We present the details in \cref{sec:dong-yoshida-applications}.

Having observed a variety of applications, we now discuss a general idea for achieving small-loss regret bounds.
Our framework hinges on average-stable offline algorithms with \emph{multiplicative} $(1+\varepsilon)$-approximation guarantees and stability under changes of the accuracy parameter on the same input.
The approximation guarantee and the average-sensitivity bound control the \decomplink{summand:approx-sens}{accuracy--stability trade-off} through the choice of $\varepsilon_t$ (\cref{lem:local-opt-eps}), while parameter stability controls the \decomplink{summand:param-change}{parameter-change cost} (\cref{lem:pc-bound}).
Once these two parts are controlled, \cref{thm:general-separable} yields a small-loss regret bound in terms of $\varphi^{\star}(\OPT_T)$.
Algorithmic compression techniques---sparsification, sampling, coresets, and sketching---are natural ingredients for this approach: the resulting offline algorithms provide multiplicative approximation guarantees, and changing the input or the accuracy parameter affects intermediate quantities such as sampling probabilities, sample counts, or reweighting factors, whose total-variation changes can be bounded directly.
Concretely, the $(1\pm\varepsilon)$-cut sparsifier for submodular hypergraphs and Lewis-weight sampling instantiate this mechanism in online submodular function minimization and online $\ell_1$ regression, respectively; coreset construction \citep{Huang2020-cb} and matrix sketching~\citep{Cohen2017-th} play analogous roles in online $(k,z)$-clustering and online low-rank matrix approximation, as detailed in \cref{sec:dong-yoshida-applications}.
More broadly, sparsifiers for sums of functions such as $p$-norms and generalized linear models \citep{Jambulapati2023-hr,Jambulapati2024-by} provide promising candidates for further applications once the corresponding stability bounds are verified.
This concrete route from offline compression algorithms to small-loss regret guarantees under the random-order model constitutes one conceptual contribution of our work.\looseness=-1

\section{Conclusion and limitations}\label{sec:conclusion}
Starting from the fixed-$\varepsilon$, $(1+\varepsilon)$-approximate-regret framework of \citet{Dong2023-yh}, we have established a small-loss regret bound typically of order $\tilde{O}(\varphi^{\star}(\mathrm{OPT}_T))$ under the random-order model.
Our core idea is to allow approximation accuracy $\varepsilon_t$ to change over time, which leads to a novel decomposition of the regret into the accuracy--stability trade-off and the parameter-change cost parts.
We bound both parts by setting $\varepsilon_t$ to optimize the former trade-off part and leveraging stability of offline algorithms with respect to parameter changes, together with the properties of the concave conjugate~$\varphi^{\star}$.
Crucially, our framework does not require structural assumptions on loss functions, such as linearity or smoothness, and applies broadly to any problems admitting offline approximation algorithms with bounded average sensitivity and parameter-change stability.
We have instantiated our framework for various problems, including online submodular function minimization and online $\ell_1$ regression.
Beyond these specific applications, we have discussed how our framework can serve as a general recipe that connects widely used algorithmic compression techniques---such as sparsifiers, sampling, coresets, and sketching---with small-loss regret bounds in online learning under the random-order model.\looseness=-1

We close with a few remarks on limitations and open questions.
First, our results are limited to the random-order model.
As a generalization, the \emph{stochastically extended adversarial model} has attracted growing attention \citep{Sachs2022-an,Chen2024-pi,Wang2024-ej}, and understanding how our average sensitivity-based approach extends to this more general model is an intriguing direction for future work.
Second, our method runs an offline algorithm on the data observed so far at each round, and the resulting computational cost can be high.
Integrating dynamic or incremental implementations of the required sparsifiers, sketches, or sampling distributions \citep{Soma2024-pt,Woodruff2025-na} is an important direction for improving practicality.
Third, the tightness of our small-loss regret bounds remains open.
Taking online submodular function minimization as an example, we can show a lower bound of $\Omega(\sqrt{n\OPT_T})$ similar to the standard online convex optimization (see \cref{asec:proof-sqrt-lower-bound} for details).
This applies when $\OPT_T\asymp T$ and hence matches the $O(\sqrt{nT})$ upper bound \citep{Hazan2012-hc}; if we do not care about efficiency, the multiplicative weights update with~$2^n$ experts always achieves $O(\sqrt{n\OPT_T}+n)$.
However, focusing on polynomial-time algorithms in the regime of $\OPT_T=o(T)$, exploring improvements over our bound of $\tilde O\prn[\big]{n^3+n^{3/4}\OPT_T^{3/4}}$ is an interesting question for future work.
Finally, although we have verified the parameter-change stability requirement in \cref{assump:adaptive-pc-stability} for all relevant applications, it remains desirable to identify more intuitive and easily checkable sufficient conditions for this assumption. Such an investigation may further clarify what properties of offline algorithms enable small-loss regret bounds through our framework.\looseness=-1

\section*{Acknowledgements}
Shinsaku Sakaue was supported by JST BOOST Program Grant Number JPMJBY24D1.
Yuichi Yoshida was supported by JSPS KAKENHI Grant Number 24K02903.
{\newrefcontext[sorting=nyt]
\printbibliography[heading=bibintoc]}
\todo{to be checked}
\newrefcontext[sorting=ynt]

\clearpage
\appendix
\numberwithin{theorem}{section}

\section{Proof of Lemma~\ref{lem:rom-lem}}\label{sec:rom-lem-proof}
\romlemma*
\begin{proof}
Let $\theta^\star \in \Theta$ be a minimizer that attains $\OPT_T$ on $X_T = \set*{x_t}_{t=1}^T$.
For any $t \in [T]$ and permutation~$\pi$ of $\set*{x_t}_{t=1}^T$, we have
\[
  \OPT_t = \min_{\theta\in\Theta} \sum_{i=1}^t \ell\prn*{\theta, x_{\pi(i)}}
\le \sum_{i=1}^{t} \ell\prn*{\theta^\star, x_{\pi(i)}}.
\]
Consider taking expectation over the uniformly random permutation $\pi$.
As each $x \in X_T$ appears in the first $t$ positions with probability $\frac{t}{T}$, the right-hand side equals $\frac{t}{T}\sum_{x \in X_T} \ell(\theta^\star,x)=\frac{t}{T}\OPT_T$, hence
\[
\E\brc*{\OPT_t} \le \frac{t}{T}\OPT_T.
\]
Dividing by $t$ and summing for $t=1,\dots,T$ yield the claim.
\end{proof}

\section{Proof of Proposition~\ref{prop:regret-decomp}}\label[appendix]{asec:proofs-bto}
The following lemma is central to the derivation of the regret decomposition in \cref{prop:regret-decomp}.
\begin{lemma}[Single-step bound]
\label[lemma]{lem:onestep}
Let $t \in \set{1,\dots,T-1}$ and $\theta_{t+1}\in\Theta$ denote the output of $\mathcal{A}_{\varepsilon_t}$ in \cref{alg:bto}.
Suppose $\varepsilon_t=\varepsilon_{X_t}$ is specified by a deterministic measurable rule $Y\mapsto\varepsilon_Y$ that depends only on the multiset $Y$.
Then, it holds that
\begin{align}
  &\E\brc*{\ell\prn*{\theta_{t+1},x_{t+1}}}
  \\
  \le{}&
\E\brc*{\frac{1+\varepsilon_t}{t}\,\OPT_t + \frac{\varphi(\varepsilon_t)}{t} +
\frac{2}{t}\sum_{x\in X_t}
\mathrm{TV}\prn*{
  \mathcal A_{\varepsilon_{X_t}}\prn*{X_t\setminus x},
  \mathcal A_{\varepsilon_{X_t\setminus x}}\prn*{X_t\setminus x}
}} + O(1/T).
\end{align}
\end{lemma}
\begin{proof}
Fix $t\in\{1,\ldots,T-1\}$ and let $X_t=\set*{x_1,\ldots,x_t}$.
Taking expectation over the randomness of the offline algorithms and the random order of $\set*{x_t}_{t=1}^T$, we start from
\begin{equation}\label{eq:start-from}
  \E\brc*{\ell(\theta_{t+1},x_{t+1})-\frac{1}{t} \ell(\theta_{t+1},X_t)}
=
\E\brc*{\frac{1}{t}\sum_{x\in X_t}\prn*{\ell(\mathcal A_{\varepsilon_{X_t}}(X_t),x_{t+1})-\ell(\mathcal A_{\varepsilon_{X_t}}(X_t),x)}},
\end{equation}
which measures how far the fresh loss $\ell(\theta_{t+1},x_{t+1})$ can drift from the empirical average over $X_t$.
For $x \in X_t$, add and subtract the corresponding loss term obtained by running the algorithm on $X_t\setminus x$ with the parameter recomputed from $X_t\setminus x$, namely $\ell(\mathcal A_{\varepsilon_{X_t\setminus x}}(X_t\setminus x),x)$, inside each summand on the right-hand side to obtain
\begin{align}
  \eqref{eq:start-from}
  =
  &\E\brc*{\frac{1}{t}\sum_{x\in X_t}\prn*{\ell(\mathcal A_{\varepsilon_{X_t}}(X_t),x_{t+1})-\ell(\mathcal A_{\varepsilon_{X_t\setminus x}}(X_t\setminus x),x)}}
  \\
  &+\E\brc*{\frac{1}{t}\sum_{x\in X_t}\prn*{\ell(\mathcal A_{\varepsilon_{X_t\setminus x}}(X_t\setminus x),x)-\ell(\mathcal A_{\varepsilon_{X_t}}(X_t),x)}}.
\end{align}
Here, the element removed uniformly from $X_t$ has the same distribution as $x_{t+1}$ under the random-order model.
Formally, sample an index $I\sim\mathrm{Unif}(\{1,\ldots,t\})$ and set $x=x_I$ and $U=X_t\setminus x_I$.
Conditioned on $U$, the two held-out points $x_I$ and $x_{t+1}$ are exchangeable under a uniformly random permutation of $X_{t+1}$
(they are exactly the two elements of $X_{t+1}\setminus U$).
Hence, recalling that $\varepsilon_U$ is a function only of the multiset $U$, we have $\E[\ell(\mathcal A_{\varepsilon_U}(U),x_I)\mid U]=\E[\ell(\mathcal A_{\varepsilon_U}(U),x_{t+1})\mid U]$, which allows us to replace $x$ in the first expectation by $x_{t+1}$.\footnote{
  The point of the argument is to use the parameter recomputed from the reduced dataset when invoking exchangeability, and then pay for the difference from the original parameter as the separate parameter-change term.
}
Therefore, we obtain
\begin{align}
  &\E\brc*{\ell(\theta_{t+1},x_{t+1})-\frac{1}{t} \ell(\theta_{t+1},X_t)}
  \\
  ={}&
  \E\brc*{\frac{1}{t}\!\sum_{x\in X_t}\prn*{\ell(\mathcal A_{\varepsilon_{X_t}}(X_t),x_{t+1})\!-\!\ell(\mathcal A_{\varepsilon_{X_t\setminus x}}(X_t\!\setminus\!x),x_{t+1})}}
  \\
  &+\E\brc*{\frac{1}{t}\!\sum_{x\in X_t}\prn*{\ell(\mathcal A_{\varepsilon_{X_t\setminus x}}(X_t\!\setminus\!x),x)\!-\!\ell(\mathcal A_{\varepsilon_{X_t}}(X_t),x)}}. \label{eq:onestep-two-terms}
\end{align}
Here, the range of $\ell(\cdot, x')$ is $[0,1]$ for any $x' \in \mathcal{X}$, which implies
\begin{align}
    \abs*{\E\brc*{\ell(\mathcal{A}_{\varepsilon_{X_t}}(X_t), x') - \ell(\mathcal{A}_{\varepsilon_{X_t\setminus x}}(X_t\setminus x), x')}}
  \le
  \mathrm{TV}\prn*{\mathcal{A}_{\varepsilon_{X_t}}(X_t),\mathcal{A}_{\varepsilon_{X_t\setminus x}}(X_t\setminus x)}.
\end{align}
Therefore, the above two components are bounded above by $\mathrm{TV}\prn*{\mathcal A_{\varepsilon_{X_t}}(X_t),\mathcal A_{\varepsilon_{X_t\setminus x}}(X_t\setminus x)}$ in expectation.
We next split this total variation distance by the triangle inequality into a fixed-$\varepsilon$ deletion part and a term comparing two parameters on the same reduced dataset:
\begin{align}
  &\mathrm{TV}\prn*{\mathcal A_{\varepsilon_{X_t}}(X_t),\mathcal A_{\varepsilon_{X_t\setminus x}}(X_t\setminus x)}
  \\
\le{}&
\mathrm{TV}\prn*{\mathcal A_{\varepsilon_{X_t}}(X_t),\mathcal A_{\varepsilon_{X_t}}(X_t\setminus x)}
+
\mathrm{TV}\prn*{\mathcal A_{\varepsilon_{X_t}}(X_t\setminus x),\mathcal A_{\varepsilon_{X_t\setminus x}}(X_t\setminus x)}.
\end{align}
Consequently, using the fixed-$\varepsilon_t$ average sensitivity bound in \cref{assump:approx-avesen}, we obtain\looseness=-1
\[
\E\brc*{\ell(\theta_{t+1},x_{t+1})}
\le
\E\brc*{\frac{1}{t} \ell(\theta_{t+1},X_t)
+
\frac{\varphi(\varepsilon_t)}{t}
+
\frac{2}{t}\sum_{x\in X_t}
\mathrm{TV}\prn*{
\mathcal A_{\varepsilon_{X_t}}\prn*{X_t\setminus x},
\mathcal A_{\varepsilon_{X_t\setminus x}}\prn*{X_t\setminus x}
}},
\]
where the factor $2$ comes from the two terms on the right-hand side in \eqref{eq:onestep-two-terms}.
By $(1+\varepsilon_t)$-approximation of $\mathcal A_{\varepsilon_t}$ in \cref{assump:approx-avesen}, for any $X_t \in \mathcal{X}^t$, we have
\[
\E_{\mathcal A_{\varepsilon_t}}\brc*{\ell(\theta_{t+1},X_t)}
=\E_{\mathcal A_{\varepsilon_t}}\brc*{\ell(\mathcal A_{\varepsilon_t}(X_t),X_t)}
\le (1+\varepsilon_t) \OPT_t + O(t/T).
\]
Combining the above two inequalities taking expectation over the random order of $\set*{x_t}_{t=1}^T$, we obtain
\begin{align}
  &\E\brc*{\ell(\theta_{t+1},x_{t+1})}
  \\
  \le{}&
  \E\brc*{\frac{1+\varepsilon_t}{t}\,\OPT_t +  \frac{\varphi(\varepsilon_t)}{t}
  +
  \frac{2}{t}\sum_{x\in X_t}
  \mathrm{TV}\prn*{
  \mathcal A_{\varepsilon_{X_t}}\prn*{X_t\setminus x},
  \mathcal A_{\varepsilon_{X_t\setminus x}}\prn*{X_t\setminus x}
}} + O(1/T),
\end{align}

completing the proof.
\end{proof}

By using the above single-step bound and the bound on the sum of average optimal values in \cref{lem:rom-lem}, we can now prove the regret decomposition in \cref{prop:regret-decomp}.
\regretdecomprop*
\begin{proof}
  Summing both sides of \cref{lem:onestep} for $t\!=\!1,\dots,T\!-\!1$ and using $\E\brc*{\frac{1+\varepsilon_T}{T}\,\OPT_T\!+\!\frac{\varphi\prn*{\varepsilon_T}}{T}}\!\ge\!0$ yield the following inequality up to an additive $O(1)$ term:
  \begin{align}
  &\E\brc*{\sum_{t=1}^{T-1} \ell(\theta_{t+1},x_{t+1})}
  \\
  \le{}&
  \sum_{t=1}^{T-1}
  \prn*{
    \E\brc*{\frac{1+\varepsilon_t}{t}\,\OPT_t +\frac{\varphi\prn*{\varepsilon_t}}{t}}
    +
    \E\brc*{
    \frac{2}{t}\sum_{x\in X_t}
    \mathrm{TV}\prn*{
      \mathcal A_{\varepsilon_{X_t}}\prn*{X_t\setminus x},
      \mathcal A_{\varepsilon_{X_t\setminus x}}\prn*{X_t\setminus x}
    }}
  }
  \\
  \le{}&
  \sum_{t=1}^{T}
  \prn*{
    \E\brc*{\frac{1+\varepsilon_t}{t}\,\OPT_t +\frac{\varphi\prn*{\varepsilon_t}}{t}}
    +
    \E\brc*{
    \frac{2}{t}\sum_{x\in X_t}
    \mathrm{TV}\prn*{
      \mathcal A_{\varepsilon_{X_t}}\prn*{X_t\setminus x},
      \mathcal A_{\varepsilon_{X_t\setminus x}}\prn*{X_t\setminus x}
    }}
  }.
  \end{align}
  The left-hand side is bounded below by $\E\brc*{\sum_{t=1}^T \ell(\theta_{t},x_{t})} - 1$, while \cref{lem:rom-lem} implies that the right-hand side is bounded above by
  \[
  \OPT_T
  +
  \sum_{t=1}^{T}
  \prn*{
    \E\brc*{\frac{\varepsilon_t}{t}\,\OPT_t + \frac{\varphi\prn*{\varepsilon_t}}{t}}
    +
    \E\brc*{
    \frac{2}{t}\sum_{x\in X_t}
    \mathrm{TV}\prn*{
      \mathcal A_{\varepsilon_{X_t}}\prn*{X_t\setminus x},
      \mathcal A_{\varepsilon_{X_t\setminus x}}\prn*{X_t\setminus x}
    }}
  }.
  \]
  Rearranging the terms completes the proof.
\end{proof}

\section{Proof of Lemma~\ref{lem:concave-conjugate-prop}}\label[appendix]{asec:proof-concave-conjugate-prop}
\concaveconjugateproplemma*
\begin{proof}
  First, $\varphi^\star$ is non-decreasing:
  if $0<u_1<u_2$, then for every $\varepsilon\in(0,1]$, we have
  $u_1\varepsilon+\varphi(\varepsilon)
  \le
  u_2\varepsilon+\varphi(\varepsilon)$,
  which implies
\[
  \varphi^\star(u_1)
  \le
  u_1\varepsilon_{\min}(u_2)+\varphi(\varepsilon_{\min}(u_2))
  \le
  u_2\varepsilon_{\min}(u_2)+\varphi(\varepsilon_{\min}(u_2))
  =
  \varphi^\star(u_2).
\]
  The concavity of $\varphi^\star$ follows from the fact that it is defined as the pointwise minimum of affine functions $u\mapsto u\varepsilon+\varphi(\varepsilon)$ parametrized by~$\varepsilon$.
  The selected minimizer satisfies $\varepsilon_{\min}(u)\in(0,1]$ as the effective domain of $\varphi$ is contained in $(0,1]$, and it is non-increasing in $u$: if $0<u_1<u_2$ and $\varepsilon_i=\varepsilon_{\min}(u_i)$, then adding the two optimality inequalities,
  \[
  u_1\varepsilon_1+\varphi(\varepsilon_1)\le u_1\varepsilon_2+\varphi(\varepsilon_2)
  \qquad\text{and}\qquad
  u_2\varepsilon_2+\varphi(\varepsilon_2)\le u_2\varepsilon_1+\varphi(\varepsilon_1),
  \]
  yields $0\le (u_2-u_1)(\varepsilon_1-\varepsilon_2)$, hence $\varepsilon_1\ge\varepsilon_2$.
  Finally, we verify the supergradient claim.
  Fix $u>0$ and write $\varepsilon_u=\varepsilon_{\min}(u)$.
  By the optimality of $\varepsilon_u$, for any $v>0$, we have
  \[
  \varphi^\star(v)
  =
  \min_{\varepsilon\in(0,1]}
  \{v\varepsilon+\varphi(\varepsilon)\}
  \le
  v\varepsilon_u+\varphi(\varepsilon_u)
  =
  \varphi^\star(u)+\varepsilon_u(v-u).
  \]
This is exactly the supergradient inequality for the concave function
$\varphi^\star$ at $u$.
\end{proof}

\section{Proof of Lemma~\ref{lem:pc-bound}}\label[appendix]{asec:proof-pc-bound}
\pcboundlemma*
\begin{proof}
  The $t=1$ summand is at most $2$ and is absorbed into the $O(1)$ term.
  For $t\ge2$, \cref{assump:adaptive-pc-stability} applied to $X_t$ gives
  \[
  \frac{2}{t}\sum_{x\in X_t}
  \mathrm{TV}\prn*{
    \mathcal A_{\varepsilon_{X_t}}\prn*{X_t\setminus x},
    \mathcal A_{\varepsilon_{X_t\setminus x}}\prn*{X_t\setminus x}
  }
  \le
  \frac{C_{\rm pc}\varepsilon_t}{t}\sum_{x\in X_t}
  \prn*{\OPT_t-\OPT(X_t\setminus x)}.
  \]
  Conditioned on $X_t$, the prefix $X_{t-1}$ is distributed as $X_t\setminus x$ with $x \in X_t$ being a uniformly chosen occurrence, and~$\varepsilon_t$ depends only on $X_t$. Thus, we have
 $
 \E\Brc*{\OPT_{t-1}}{X_t} = \frac1t\sum_{x\in X_t}\OPT(X_t\setminus x)
 $, and hence taking expectation over $X_t$ gives
  \[
  \E\brc*{
  \frac{C_{\rm pc}\varepsilon_t}{t}\sum_{x\in X_t}
  \prn*{\OPT_t-\OPT(X_t\setminus x)}
  }
  =
  C_{\rm pc}\E\brc*{\varepsilon_t(\OPT_t-\OPT_{t-1})}.
  \]
  Summing over $t$ completes the proof.
\end{proof}

\section{Proof of Corollary~\ref{cor:simple-form}}\label[appendix]{asec:proof-cor-simple-form}
\simpleformcorollary*
\begin{proof}
  Recall that $\varphi(\varepsilon)=C_1\varepsilon^{-q}\log(\mathrm{e} + C_2/\varepsilon)$ on $(0,1]$ and $\varphi(\varepsilon)=+\infty$ outside $(0,1]$; this choice of $\varphi$ satisfies the conditions in \cref{assump:approx-avesen}.
  Let $L_0=\log(\mathrm{e}+C_2)$.
  For $u>0$, set $\widehat\varepsilon(u)=\min\{1,(qC_1/u)^{1/(q+1)}\}$.
  Substituting this feasible value into $u\varepsilon+\varphi(\varepsilon)$ gives
  \[
    \varphi^{\star}(u)
    \le
    C_1L_0
    +
    q^{-\frac{q}{q+1}}C_1^{\tfrac{1}{q+1}}u^{\tfrac{q}{q+1}}
    \prn*{q + \log\prn*{\mathrm{e} + {C_2}{q^{-\tfrac{1}{q+1}} C_1^{-\tfrac{1}{q+1}}} u^{\tfrac{1}{q+1}} }}.
  \]
  Substituting $u = \frac{\OPT_T}{H_T} + \lambda$ into the above yields the following upper bound on $\varphi^{\star}(u)$:
  \[
  C_1L_0
   +
  q^{-\tfrac{q}{q+1}}C_1^{\tfrac{1}{q+1}}\prn*{\frac{\OPT_T}{H_T} + \lambda}^{\tfrac{q}{q+1}}
  \prn*{q + \log\prn*{\mathrm{e} + {C_2}{q^{-\tfrac{1}{q+1}} C_1^{-\tfrac{1}{q+1}}} \prn*{\frac{\OPT_T}{H_T} + \lambda}^{\tfrac{1}{q+1}}}\!\!}.
  \]
  By monotonicity and subadditivity of $u\mapsto u^\alpha$ for $\alpha\in(0,1)$, and $\lambda \le H_T^{-(q+1)/q} \le 1$, we have
  \begin{equation}
    \begin{aligned}
      \prn*{\frac{\OPT_T}{H_T} + \lambda}^{\tfrac{q}{q+1}} \le \prn*{\frac{\OPT_T}{H_T}}^{\tfrac{q}{q+1}} + H_T^{-1} 
      &&
      \text{and}
      &&
      \prn*{\frac{\OPT_T}{H_T} + \lambda}^{\tfrac{1}{q+1}}
      \le
      \prn*{1 + \frac{\OPT_T}{H_T}}^{\tfrac{1}{q+1}}.
    \end{aligned}
  \end{equation}
  Therefore, we obtain the following upper bound on $H_T\varphi^{\star}(u)$:
  \[
  C_1H_TL_0
    +
    q^{-\frac{q}{q+1}}C_1^{\tfrac{1}{q+1}}
    \prn*{
      1 + \OPT_T^{\tfrac{q}{q+1}} H_T^{\tfrac{1}{q+1}}
    }
    \prn*{q + \log\prn*{\mathrm{e} + {C_2}
    \prn*{\frac{1 + {\OPT_T}/{H_T}}{q C_1}}^{\tfrac{1}{q+1}} }}.
  \]
  It remains to bound the second term in \cref{thm:general-separable}.  Since $\varphi^\star$ is non-negative and non-decreasing,
  \[
  \varphi^\star(\OPT_T+\lambda)-\varphi^\star(\lambda)
  \le
  \varphi^\star(\OPT_T+\lambda).
  \]
  Applying the same estimate for $\varphi^\star(u)$ with $u=\OPT_T+\lambda$ and using $\lambda\le1$ gives the following upper bound on $\varphi^\star(\OPT_T+\lambda)$:
  \[
  C_1L_0
  +
  q^{-\frac{q}{q+1}}C_1^{\frac1{q+1}}
  \prn*{1+\OPT_T^{\frac q{q+1}}}
  \prn*{
    q+
    \log\prn*{
      \mathrm e+
      C_2
      \prn*{\frac{1+\OPT_T}{qC_1}}^{\frac1{q+1}}
    }
  }.
  \]
  Using these two bounds in \cref{thm:general-separable} and hiding logarithmic factors in $\tilde O$ completes the proof.
\end{proof}

\section{Small-loss inconsistency bound}\label[appendix]{asec:small-loss-inconsistency}
We show that the parameter selection rule in~\eqref{eq:rule} also
yields a small-loss bound on inconsistency. We measure inconsistency following \citet{Dong2023-yh} by counting changes in the sequence of outputs played by the learner:
\[
  \mathrm{Inc}_T
  \coloneqq
  \E_{\mathsf{Alg},\set*{x_t}}
  \brc*{
    \sum_{t=1}^{T-1}
    \ones\set*{\theta_t\neq \theta_{t+1}}
  }.
\]
We first state the probability-transformation requirement.
For two output laws $P,Q$ on $\Theta$, a probability transformation from $P$ to $Q$ means a randomized map which, given an input distributed as $P$, outputs a random variable distributed as $Q$; its change probability is the probability that the input and output values differ.
The fixed-parameter deletion transformation below is the same probability
transportation requirement used by \citet{Dong2023-yh}, following
\citet[Definition~4.1 and Lemma~4.5]{Yoshida2022-pb}.
The new ingredient in our adaptive setting is the
same-input parameter-change transformation, which is needed only because
$\varepsilon_X$ and $\varepsilon_{X\setminus x}$ may differ.
\begin{assumption}[Probability transformations]
\label[assumption]{assump:prob-transform-inconsistency}
Fix $\lambda>0$ and the rule
$\varepsilon_Y=\varepsilon_{\min}(\OPT(Y)+\lambda)$ used in~\eqref{eq:rule}.
For every multiset $X$ with $|X|\ge2$ and every occurrence $x\in X$, assume
that the following two probability transformations are available.
\begin{enumerate}[leftmargin=*,itemsep=2pt]
\item
There is a same-input parameter-change transformation
$K^{\rm pc}_{X,x}$ such that, if
$
  \theta^{-}\!\sim\!
  \mathcal A_{\varepsilon_{X\setminus x}}(X\setminus x)
$
and $ \theta^0$ is the output of $K^{\rm pc}_{X,x}$ on input $\theta^{-}$,
then it holds that
\[
   \theta^0
  \sim
  \mathcal A_{\varepsilon_X}(X\setminus x)
  \qquad\text{and}\qquad
  \Pr\brc*{ \theta^0\neq\theta^{-}}
  \le
  \mathrm{TV}\prn*{
    \mathcal A_{\varepsilon_{X\setminus x}}(X\setminus x),
    \mathcal A_{\varepsilon_X}(X\setminus x)
  }.
\]
\item
There is a fixed-parameter deletion transformation
$K^{\rm del}_{X,x}$ such that, if
$
   \theta^0\sim
  \mathcal A_{\varepsilon_X}(X\setminus x)
$
and $\theta^{+}$ is the output of $K^{\rm del}_{X,x}$ on input $ \theta^0$,
then it holds that
\[
  \theta^{+}
  \sim
  \mathcal A_{\varepsilon_X}(X)
  \qquad\text{and}\qquad
  \Pr\brc*{\theta^{+}\neq \theta^0}
  \le
  \mathrm{TV}\prn*{
    \mathcal A_{\varepsilon_X}(X\setminus x),
    \mathcal A_{\varepsilon_X}(X)
  }.
\]
\end{enumerate}
\end{assumption}
The theorem below is stated for the coupled implementation:
after the first prefix, each new output is obtained from the previous one by
applying the transformations in
\cref{assump:prob-transform-inconsistency}. This preserves the same marginal law
$\mathcal A_{\varepsilon_t}(X_t)$ at every prefix, and thus the regret analysis in
\cref{thm:general-separable} is unchanged.

\begin{theorem}[Small-loss inconsistency bound]
\label{thm:small-loss-inconsistency}
Under the random-order model and \cref{assump:approx-avesen}, fix
$\lambda>0$ and suppose that \cref{assump:adaptive-pc-stability} and
\cref{assump:prob-transform-inconsistency} hold for this $\lambda$.
Then, the coupled implementation of \cref{alg:bto} with
$\varepsilon_t$ given by~\eqref{eq:rule} satisfies
\[
  \mathrm{Inc}_T
  \le
  \frac{H_T}{2}
  \varphi^\star\prn*{\frac{\OPT_T}{H_T}+\lambda}
  +
  \frac{C_{\rm pc}}{2}
  \prn*{
    \varphi^\star(\OPT_T+\lambda)-\varphi^\star(\lambda)
  }
  + O(1).
  \]
\end{theorem}
\begin{proof}
We first provide the marginal invariant of the coupled implementation. Let $\mathcal H_s=(x_1,\dots,x_s)$ denote the ordered data history up to time $s$.
We claim that, for every $s\ge1$,
\begin{equation}\label{eq:inv}
  \theta_{s+1}\mid \mathcal H_s
  \sim
  \mathcal A_{\varepsilon_{X_s}}(X_s).
\end{equation}
Indeed, for $s=1$, this holds by the initialization step after observing
$x_1$. Suppose that the invariant holds at time $t-1$, namely
\[
  \theta_t\mid \mathcal H_{t-1}
  \sim
  \mathcal A_{\varepsilon_{X_{t-1}}}(X_{t-1}).
\]
After observing $x_t$, write $X=X_t$ and let the newly added occurrence be
$x\in X$, so that $X_{t-1}=X\setminus x$. By first applying
$K^{\rm pc}_{X,x}$ and then $K^{\rm del}_{X,x}$, the output $\theta_{t+1}$
has marginal law
\[
  \theta_{t+1}\mid \mathcal H_t
  \sim
  \mathcal A_{\varepsilon_X}(X)
  =
  \mathcal A_{\varepsilon_{X_t}}(X_t).
\]
Thus, the coupled implementation preserves the same one-prefix marginal law as
the independent rerun version of \cref{alg:bto}.

The term at $t=1$ in the definition of $\mathrm{Inc}_T$ is at most one. We now
control the remaining changes. Fix $t\in\set*{2,\dots,T-1}$ and condition on
$X_t=X$. Under the random-order model, the previous prefix $X_{t-1}$ is
distributed as $X\setminus x$ with $x$ chosen uniformly from the multiset $X$.
For each such occurrence $x$, the invariant \eqref{eq:inv} implies
\[
  \theta_t
  \mid
  \{X_t=X,\; X_{t-1}=X\setminus x\}
  \sim
  \mathcal A_{\varepsilon_{X\setminus x}}(X\setminus x).
\]
Indeed, \eqref{eq:inv} holds conditional on the full ordered history
$\mathcal H_{t-1}$, and the algorithmic randomness used to generate
$\theta_t$ is independent of the next unrevealed occurrence once
$\mathcal H_{t-1}$ is fixed.
For this fixed $X$ and $x$, the union bound and
\cref{assump:prob-transform-inconsistency} give
\[
\begin{aligned}
  &\Pr\brc*{\theta_t\neq\theta_{t+1}\mid X_t=X,\;X_{t-1}=X\setminus x}
  \\
  \le{}&
  \mathrm{TV}\prn*{
    \mathcal A_{\varepsilon_{X\setminus x}}(X\setminus x),
    \mathcal A_{\varepsilon_X}(X\setminus x)
  }
  +
  \mathrm{TV}\prn*{
    \mathcal A_{\varepsilon_X}(X\setminus x),
    \mathcal A_{\varepsilon_X}(X)
  }.
\end{aligned}
\]
Averaging over the uniformly deleted occurrence $x\in X$, using the symmetry of
total variation, and then applying \cref{assump:approx-avesen} with the fixed
parameter $\varepsilon_X$, we obtain
\[
\begin{aligned}
  &\Pr\brc*{\theta_t\neq\theta_{t+1}\mid X_t=X}
  \\
  \le{}&
  \frac{1}{|X|}
  \sum_{x\in X}
  \mathrm{TV}\prn*{
    \mathcal A_{\varepsilon_X}(X),
    \mathcal A_{\varepsilon_X}(X\setminus x)
  }
  +
  \frac{1}{|X|}
  \sum_{x\in X}
  \mathrm{TV}\prn*{
    \mathcal A_{\varepsilon_X}(X\setminus x),
    \mathcal A_{\varepsilon_{X\setminus x}}(X\setminus x)
  }
  \\
  \le{}&
  \frac{\varphi(\varepsilon_X)}{2|X|}
  +
  \frac{1}{|X|}
  \sum_{x\in X}
  \mathrm{TV}\prn*{
    \mathcal A_{\varepsilon_X}(X\setminus x),
    \mathcal A_{\varepsilon_{X\setminus x}}(X\setminus x)
  }.
\end{aligned}
\]
Applying \cref{assump:adaptive-pc-stability} to the second term yields, for
$X=X_t$,
\[
  \Pr\brc*{\theta_t\neq\theta_{t+1}\mid X_t}
  \le
  \frac{\varphi(\varepsilon_t)}{2t}
  +
  \frac{C_{\rm pc}\varepsilon_t}{2t}
  \sum_{x\in X_t}
  \prn*{\OPT_t-\OPT(X_t\setminus x)}.
\]
Taking expectation and using that, conditioned on $X_t$, the prefix
$X_{t-1}$ is distributed as $X_t\setminus x$ for a uniformly chosen occurrence
$x\in X_t$, we get
\[
  \E\brc*{
    \frac{C_{\rm pc}\varepsilon_t}{2t}
    \sum_{x\in X_t}
    \prn*{\OPT_t-\OPT(X_t\setminus x)}
  }
  =
  \frac{C_{\rm pc}}{2}
  \E\brc*{\varepsilon_t(\OPT_t-\OPT_{t-1})}.
\]
Summing over $t=2,\dots,T-1$, adding the possible initial change, and extending
the non-negative sums to $t=1,\dots,T$ give
\[
  \mathrm{Inc}_T
  \le
  \frac12
  \E\brc*{\sum_{t=1}^T\frac{\varphi(\varepsilon_t)}{t}}
  +
  \frac{C_{\rm pc}}{2}
  \E\brc*{\sum_{t=1}^T
    \varepsilon_t(\OPT_t-\OPT_{t-1})
  }
  +O(1),
\]
where $\OPT_0=0$.

The rest of the proof bounds the two sums similar to the proof of \cref{thm:general-separable}.
Since
$\OPT_t\varepsilon_t/t\ge0$, \cref{lem:local-opt-eps} implies
\[
  \sum_{t=1}^T\frac{\varphi(\varepsilon_t)}{t}
  \le
  \sum_{t=1}^T
  \prn*{
    \frac{\OPT_t}{t}\varepsilon_t
    +
    \frac{\varphi(\varepsilon_t)}{t}
  }
  \le
  H_T
  \varphi^\star\prn*{
    \frac{\sum_{t=1}^T \OPT_t/t}{H_T}
    +
    \lambda
  }.
\]
Taking expectation, applying Jensen's inequality to the concave function
$\varphi^\star$, and then using \cref{lem:rom-lem} and the monotonicity of
$\varphi^\star$, we obtain
\[
  \E\brc*{\sum_{t=1}^T\frac{\varphi(\varepsilon_t)}{t}}
  \le
  H_T
  \varphi^\star\prn*{
    \frac{\OPT_T}{H_T}
    +
    \lambda
  }.
\]

For the second sum, set $a_t=\OPT_t+\lambda$ and $a_0=\lambda$.
Since the losses are non-negative, $\OPT_t$ is non-decreasing along every
realized prefix sequence. Hence, by the supergradient property of
$\varepsilon_{\min}$ in \cref{lem:concave-conjugate-prop}, we obtain
\[
  \varepsilon_t(\OPT_t-\OPT_{t-1})
  =
  \varepsilon_{\min}(a_t)(a_t-a_{t-1})
  \le
  \varphi^\star(a_t)-\varphi^\star(a_{t-1}).
\]
Telescoping over $t=1,\dots,T$ gives
\[
  \sum_{t=1}^T
  \varepsilon_t(\OPT_t-\OPT_{t-1})
  \le
  \varphi^\star(\OPT_T+\lambda)-\varphi^\star(\lambda).
\]
Taking expectation of this and combining the two bounds complete the proof.
\end{proof}
This theorem, instantiated with the function $\varphi$ in \cref{cor:simple-form}, specializes to the following small-loss inconsistency bound, which recovers the fixed-$\varepsilon$ bound of \citet{Dong2023-yh}.
\begin{corollary}
\label[corollary]{cor:inconsistency-simple-form}
Let $q>0$ be a constant. Assume
$
  \varphi(\varepsilon)
  =
  C_1\varepsilon^{-q}\log(\mathrm e+C_2/\varepsilon)
$
for $\varepsilon\in(0,1]$ and $\varphi(\varepsilon)=+\infty$ for $\varepsilon\notin(0,1]$, where
$C_1\ge1$ and $C_2\ge0$. Under the assumptions of
\cref{thm:small-loss-inconsistency}, for any
$\lambda\in(0,H_T^{-(q+1)/q}]$, it holds that
\[
  \mathrm{Inc}_T
  =
  \tilde O
  \prn*{
    (1+C_{\rm pc})
    \prn*{
      C_1
      +
      C_1^{1/(q+1)}\OPT_T^{q/(q+1)}
    }
  }.
\]
Moreover, let
$
  \bar\varepsilon
  \coloneqq
  \varepsilon_{\min}(\OPT_T+\lambda)
$.
Then, for any fixed $\varepsilon\in(0,\bar\varepsilon]$, the same bound also implies
\[
  \mathrm{Inc}_T
  =
  O\prn*{(H_T+C_{\rm pc})\varphi(\varepsilon)}.
\]
Thus, at the hindsight scale $\bar\varepsilon$ (or the known scale $\varepsilon_{\min}(T+\lambda)$ due to the monotonicity of $\varepsilon_{\min}$ and $\ell(\cdot, x_t) \in [0,1]$) and under $C_{\rm pc}=O(1)$, the coupled implementation recovers the
fixed-$\varepsilon$ inconsistency bound of \citet{Dong2023-yh}.
\end{corollary}

\begin{proof}
The first bound follows from the same estimate on $\varphi^\star$ as in
\cref{cor:simple-form}. Indeed, the bound in
\cref{thm:small-loss-inconsistency} consists of
\[
  \frac{H_T}{2}
  \varphi^\star\prn*{\frac{\OPT_T}{H_T}+\lambda}
  \qquad\text{and}\qquad
  \frac{C_{\rm pc}}{2}
  \prn*{
    \varphi^\star(\OPT_T+\lambda)-\varphi^\star(\lambda)
  },
\]
which are the same two quantities appearing in \cref{thm:general-separable},
up to absolute constants. Therefore, applying the proof of
\cref{cor:simple-form} gives
\[
  \mathrm{Inc}_T
  =
  \tilde O
  \prn*{
    (1+C_{\rm pc})
    \prn*{
      C_1
      +
      C_1^{1/(q+1)}\OPT_T^{q/(q+1)}
    }
  }.
\]
We next prove the fixed-$\varepsilon$ comparison. Since the losses are
non-negative, $\OPT_t\le\OPT_T$ holds for every prefix. Hence, by the monotonicity
of $\varepsilon_{\min}$ in \cref{lem:concave-conjugate-prop}, it holds that
\[
  \varepsilon_t
  =
  \varepsilon_{\min}(\OPT_t+\lambda)
  \ge
  \varepsilon_{\min}(\OPT_T+\lambda)
  =
  \bar\varepsilon
  \ge
  \varepsilon.
\]
The function $\varphi$ is decreasing on $(0,1]$, hence
$\varphi(\varepsilon_t)\le\varphi(\varepsilon)$ for every $t$. Therefore, the
fixed-parameter deletion part in the proof of \cref{thm:small-loss-inconsistency} is bounded directly as
\[
  \frac12
  \sum_{t=1}^T
  \frac{\varphi(\varepsilon_t)}{t}
  \le
  \frac{H_T}{2}\varphi(\varepsilon).
\]
It remains to control the parameter-change part in the proof of \cref{thm:small-loss-inconsistency} at the same scale. Put
$u=\OPT_T+\lambda$ and $\bar\varepsilon=\varepsilon_{\min}(u)$. By optimality, we have
$
  \varphi^\star(u)
  =
  u\bar\varepsilon+\varphi(\bar\varepsilon)
$.
We claim
\[
  u\bar\varepsilon
  =
  O_q\prn*{\varphi(\bar\varepsilon)},
\]
where the hidden constant depends only on $q$. Indeed, by comparing the optimal value at $\bar\varepsilon$ with the feasible
point $\bar\varepsilon/2$, we have
$
u\bar\varepsilon+\varphi(\bar\varepsilon)
  \le
  \frac{u\bar\varepsilon}{2}
  +
  \varphi(\bar\varepsilon/2)
$. Thus, we obtain
$
  \frac{u\bar\varepsilon}{2}
  \le
  \varphi(\bar\varepsilon/2)-\varphi(\bar\varepsilon)
  \le
  \varphi(\bar\varepsilon/2)
$.
For the specific form of $\varphi$, we have
\[
  \varphi(\bar\varepsilon/2)
  =
  2^q C_1\bar\varepsilon^{-q}
  \log(\mathrm e+2C_2/\bar\varepsilon)
  =
  O_q\prn*{\varphi(\bar\varepsilon)},
\]
where the last step uses
$\log(\mathrm e+2z)=O(\log(\mathrm e+z))$ for $z\ge0$.
Therefore, $u\bar\varepsilon=O_q(\varphi(\bar\varepsilon))$ holds as claimed.
Consequently, by $\varepsilon\le\bar\varepsilon$ and the monotonicity of $\varphi$, we obtain
\[
  \varphi^\star(u)
  =
  u\bar\varepsilon+\varphi(\bar\varepsilon)
  =
  O_q\prn*{\varphi(\varepsilon)}.
\]
Since
$
  \varphi^\star(\OPT_T+\lambda)-\varphi^\star(\lambda)
  \le
  \varphi^\star(\OPT_T+\lambda)
$,
the parameter-change contribution is
$O_q(C_{\rm pc}\varphi(\varepsilon))$. Combining this with the
$H_T\varphi(\varepsilon)/2$ bound above yields
\[
  \mathrm{Inc}_T
  =
  O_q\prn*{
    (H_T+C_{\rm pc})
    \varphi(\varepsilon)
  }.
\]
Therefore, for constant $q > 0$ and $C_{\rm pc}=O(1)$, this recovers the fixed-$\varepsilon$ inconsistency bound of $O(\varphi(\varepsilon)\log T)$ obtained by \citet{Dong2023-yh}.
\end{proof}

\section{On a naive doubling reduction}\label[appendix]{sec:doubling-trick}
We explain why we do not use a naive doubling reduction based on the fixed-$\varepsilon$ guarantee of \citet{Dong2023-yh} as a substitute for our time-varying approach in \cref{sec:adaptive-control}.
For a fixed value of~$\varepsilon$, their batch-to-online transformation yields a $(1+\varepsilon)$-approximate regret guarantee; equivalently, at the level of standard regret, the bound has the form of $\Reg_T \lesssim \varepsilon\OPT_T + H_T\varphi\prn*{\varepsilon}$ up to constants.
Thus, it is natural to ask whether one can recover a small-loss regret bound by a phase-wise doubling trick that guesses the unknown scale of the phase-local offline optimum and runs the fixed-$\varepsilon$ algorithm separately in each phase.
The obstruction described below is not the deterministic algebra of such a doubling argument, but the random-order exchangeability step needed to justify applying a fixed-horizon guarantee inside phases whose endpoints are data-dependent.

Consider a threshold-based phase scheme, analogous to the standard doubling trick for an unknown scale parameter (see, e.g., \citet[Exercise~2.10]{Cesa-Bianchi2006-oa}).
Conditional on the history before a phase starts, the unrevealed data points are still in uniformly random order.
This observation alone, however, does not justify applying the fixed-horizon proof to a phase with a random endpoint, as we now explain.
In a phase with threshold $B$, one fixes a parameter $\varepsilon$ as if the phase-local optimum were of order $B$, restarts the batch-to-online procedure using only the observations in the current phase, and stops the phase when the phase-local optimum first exceeds $B$.
Let $Y_s=\prn*{y_1,\ldots,y_s}$ denote the phase-local prefix after $s$ revealed data points, and define the stopping time
\[
  \tau \coloneqq \inf \Set*{s\ge0}{\OPT\prn*{Y_s}>B},
  \qquad
  \text{where}
  \;\;
  \inf\emptyset=+\infty.
\]
Because the losses are non-negative, $s\mapsto\OPT\prn*{Y_s}$ is non-decreasing; hence the phase-survival event at local time $s$ is equivalently
\[
  \set*{\tau>s}=\set*{\OPT\prn*{Y_s}\le B}.
\]
Fixing $\varepsilon$ within the phase would remove the parameter-change term inside that phase, but the proof would still have to justify the single-step random-order comparison after conditioning on the survival event $\set*{\tau>s}$.

To see the issue, recall the exchangeability step used in the proof of \cref{lem:onestep}, which also appears in the fixed-$\varepsilon$ proof of \citet[Theorem~4.1]{Dong2023-yh}.
For a fixed, non-random local time~$s\ge1$, sample $I\sim\mathrm{Unif}\prn*{\set*{1,\ldots,s}}$, set $y=y_I$, and put $U=Y_s\setminus y$.
If no stopping event is imposed, then conditional on $U$, the deleted point $y$ and the next point $y_{s+1}$ are exchangeable under the random-order model.
Equivalently, after also conditioning on the unordered pair of held-out points, their two possible relative orders are equally likely.
Since the law of $\mathcal A_\varepsilon(U)$ depends only on $U$, this exchangeability gives the identity that allows the fresh loss on $y_{s+1}$ to be compared with the loss on a uniformly deleted point from $Y_s$.
This is the step used before applying average sensitivity.

In a stopped phase, however, the prediction at local time $s+1$ is made only on the event $\set*{\tau>s}$.
After conditioning on the same reduced prefix $U=Y_s\setminus y$, the survival event is no longer symmetric in the deleted point and the next point.
Indeed, it imposes the constraint
\[
  \OPT\prn*{U\cup\set*{y}} \le B
\]
on the point that has already been included in the prefix, whereas the prospective next point $y_{s+1}$ is not subject to the same constraint before it is revealed.
Consequently, the stopped analogue of the deterministic-time identity may fail:
\[
\E\Brc*{\ell\prn*{\mathcal{A}_{\varepsilon}(U),y}}{U,\tau>s}
\quad\text{need not equal}\quad
\E\Brc*{\ell\prn*{\mathcal{A}_{\varepsilon}(U),y_{s+1}}}{U,\tau>s}.
\]
The corresponding equality does hold without conditioning on phase survival.

A minimal example illustrates the asymmetry.
After conditioning on $U$, suppose that the unordered pair consisting of the deleted point and the next point is $\set*{a,b}$, and that we have
\[
  \OPT\prn*{U\cup\set*{a}} \le B
  \qquad
  \text{and}
  \qquad
  \OPT\prn*{U\cup\set*{b}} > B.
\]
Before conditioning on survival, the two local orderings $\prn*{a,b}$ and $\prn*{b,a}$ are equally likely.
Let
\[
  \mathcal E_{a,b}\coloneqq\set*{\set*{y,y_{s+1}}=\set*{a,b}}.
\]
After conditioning on $\set*{\tau>s}$, only the ordering in which $a$ is already in the prefix can survive; more explicitly,
\[
  \Pr\Brc*{y=a,\ y_{s+1}=b}{U,\mathcal E_{a,b},\tau>s}=1
  \qquad
  \text{and}
  \qquad
  \Pr\Brc*{y=b,\ y_{s+1}=a}{U,\mathcal E_{a,b},\tau>s}=0.
\]
Thus, the uniformly deleted point from a survived prefix is biased toward candidates that keep the phase below the threshold, whereas the next point may be precisely the boundary item that causes the phase to stop.
Conditioning instead on $\set*{\tau>s+1}$ would constrain the next point as well, but then the boundary round crossing the threshold would be excluded from the fixed-horizon comparison and would have to be handled separately.

As we have seen above, the difficulty is not merely the possible unit loss on the threshold-crossing rounds, whose number would be logarithmic in a standard doubling scheme.  The more fundamental issue is that threshold-defined phase boundaries condition every non-boundary prediction in the phase on the survival event $\set*{\tau>s}$, and this conditioning is precisely what invalidates the black-box use of the fixed-horizon leave-one-out exchangeability.
Therefore, the fixed-$\varepsilon$ guarantee cannot be applied to threshold-defined phases through a black-box conditioning argument.
A valid doubling proof would need an additional stopped-process analysis that reestablishes the required leave-one-out comparison, possibly with boundary terms, or it would need to use data-independent phase boundaries.
We avoid this complication in the main proof: the time-varying analysis works at deterministic global times and pays the explicit parameter-change term, which is then controlled by \cref{assump:adaptive-pc-stability}.

\section{Basic tools for average sensitivity analysis}\label[appendix]{sec:basic-tools}
We present basic tools useful for analyzing average sensitivity.
First, we introduce the standard coupling lemma for the total variation distance.
\begin{lemma}[Coupling lemma]\label[lemma]{lem:tv-coupling}
Let $(\Omega,\mathcal F)$ be a measurable space.
Consider probability measures $P$, $Q$ on it, and let $\mathcal C(P,Q)$ be the set of all couplings of $(P,Q)$.
Then, for any coupling $\Gamma\in\mathcal C(P,Q)$, it holds that
\[
  \mathrm{TV}(P,Q)
  \le
  \Pr_{(X,Y)\sim \Gamma}[X\neq Y].
\]
Moreover, we have
\[
  \mathrm{TV}(P,Q)
  =
  \inf_{\Gamma\in\mathcal C(P,Q)} \Pr_{(X,Y)\sim \Gamma}[X\neq Y],
\]
and the infimum is attained by some optimal coupling if $(\Omega,\mathcal F)$ is a standard Borel space.
\end{lemma}
\begin{proof}
See, for example, \citet[Proposition~4.7]{Levin2017-dj}.
\end{proof}

Throughout this paper, we only consider measurable spaces that are standard Borel spaces.
The following lemmas are consequences of the coupling lemma.
\begin{lemma}[Product decomposition]\label[lemma]{lem:product-tv-general}
Let $E$ be a finite set and $\{(\Omega_f,\mathcal{F}_f)\}_{f\in E}$ measurable spaces.
For each $f\in E$, let $P_f$ and $Q_f$ be probability measures on $(\Omega_f,\mathcal{F}_f)$.
Then, the total variation distance between the product measures satisfies
\[
  \mathrm{TV}\prn*{\bigotimes_{f\in E} P_f,\ \bigotimes_{f\in E} Q_f}
  \le
  \sum_{f\in E}\mathrm{TV}(P_f,Q_f).
\]
\end{lemma}
\begin{proof}
For each $f\in E$, let $\Gamma_f$ be a coupling measure of $(P_f, Q_f)$ such that
\[
  \Pr_{(X_f,Y_f)\sim \Gamma_f}[X_f \neq Y_f] = \mathrm{TV}(P_f, Q_f),
\]
which exists by the coupling lemma (\cref{lem:tv-coupling}).
Define the product coupling measure
\[
  \Gamma \coloneqq  \bigotimes_{f\in E} \Gamma_f,
\]
which is a coupling of the product measures $\bigotimes_{f\in E} P_f$ and $\bigotimes_{f\in E} Q_f$.
Let $(X,Y)$ denote a random variable drawn from $\Gamma$, where $X = (X_f)_{f\in E}$ and $Y = (Y_f)_{f\in E}$.
Then, we have $X \sim \bigotimes_{f\in E} P_f$ and $Y \sim \bigotimes_{f\in E} Q_f$.
Under this product coupling, the probability of $X \neq Y$ equals the probability that there exists $f\in E$ such that $X_f \neq Y_f$.
By the union bound, we have
\[
  \Pr_{(X,Y) \sim \Gamma}[X \neq Y]
  \!=\!\!
  \Pr_{(X,Y) \sim \Gamma}\brc*{\exists f\in E: X_f \neq Y_f}
  \le \sum_{f\in E} \Pr_{(X_f,Y_f) \sim \Gamma_f}[X_f \neq Y_f]
  \!=\!
  \sum_{f\in E} \mathrm{TV}(P_f, Q_f).
\]
By using the coupling lemma (\cref{lem:tv-coupling}) again, we lower bound the left-hand side as
\[
  \mathrm{TV}\prn*{\bigotimes_{f\in E} P_f,\ \bigotimes_{f\in E} Q_f}
  \le
  \Pr_{(X,Y) \sim \Gamma}[X \neq Y],
\]
which completes the proof.
\end{proof}

\begin{lemma}[Conditional decomposition]\label[lemma]{lem:tv-decomp}
Let $(\Omega,\mathcal F)$ be a measurable space of the form
$\Omega=\mathcal Z\times\mathcal W$, where $\mathcal Z$ is a countable set
equipped with the discrete $\sigma$-algebra and $\mathcal W$ is a measurable space.
Let $P$ and $Q$ be probability measures on $(\Omega,\mathcal F)$.
Let $P_Z$ and $Q_Z$ denote the marginals of $P$ and $Q$ on $\mathcal Z$, respectively,
and let $P_{W\mid Z=z}$ and $Q_{W\mid Z=z}$ denote the corresponding conditional distributions
on $\mathcal W$ given $Z=z\in\mathcal Z$.
Then, it holds that
\[
  \mathrm{TV}(P,Q)
  \le
  \mathrm{TV}(P_Z,Q_Z)
  +
  \sum_{z\in\mathcal Z}
  \min\{P_Z(z),Q_Z(z)\}
  \,\mathrm{TV}\!\left(P_{W\mid Z=z},\,Q_{W\mid Z=z}\right).
\]
\end{lemma}

\begin{proof}
By the coupling lemma (\cref{lem:tv-coupling}), we take $\Gamma^{\star}\in\mathcal C(P_Z,Q_Z)$ as an optimal coupling of the marginals such that
\[
  \Pr_{(Z,Z')\sim\Gamma^{\star}}[Z\neq Z']
  =
  \mathrm{TV}(P_Z,Q_Z).
\]
Also, for each $z\in\mathcal Z$, let $\Gamma_z\in\mathcal C(P_{W\mid Z=z},Q_{W\mid Z=z})$
be an optimal coupling satisfying
\[
  \Pr_{(W,W')\sim\Gamma_z}[W\neq W']
  =
  \mathrm{TV}\!\left(P_{W\mid Z=z},Q_{W\mid Z=z}\right).
\]
Consider a coupling $\Gamma\in\mathcal C(P,Q)$ constructed as follows.
First, sample $(Z,Z')\sim\Gamma^{\star}$.
If $Z\neq Z'$, sample $W\sim P_{W\mid Z}$ and $W'\sim Q_{W\mid Z'}$ independently.
If $Z=Z'=z$, sample $(W,W')\sim\Gamma_z$.
By construction, $(Z,W)\sim P$ and $(Z',W')\sim Q$ hold.
Under this coupling, we have
\[
  \Pr_{(Z,W,Z',W')\sim \Gamma}\bigl[(Z,W)\neq (Z',W')\bigr]
  \le
  \Pr_{(Z,Z')\sim \Gamma^{\star}}\bigl[Z\neq Z'\bigr]
  +
  \Pr_{(Z,W,Z',W')\sim \Gamma}\bigl[Z=Z',\,W\neq W'\bigr].
\]
The first term equals $\mathrm{TV}(P_Z,Q_Z)$ by the definition of $\Gamma^{\star}$.
For the second term, conditioning on $Z=Z'=z$ yields
\[
  \Pr_{(Z,W,Z',W')\sim \Gamma}\bigl[Z=Z',\,W\neq W'\bigr]
  =
  \sum_{z\in\mathcal Z}
  \Pr_{(Z,Z')\sim \Gamma^{\star}}[Z=Z'=z]\,
  \mathrm{TV}\!\left(P_{W\mid Z=z},Q_{W\mid Z=z}\right).
\]
For any coupling of $(P_Z,Q_Z)$, the diagonal mass at $z$ is at most
$\min\{P_Z(z),Q_Z(z)\}$, hence
\[
  \Pr_{(Z,Z')\sim \Gamma^{\star}}[Z=Z'=z]
  \le
  \min\{P_Z(z),Q_Z(z)\}.
\]
Combining the above bounds with \cref{lem:tv-coupling} gives
\begin{align}
  \mathrm{TV}(P,Q)
  &\le\Pr_{(Z,W,Z',W')\sim \Gamma}\bigl[(Z,W)\neq (Z',W')\bigr]
  \\
  &\le
  \mathrm{TV}(P_Z,Q_Z)
  +
  \sum_{z\in\mathcal Z}
  \min\{P_Z(z),Q_Z(z)\}
  \,\mathrm{TV}\!\left(P_{W\mid Z=z},Q_{W\mid Z=z}\right).
\end{align}
This completes the proof.
\end{proof}

The following bound on the total variation distance between two uniform distributions is also convenient in our analysis.
\begin{lemma}\label[lemma]{lem:uniform-tv}
Let $B, B', \varepsilon>0$, and consider uniform distributions defined on $[B,(1+\varepsilon)B]$ and $[B',(1+\varepsilon)B']$, denoted by
$P =\mathrm{Unif}\prn*{[B,(1+\varepsilon)B]}$ and
$Q = \mathrm{Unif}\prn*{[B',(1+\varepsilon)B']}$, respectively.
Then, we have
\[
  \mathrm{TV}(P,Q)
  \le
  \frac{1+\varepsilon}{\varepsilon}
  \abs*{1-\frac{B'}{B}}.
\]
\end{lemma}
\begin{proof}
See the proof of \citet[Lemma~2.3]{Kumabe2022-hu}.
\end{proof}

\section{Poissonized wrapper}\label[appendix]{sec:poissonized-sampling}
We present several elementary tools for Poissonized sampling routines, which will be used to compare sampling procedures at different accuracy parameters.
Below, by a sampling transcript, we mean the complete random record generated by the sampling stage before any deterministic post-processing, including the realized sample count, selected indices or items, and random reweighting factors.
In applications such as Lewis-weight sampling for $\ell_1$ regression, ridge-leverage sampling for low-rank approximation, and two-stage coreset sampling for clustering, original offline guarantees are typically formulated for a fixed number of sampled rows or points.
If this number is chosen deterministically as a function of $\varepsilon$, then a small change from $\varepsilon$ to $\eta$  may change the underlying transcript space from sequences of length $m(\varepsilon)$ to sequences of length $m(\eta)$.
The two transcript laws are then supported on different sample-count coordinates, and their total variation distance can be one even when the distribution of each individual draw varies smoothly.
To remove this artificial discontinuity, we use a Poissonized wrapper: first draw a Poisson random sample count and then take that many independent samples.
The Poisson count admits a bound on the total variation distance shown below, and thus changes in the mean sample count can be controlled directly.
In what follows, $N\sim\mathrm{Pois}(\mu)$ means that~$N$ is a non-negative integer-valued random variable satisfying $\Pr[N=m]=\mathrm{e}^{-\mu}\mu^m/m!$ for all~$m\in\Z_{\ge0}$.
These tools are instantiated for Lewis-weight sampling in \cref{sec:lewis-avg-sen}, for two-stage coreset sampling in \cref{subsec:online-kz}, and for ridge-leverage sampling in \cref{sec:online-lrma}.
Below are useful properties of the Poisson distribution and the Poissonized sampling procedure.
\begin{lemma}[Poisson total variation]\label[lemma]{lem:poisson-tv}
For any $\mu,\nu\ge0$, we have
\[
\mathrm{TV}\prn*{\mathrm{Pois}(\mu),\mathrm{Pois}(\nu)}
\le
1-\mathrm{e}^{-|\mu-\nu|}
\le
|\mu-\nu|.
\]
\end{lemma}
\begin{proof}
By symmetry, it is enough to consider the case of $\nu\ge\mu$.  Let $N_\mu\sim\mathrm{Pois}(\mu)$ and $N_{\nu-\mu}\sim\mathrm{Pois}(\nu-\mu)$ be independent, and define
\[
  N_\nu=N_\mu+N_{\nu-\mu}.
\]
The sum of independent Poisson random variables is Poisson with mean equal to the sum of the means, hence $N_\nu\sim\mathrm{Pois}(\nu)$. Thus, $(N_\mu,N_\nu)$ is a coupling of $\mathrm{Pois}(\mu)$ and $\mathrm{Pois}(\nu)$.  Since $N_\nu=N_\mu+N_{\nu-\mu}$ and $N_{\nu-\mu}\ge0$, the two coupled variables are different if and only if $N_{\nu-\mu}>0$.  Therefore, by the coupling lemma (\cref{lem:tv-coupling}), we have
\[
  \mathrm{TV}\prn*{\mathrm{Pois}(\mu),\mathrm{Pois}(\nu)}
  \le
  \Pr[N_\mu\ne N_\nu]
  =
  \Pr[N_{\nu-\mu}>0].
\]
Using the definition of the Poisson distribution, we have
\[
  \Pr[N_{\nu-\mu}>0]
  =
  1-\Pr[N_{\nu-\mu}=0]
  =
  1-\mathrm{e}^{-(\nu-\mu)}.
\]
This proves the first inequality when $\nu\ge\mu$, and the general case follows by replacing $\nu-\mu$ with $|\mu-\nu|$.  The second inequality follows from $1-\mathrm{e}^{-a}\le a$ for all $a\ge0$, applied with $a=|\mu-\nu|$.
\end{proof}

\begin{lemma}[Poissonized transcript coupling]\label[lemma]{lem:poissonized-transcript-tv}
For each $m\in\Z_{\ge0}$, let $\mathcal T_m$ be a count-$m$ transcript space, and let $P_m$ and $Q_m$ be probability laws on $\mathcal T_m$.  Let $\mathsf{Mix}(\mu,\{P_m\}_{m\ge0})$ denote the law on the disjoint union of the count-indexed transcript spaces obtained by first drawing $N\sim\mathrm{Pois}(\mu)$ and then drawing a transcript from $P_N$.  Then, for any $\mu,\nu\ge0$, we have
\[
\mathrm{TV}\prn*{
\mathsf{Mix}(\mu,\{P_m\}_{m\ge0}),
\mathsf{Mix}(\nu,\{Q_m\}_{m\ge0})
}
\le
|\mu-\nu|
+
\E_{N\sim\mathrm{Pois}(\mu)}
\brc*{\mathrm{TV}(P_N,Q_N)}.
\]
\end{lemma}
\begin{proof}
Apply the same conditional-decomposition argument as in \cref{lem:tv-decomp} to the count coordinate and use \cref{lem:poisson-tv} for the count marginals.  The remaining conditional term is bounded by
\[
\sum_{m\ge0}
\min\{\Pr[\mathrm{Pois}(\mu)=m],\Pr[\mathrm{Pois}(\nu)=m]\}
\,\mathrm{TV}(P_m,Q_m)
\le
\E_{N\sim\mathrm{Pois}(\mu)}\brc*{\mathrm{TV}(P_N,Q_N)}.
\]
Thus, we have the desired bound.
\end{proof}

\begin{lemma}[From fixed size to Poissonized size]\label[lemma]{lem:fixed-to-poissonized}
Suppose a fixed-size sampling argument applies uniformly to every sample-size parameter $s\ge m_0$: when the routine is run with exactly $s$ samples, it fails with probability at most $\delta$.
Consider the Poissonized version: after drawing $N$, use the same fixed-size routine with its sample-size parameter set to the realized value $s=N$ whenever $N\ge m_0$ (for example, if a formula in the fixed-size routine contains the sample size, substitute the realized count); for $N<m_0$, output arbitrarily.  If $N\sim\mathrm{Pois}(2m_0)$, then the Poissonized routine fails with probability at most $\delta+\mathrm{e}^{-\Omega(m_0)}$.
\end{lemma}
\begin{proof}
Condition on the event $N=m$ for each $m\ge m_0$ and apply the fixed-size guarantee with sample-size parameter $s=m$.  The complement has probability $\Pr[N<m_0]\le\mathrm{e}^{-\Omega(m_0)}$ by a standard lower-tail bound for Poisson random variables.
\end{proof}

We will also use the following bound on the total variation distance between two uniform distributions with different parameters.
\begin{lemma}[Uniform perturbations with different parameters]\label[lemma]{lem:uniform-tv-changing-width}
Let $B,B'>0$ and $0<\varepsilon\le\eta\le1$.  Let $P=\mathrm{Unif}([B,(1+\varepsilon)B])$ and $Q=\mathrm{Unif}([B',(1+\eta)B'])$.  Then, it holds that
\[
\mathrm{TV}(P,Q)
\le
O\prn*{
\frac{1}{\varepsilon}\abs*{1-\frac{B'}{B}}
+
\frac{\eta-\varepsilon}{\varepsilon}
}.
\]
\end{lemma}
\begin{proof}
Consider using the triangle inequality to compare $P$ and $Q$ through the intermediate distribution $Q'=\mathrm{Unif}([B',(1+\varepsilon)B'])$.
The first comparison of $P$ and $Q'$ is controlled by \cref{lem:uniform-tv}.  For the second comparison of $Q'$ and $Q$, the left endpoint is fixed and only the interval length changes; the overlap ratio is $\varepsilon/\eta$, and thus the total variation distance is $1-\varepsilon/\eta\le(\eta-\varepsilon)/\varepsilon$.
\end{proof}

\section{Regularity condition for establishing Assumption~\ref{assump:adaptive-pc-stability}}\label[appendix]{sec:calculus-for-corollary}
We present an elementary consequence of the accuracy parameter choice given in \eqref{eq:rule} that will play a key role in verifying \cref{assump:adaptive-pc-stability} in each application.
For a sensitivity function $\varphi$ differentiable on~$(0,1]$, put
\[
  \psi(\varepsilon)\coloneqq -\varphi'(\varepsilon)
  \qquad
  (0<\varepsilon\le1),
\]
where the value at $\varepsilon=1$ is understood as the left-derivative value.
We use the following three primitive conditions: there exist constants $A,\rho,B>0$ such that
\begin{align}
  \varphi(\varepsilon) &\le A\varepsilon\psi(\varepsilon)
  &&\text{for all } 0<\varepsilon\le1, \label{eq:marginal-power-value}\\
  \psi(\eta) &\le \psi(\varepsilon)\prn*{\frac{\varepsilon}{\eta}}^\rho
  &&\text{for all } 0<\varepsilon\le\eta\le1, \label{eq:marginal-power-scale}\\
  \psi(1) &\ge B. \label{eq:marginal-power-endpoint}
\end{align}
The first condition says that the value of $\varphi$ is controlled by its marginal cost $\psi(\varepsilon)= -\varphi'(\varepsilon)$ multiplied by $\varepsilon$, while the second condition says that this marginal cost is power-regular under multiplicative changes of the accuracy parameter.
The third one is a technical endpoint condition.
The following lemma offers a convenient calculus for verifying \cref{assump:adaptive-pc-stability}.

\begin{lemma}\label[lemma]{lem:adaptive-rule-calculus}
Let $\varphi\colon \R_{\ge0}\to\R_{\ge0}\cup\set*{+\infty}$ satisfy $\varphi(\varepsilon)=+\infty$ outside $(0,1]$, be finite and continuously differentiable on $(0,1]$, and satisfy $\lim_{\varepsilon\downarrow0}\varphi(\varepsilon)=+\infty$.
Assume that $\psi(\varepsilon)\coloneqq -\varphi'(\varepsilon)>0$ holds on $(0,1]$ and that \eqref{eq:marginal-power-value}--\eqref{eq:marginal-power-endpoint} hold for some constants $A,\rho,B>0$.
Let $\varepsilon_{\min}$ denote the parameter selection rule in \eqref{eq:rule}.
If $u\ge v>0$ and $u-v\le1$, then we have $\varepsilon_{\min}(u)\le\varepsilon_{\min}(v)$ and
\[
  \varphi(\varepsilon_{\min}(u))
  \frac{\varepsilon_{\min}(v)-\varepsilon_{\min}(u)}{\varepsilon_{\min}(u)}
  =
  O_{A,\rho,B}\prn*{\varepsilon_{\min}(u)(u-v)},
\]
equivalently,
$
  \varphi(\varepsilon_{\min}(u))
  \prn*{\varepsilon_{\min}(v)-\varepsilon_{\min}(u)}
  =
  O_{A,\rho,B}\prn*{\varepsilon_{\min}(u)^2(u-v)}
$.
In particular, if the function $\varphi$ takes the form considered in \cref{cor:simple-form}, i.e.,
\[
  \varphi(\varepsilon)=C_1\varepsilon^{-q}\log(\mathrm e+C_2/\varepsilon)
  \qquad (0<\varepsilon\le1),
\]
with $\varphi(\varepsilon)=+\infty$ outside $(0,1]$, where $q>0$, $C_1\ge1$, and $C_2\ge0$, then \eqref{eq:marginal-power-value}--\eqref{eq:marginal-power-endpoint} hold with $A=1/q$, $\rho=q+1$, and $B=q$.
Consequently, for this specific form, it holds that
\[
  \varphi(\varepsilon_{\min}(u))
  \frac{\varepsilon_{\min}(v)-\varepsilon_{\min}(u)}{\varepsilon_{\min}(u)}
  =
  O_q\prn*{\varepsilon_{\min}(u)(u-v)},
\]
equivalently,
$
  \varphi(\varepsilon_{\min}(u))
  \prn*{\varepsilon_{\min}(v)-\varepsilon_{\min}(u)}
  =
  O_q\prn*{\varepsilon_{\min}(u)^2(u-v)}
$.
\end{lemma}
\begin{proof}
The monotonicity $\varepsilon_{\min}(u)\le\varepsilon_{\min}(v)$ follows from the properties of $\varepsilon_{\min}$ in \cref{lem:concave-conjugate-prop}.
It remains to prove the quantitative bound.
We first prove the general bound under \eqref{eq:marginal-power-value}--\eqref{eq:marginal-power-endpoint}, and then verify these conditions for the specific form.
Let
\[
  a\coloneqq \varepsilon_{\min}(u),
  \qquad
  b\coloneqq \varepsilon_{\min}(v),
  \qquad
  \Delta\coloneqq u-v,
  \qquad
  \text{and}
  \qquad
  r\coloneqq \frac{b-a}{a}.
\]
By the monotonicity of $\varepsilon_{\min}(\cdot)$, we have $0<a\le b\le1$.
If $a=1$, then $b=1$ and the claim is trivial.
Hence, assume $a<1$.
For a scalar $w>0$, write
$g_w(\varepsilon)\coloneqq w\varepsilon+\varphi(\varepsilon)$
for the objective minimized in the definition of $\varepsilon_{\min}(w)$.
Since $a=\varepsilon_{\min}(u)$ is an interior minimizer of $g_u$ and
$g_u'(\varepsilon)=u-\psi(\varepsilon)$, we have $u=\psi(a)$.
If $b<1$, the same argument gives $v=\psi(b)$.
If $b=1$, the one-sided optimality condition for minimizing $g_v$ over
$(0,1]$ gives $g_v'(1)\le0$, equivalently $v\le\psi(1)=\psi(b)$.
Thus, in both cases, it holds that
\begin{equation}\label{eq:mpower-delta-psi}
  \Delta
  \ge
  \psi(a)-\psi(b)
  =
  u - \psi(b).
\end{equation}
By \eqref{eq:marginal-power-scale}, we have
\[
  \psi(b)
  \le
  \psi(a)\prn*{\frac{a}{b}}^\rho
  =
  u(1+r)^{-\rho}.
\]
Combining this with \eqref{eq:mpower-delta-psi} yields
\begin{equation}\label{eq:mpower-delta-lower}
  \Delta
  \ge
  u\prn*{1-(1+r)^{-\rho}}.
\end{equation}

We split the argument into two cases.
First suppose $r\le1$.
Since $s\mapsto1-(1+s)^{-\rho}$ is concave, non-negative, and vanishes at $s=0$, we have
\[
  1-(1+r)^{-\rho}
  \ge
  \prn*{1-2^{-\rho}}r
  \qquad (0\le r\le1).
\]
Therefore, \eqref{eq:mpower-delta-lower} gives $\Delta\ge(1-2^{-\rho})ur$.
Using \eqref{eq:marginal-power-value}, we obtain
\[
  \varphi(a)r
  \le
  A a\psi(a)r
  =
  A aur
  \le
  \frac{A}{1-2^{-\rho}}a\Delta.
\]
This proves the desired bound when $r\le1$.

Next suppose that $r>1$ holds.
Then \eqref{eq:mpower-delta-lower} gives
\[
  \Delta
  \ge
  \prn*{1-2^{-\rho}}u.
\]
Since $\Delta\le1$, it follows that $u\le(1-2^{-\rho})^{-1}$.
On the other hand, applying \eqref{eq:marginal-power-endpoint} and \eqref{eq:marginal-power-scale} with $\varepsilon=a$ and $\eta=1$ gives
\[
  B
  \le
  \psi(1)
  \le
  \psi(a)a^\rho
  =
  ua^\rho.
\]
Hence, whenever this case occurs, $a$ is bounded below by a positive constant depending only on $\rho$ and $B$; specifically, we have
\[
  a\ge  a_{\rho,B}
  \coloneqq
  \min\set*{1,\prn*{B(1-2^{-\rho})}^{1/\rho}}.
\]
Since $b\le1$, \eqref{eq:marginal-power-value} and the preceding lower bound on $\Delta$ give
\[
  \varphi(a)r
  =
  \varphi(a)\frac{b-a}{a}
  \le
  \frac{\varphi(a)}{a}
  \le
  A\psi(a)
  =
  Au
  \le
  \frac{A}{1-2^{-\rho}}\Delta
  \le
  \frac{A}{(1-2^{-\rho}) a_{\rho,B}}a\Delta.
\]
This proves the desired bound when $r>1$.

It remains to verify the three primitive conditions for $\varphi(\varepsilon)=C_1\varepsilon^{-q}\log(\mathrm e+C_2/\varepsilon)$.
Set
\[
  L(\varepsilon)=\log(\mathrm e+C_2/\varepsilon),
  \qquad
  R(\varepsilon)=\frac{C_2}{\mathrm e\varepsilon+C_2},
  \qquad
  \text{and}
  \qquad
  M(\varepsilon)=qL(\varepsilon)+R(\varepsilon).
\]
Then, we have
\[
  \varphi(\varepsilon)
  =
  C_1\varepsilon^{-q}L(\varepsilon)
  \qquad\text{and}\qquad
  \psi(\varepsilon)
  =
  -\varphi'(\varepsilon)
  =
  C_1\varepsilon^{-q-1}M(\varepsilon).
\]
Since $R(\varepsilon)\ge0$, we have $M(\varepsilon)\ge qL(\varepsilon)$.
Thus, from $L(\varepsilon)\ge1$, it follows that
\[
  \frac{\varphi(\varepsilon)}{\varepsilon\psi(\varepsilon)}
  =
  \frac{L(\varepsilon)}{M(\varepsilon)}
  \le
  \frac1q.
\]
Hence \eqref{eq:marginal-power-value} holds with $A=1/q$.
Moreover, both $L$ and $R$ are non-increasing on $(0,1]$, and thus $M$ is non-increasing.
Therefore, if $0<\varepsilon\le\eta\le1$, then we have
\[
  \psi(\eta)
  =
  C_1\eta^{-q-1}M(\eta)
  \le
  C_1\eta^{-q-1}M(\varepsilon)
  =
  \psi(\varepsilon)\prn*{\frac{\varepsilon}{\eta}}^{q+1}.
\]
Thus, \eqref{eq:marginal-power-scale} holds with $\rho=q+1$.
Finally, it holds that
\[
  \psi(1)
  =
  C_1\prn*{qL(1)+R(1)}
  \ge
  q,
\]
because $C_1\ge1$, $L(1)\ge1$, and $R(1)\ge0$.
Thus \eqref{eq:marginal-power-endpoint} holds with $B=q$.
Plugging $A=1/q$, $\rho=q+1$, and $B=q$ into the general bound completes the proof.
\end{proof}

A convenient sufficient template for checking \cref{lem:adaptive-rule-calculus} is the following.
Suppose
\[
  \varphi(\varepsilon)=C\varepsilon^{-q}L(\varepsilon),
  \qquad
  C\ge1,
  \qquad
  \text{and}
  \qquad
  q>0,
\]
where $L\ge1$ is continuously differentiable and non-increasing on $(0,1]$.
If
\[
  M(\varepsilon)\coloneqq qL(\varepsilon)-\varepsilon L'(\varepsilon)
\]
is non-increasing, then $\psi(\varepsilon)=C\varepsilon^{-q-1}M(\varepsilon)$ satisfies \eqref{eq:marginal-power-value} with $A=1/q$ and \eqref{eq:marginal-power-scale} with $\rho=q+1$.
Moreover, since $M(1)\ge qL(1)\ge q$, condition \eqref{eq:marginal-power-endpoint} holds with $B=q$.
Thus the proof above applies without using any special property of the logarithm beyond these monotonicity checks.

\section{Offline algorithm for submodular function minimization}\label[appendix]{sec:submod-avg-sens-proof}
We detail the offline algorithm for submodular function minimization and prove the average sensitivity bound stated in \cref{prop:submod-avg-sens}.
Below, we work under value-oracle access to each splitting submodular function as in \citet{Kenneth2024-pf}, while we assume the exact-real model to handle perturbation by uniform noise introduced later.
Throughout the anchored-lift discussion,~$V$ denotes the original ground set and $n=|V|$.
In the generic sparsification analysis, we write the vertex set as $U$.
In the anchored setting, $U=V^+$ and hence $|U|=n+1$, which does not change the asymptotic bounds after translating back to the original ground-set size.

\subsection{Anchored lift for non-normalized submodular losses}\label[appendix]{subsec:anchored-lift}
The cut-sparsifier theorem of \citet{Kenneth2024-pf} is stated for splitting functions normalized at the empty set.  The online submodular losses in \cref{sec:submodular} need not satisfy $\ell_t(\emptyset)=0$.  We therefore use the following anchored representation.

Let $V$ be the original ground set and let $r\notin V$ be a dummy vertex.
For this anchored lift, set $U\coloneqq V^+=V\cup\set{r}$.
For a non-negative submodular function $\ell\colon2^V\to\R_{\ge0}$, define $g_\ell\colon2^{U}\to\R_{\ge0}$ by\looseness=-1
\[
  g_\ell(A)
  =
  \begin{cases}
    0, & A=\emptyset,\\
    \ell(A\cap V), & A\neq\emptyset.
  \end{cases}
\]
A value query to $g_\ell$ is implemented by either returning zero, when the queried set is empty, or by making one value query to the original oracle for $\ell$ on $A\cap V$.  Thus the lift is compatible with the same value-oracle model as the original losses.

\begin{lemma}\label[lemma]{lem:anchored-lift}
  The function $g_\ell$ is non-negative and submodular on $U$, and it satisfies $g_\ell(\emptyset)=0$.  Moreover, for every $S\subseteq V$, we have $g_\ell(S\cup\set{r})=\ell(S)$.
\end{lemma}

\begin{proof}
  Non-negativity and the identity on anchored sets are immediate from the definition.  It remains to check submodularity.  Take arbitrary $A,B\subseteq U$.  If $A=\emptyset$ or $B=\emptyset$, the submodularity inequality for $g_\ell$ holds with equality.  Assume next that $A,B\neq\emptyset$ and $A\cap B\neq\emptyset$.  Then, we have
  \[
    g_\ell(A)+g_\ell(B)
    =
    \ell(A\cap V)+\ell(B\cap V)
    \ge
    \ell((A\cup B)\cap V)+\ell((A\cap B)\cap V)
    =
    g_\ell(A\cup B)+g_\ell(A\cap B),
  \]
  where the inequality is the submodularity of $\ell$.  Finally, suppose that $A,B\neq\emptyset$ and $A\cap B=\emptyset$.  Since $g_\ell(A\cap B)=0$ and $\ell(\emptyset)\ge0$, submodularity of $\ell$ gives
  \[
    g_\ell(A)+g_\ell(B)
    =
    \ell(A\cap V)+\ell(B\cap V)
    \ge
    \ell((A\cup B)\cap V)+\ell(\emptyset)
    \ge
    g_\ell(A\cup B)+g_\ell(A\cap B).
  \]
  Thus, $g_\ell$ is submodular.
\end{proof}

For a multiset $X$ of submodular loss functions, let $H_X^+=(U,E_X,\set{g_e}_{e\in E_X})$ be the lifted hypergraph obtained by creating one hyperedge $e=U$ for each $\ell\in X$ and setting $g_e=g_\ell$.  The original objective is exactly represented on anchored cuts:
\[
  \sum_{\ell\in X}\ell(S)
  =
  \mathrm{cut}_{H_X^+}(S\cup\set{r})
  \qquad
  \text{for every $S\subseteq V$.}
\]
After constructing a sparsifier $\tilde H_X^+$, we do not minimize over all cuts of $U=V^+$.  Instead, we define an ordinary set function on the original ground set by
\[
  F_{\tilde H_X^+}(S)
  \coloneqq
  \mathrm{cut}_{\tilde H_X^+}(S\cup\set{r})
  =
  \sum_{e\in \tilde E_X} w_e\,g_e((S\cup\set{r})\cap e)
  \qquad \text{for every\; $S\subseteq V$},
\]
where $\tilde E_X$ is the sampled hyperedge multiset and $w_e\ge0$ denotes the sampled reweighting factor.  This function is submodular on $V$: for each sampled hyperedge $e$, the map $S\mapsto g_e((S\cup\set{r})\cap e)$ preserves unions and intersections in the sense that the corresponding arguments to $g_e$ for $S\cup T$ and $S\cap T$ are the union and intersection of the arguments for $S$ and $T$, and hence submodularity of $g_e$ applies.  A non-negative weighted sum of these functions is again submodular.

Moreover, $F_{\tilde H_X^+}$ has an explicit value oracle from the original loss oracles.  In the lifted instance every sampled hyperedge corresponds to some original loss $\ell$, and $e=U=V^+$, hence $g_e((S\cup\set{r})\cap e)=\ell(S)$.  Therefore, a query to $F_{\tilde H_X^+}(S)$ is answered by querying each sampled loss oracle at the same set $S$ and summing the returned values with the sampled weights $w_e$.  Thus, the post-processing step is exactly standard value-oracle submodular function minimization on the ground set $V$, with a fixed deterministic tie-breaking rule.
Note that, although the ambient lifted hypergraph has the empty cut $A=\emptyset$, the post-processing step minimizes only over anchored cuts $A=S\cup\{r\}$ over $S \subseteq V$.
Thus, $A=\emptyset$ is not feasible, and the original empty action corresponds to $\{r\}$, not to $\emptyset$.

From the next subsection onward, $H=(U,E,\set{g_e}_{e\in E})$ denotes a generic submodular hypergraph.
When these statements are applied to the anchored construction above, we take $U=V^+$ and restrict the final minimization to anchored cuts $S\cup\set{r}$ with $S\subseteq V$.

\subsection{Sparsification via importance sampling}
We describe the importance-sampling-based algorithm of \citet{Kenneth2024-pf} in detail.
Recall that a submodular hypergraph is given by $H=(U,E,\set{g_e}_{e\in E})$ with vertex set $U$, hyperedges~$E$, and non-negative submodular splitting functions $g_e\colon 2^e\to\R_{\ge0}$ satisfying $g_e(\emptyset)=0$ for each~$e\in E$.
Define
\[
  g^e_{a\to b} \coloneqq \min_{{A\subseteq U:\;a\in A,\;b\notin A}} g_e(A\cap e)
  \qquad
  \text{for distinct $a,b\in U$.}
\]
In what follows, all sums over ordered pairs are taken over distinct pairs.
Following \citet{Kenneth2024-pf}, define
\[
  \rho_e(H) \coloneqq \sum_{\substack{(a,b)\in U\times U\\ a\ne b}} \frac{g^e_{a\to b}}{\sum_{f\in E} g^f_{a\to b}},
  \quad
  \rho'_e(H) \coloneqq \frac{g_e(e)}{\sum_{f\in E} g_f(f)}
  \quad
  \text{and}\quad
  s_e(H) \coloneqq \rho_e(H)+\rho'_e(H),
\]
with the convention that zero denominators yield zero ratios.
Note that we have
\begin{equation}
  \label{eq:rho-sum-bounds}
  \sum_{e\in E} \rho_e(H) \le |U|^2
  \qquad
  \text{and}
  \qquad
  \sum_{e\in E} \rho'_e(H) \le 1.
\end{equation}

Let $M > 0$, which is specified later.
The sparsification algorithm of \citet{Kenneth2024-pf} samples each $e\in E$ independently with probability
\[
  p_e(H)=\min\set*{1,\,M s_e(H)},
\]
sets $g'_e \equiv g_e / p_e(H)$ for each selected hyperedge $e \in E$, and returns the resulting sparsifier.
\Citet{Kenneth2024-pf} show that, for $\delta \in (0,1)$, setting $M = \Theta(\varepsilon^{-2}(|U| + \log(1/\delta)))$ yields a $\prn*{1\pm\varepsilon}$-sparsifier $H' = (U, E', \set{g'_e}_{e\in E'})$ with probability at least $1-\delta$, and the sparsifier size is $\E[\abs{E'}]=\sum_e p_e(H)\le M(|U|^2+1) = O(\varepsilon^{-2}|U|^2(|U| + \log(1/\delta)))$.
For our purposes, we set $\delta = 1/T$, which adds $O(t/T)$ to the expected loss evaluated on $X_t$ in the failure event and hence satisfies \cref{assump:approx-avesen}.
Once the sparsifier $H'$ is constructed with internal accuracy
$\varepsilon'=\Theta(\varepsilon)$, as specified in
\cref{rem:accuracy-translation} below, we minimize the relevant sparse cut objective
to obtain a $(1+\varepsilon)$-approximate solution;
in the anchored setting this applies to $S\mapsto\mathrm{cut}_{H'}(S\cup\set{r})$ over the original actions $S\subseteq V$.\looseness=-1

\begin{remark}[Accuracy translation]\label{rem:accuracy-translation}
  Suppose that an approximating objective function $\tilde F$ is a two-sided
  $(1\pm\varepsilon')$ approximation to an original objective function $F$.
  Then, any minimizer of $\tilde F$ is a
  $(1+\varepsilon')/(1-\varepsilon')$-approximate minimizer of $F$.
  Therefore, for target accuracy $\varepsilon\in(0,1]$, it is sufficient to
  run the sparsification or sampling routine with internal accuracy
  $\varepsilon'=\varepsilon/3$, since we have
  \[
    \frac{1+\varepsilon'}{1-\varepsilon'}
    =
    \frac{1+\varepsilon/3}{1-\varepsilon/3}
    \le 1+\varepsilon
    \qquad\text{for all\;} 0<\varepsilon\le1.
  \]
  In the sequel, we write $\varepsilon$ for the internal accuracy parameter
  for notational simplicity; this changes the stated sensitivity bounds only
  by absolute constant factors.
\end{remark}

\textbf{Computational complexity.}
The sparsifier construction of \citet{Kenneth2024-pf} is constructive and runs in polynomial time under their standard value-oracle and finite-precision assumptions.
We use this construction as a subroutine; the additional continuous perturbations introduced below are interpreted under the exact-real oracle model stated in the preliminaries.
The resulting sparsifier has $\tilde O(\varepsilon^{-2}|U|^3)$ hyperedges, and the remaining optimization task is a standard submodular function minimization problem in the value-oracle model \citep{Schrijver2000-tj,Iwata2001-ao,Jiang2024-sb}.  In the anchored reduction, we have $|U|=n+1$, and \cref{subsec:anchored-lift} shows that the optimized function is an ordinary submodular function on the original ground set and that its value oracle is implemented by summing value-oracle queries to the sampled losses.
Hence, for every fixed $\varepsilon$, the offline $(1+\varepsilon)$-approximation algorithm runs in polynomial time under these standard assumptions.

For later use, we present basic lemmas on this importance-sampling-based sparsification.
\begin{lemma}\label[lemma]{lem:clip-lip}
  For $x,y\ge 0$ and $M>0$, $\abs{\min\set*{1,Mx}-\min\set*{1,My}}\le M\abs{x-y}$ holds.
\end{lemma}
\begin{proof}
  The map $x\mapsto \min\set*{1,Mx}$ is piecewise linear with slope in $[0,M]$, hence $M$-Lipschitz.
\end{proof}

\begin{lemma}\label[lemma]{lem:l1change}
  For any $e\in E$, we have
  $
    \sum_{f\in E,\;f \neq e}\abs{ s_f(H)-s_f(H-e) } \le s_{e}(H)
  $.
\end{lemma}

\begin{proof}
  First, we have
  \[
    \sum_{f\in E,\;f \neq e} \abs{ s_f(H)-s_f(H-e) } \le \sum_{f\in E,\;f \neq e} \abs{ \rho_f(H)-\rho_f(H-e) } + \sum_{f\in E,\;f \neq e} \abs{ \rho'_f(H)-\rho'_f(H-e) }.
  \]
  We bound the two parts separately.
  Fix distinct $a,b \in U$ and $e \in E$ and abbreviate $D=\sum_{f \in E} g^f_{a\to b}$ and $\alpha=g^{e}_{a\to b}$.
  If $D = 0$, $\rho_f \equiv 0$ and the $\rho$ part is zero; if $D = \alpha$, $\rho_f(H-e) = \rho_f(H)$ for all $f\neq e$ and thus the $\rho$ part is again zero.  Hence, we may assume that $D > 0$ and $\alpha < D$.
  It holds that
  \[
    \sum_{f\in E,\;f \neq e} \abs*{\frac{g^f_{a\to b}}{D-\alpha} - \frac{g^f_{a\to b}}{D}}
    =
    \sum_{f\in E,\;f \neq e} \prn*{\frac{g^f_{a\to b}}{D-\alpha} - \frac{g^f_{a\to b}}{D}}
    =
    1 - \frac{D - \alpha}{D}
    =
    \frac{\alpha}{D}.
  \]
  Summing over distinct ordered pairs $(a,b)$ and using $|\sum_i x_i| \le \sum_i |x_i|$ yield the claim for the $\rho$ part.
  The same argument with $D=\sum_{f\in E} g_f(f)$ and $\alpha=g_{e}(e)$ bounds the $\rho'$ part.
  Consequently, the desired inequality follows by adding the two parts.
\end{proof}

\subsection{Weight perturbation}\label[appendix]{subsec:weight-perturbation}
We introduce a small random perturbation to the weights of sampled hyperedges, thereby controlling the average sensitivity of the full output, including both sampled hyperedges and assigned weights.
This idea is already studied in \citet{Yoshida2022-pb} and \citet{Dong2023-yh}, but we restate it here for completeness.

Recall that we sample each hyperedge $e\in E$ independently with probability
\[
  p_e(H) = \min\set*{1,\,M s_e(H)},
\]
and, if selected, rescale $g_e$ by the weight $1/p_e(H)$.
By setting $M = \Theta(\varepsilon^{-2}(|U| + \log T))$, with probability at least $1 - 1/T$, this construction yields a $(1\pm \varepsilon/2)$-cut sparsifier of $H$, where the division by two is for later convenience.

\textbf{Perturbation rule.\;}
Fix $\varepsilon\in(0,1]$.
Whenever a hyperedge $e$ is sampled, instead of using the deterministic weight $1/p_e(H)$,
we draw an independent random variable
\[
  \widetilde p_e \sim \mathrm{Unif}\prn*{[p_e(H),(1+\varepsilon/2)p_e(H)]}
\]
and assign the weight $1/\widetilde p_e$ to $g_e$.
If $e$ is not sampled, its weight is set to $0$.
We denote by $\tilde H$ the resulting randomly reweighted subgraph, and by $H'$ the original
(unperturbed) sparsifier that uses weights $1/p_e(H)$.
Note that $\widetilde p_e$ is used solely as a weight parameter and is not required to lie in $[0,1]$.

\textbf{Preservation of approximation guarantee.\;}
By construction, whenever $e$ is selected, we have
\[
  \frac{1}{(1+\varepsilon/2)p_e(H)} \le \frac{1}{\widetilde p_e} \le \frac{1}{p_e(H)}.
\]
Therefore, for every cut $A\subseteq U$, we have
\[
  \frac{1}{1+\varepsilon/2}\,\mathrm{cut}_{H'}(A)
  \le
  \mathrm{cut}_{\tilde H}(A)
  \le
  \mathrm{cut}_{H'}(A).
\]
Combining this inequality with the fact that $H'$ is a $(1\pm \varepsilon/2)$-cut sparsifier
of $H$, we obtain
\[
  \frac{1-\varepsilon/2}{1+\varepsilon/2}\,\mathrm{cut}_{H}(A)
  \le
  \mathrm{cut}_{\tilde H}(A)
  \le
  (1+\varepsilon/2)\,\mathrm{cut}_{H}(A)
  \qquad
  \text{for all $A\subseteq U$.}
\]
Since $\frac{1-\varepsilon/2}{1+\varepsilon/2} \ge 1-\varepsilon$ for $\varepsilon\in(0,1]$,
the perturbed sparsifier $\tilde H$ is a $(1\pm \varepsilon)$-cut sparsifier of $H$.

\subsection{Bounding average sensitivity}
We show that, given a submodular hypergraph $H = (U, E, \set{g_e}_{e \in E})$, the law of the output of the $(1\pm\varepsilon)$-sparsification algorithm, equipped with the above weight perturbation, satisfies
  \[
    \frac{1}{|E|} \sum_{e \in E} \mathrm{TV}\prn*{\mathcal{L}(H), \mathcal{L}(H - e)}
    =
    O\prn*{ \frac{\varepsilon^{-3}(|U|^{3} + |U|^2\log T)}{|E|}},
  \]
thereby proving \cref{prop:submod-avg-sens}.

\textbf{Decomposition into independent coordinates.\;}
Let $Z=(Z_f)_{f\in E}\in\{0,1\}^E$ denote the random subset indicator vector, where $Z_f=1$ iff $f \in E$ is sampled, and let
$W=(W_f)_{f\in E}\in\R_{\ge0}^E$ denote the resulting random weight vector (with $W_f=0$ when $Z_f=0$).
Note that specifying these random variables determines the output sparsifier.
Fix any $e\in E$, and let $P$ and $Q$ be the joint distributions of $(Z,W)$ produced from $H$ and $H-e$, respectively.
We view both $P$ and $Q$ as probability measures on the common space $\{0,1\}^E\times\mathbb{R}_{\ge 0}^E$ by identifying the $e$-th coordinate of the sampling probability vector under $H-e$ with $p_e(H-e)=0$.
By the procedure of the above sparsification algorithm, for each $f\in E$, the random pair $(Z_f,W_f)$ is generated independently of all the other coordinates, both under $P$ and under $Q$.
Therefore, $P$ and $Q$ admit the product decompositions
\[
P=\bigotimes_{f\in E} P^{(f)}
\qquad
\text{and}
\qquad
Q=\bigotimes_{f\in E} Q^{(f)},
\]
where $P^{(f)}$ and $Q^{(f)}$ are probability measures on $\{0,1\}\times\R_{\ge0}$ corresponding to hyperedge $f$.
We let $Q^{(e)}=\delta_{(0,0)}$, where $\delta_{(0,0)}$ is the Dirac measure at $(0,0)$.
By applying \cref{lem:product-tv-general}, we have
\begin{equation}\label{eq:tv-product-first}
  \mathrm{TV}(P,Q)
  \le
  \sum_{f\in E}
  \mathrm{TV}\!\left(P^{(f)},Q^{(f)}\right).
\end{equation}
Below, we analyze the total variation distance for each coordinate $f$.

\textbf{Decomposition via conditioning.\;}
For each $f\in E$, the distribution $P^{(f)}$ is a two-point mixture of the form
\[
P^{(f)}
=
(1-p_f(H))\,\delta_{(0,0)}
+
p_f(H)\,P_{(Z_f,W_f)\mid Z_f=1},
\]
and similarly
\[
Q^{(f)}
=
(1-p_f(H-e))\,\delta_{(0,0)}
+
p_f(H-e)\,Q_{(Z_f,W_f)\mid Z_f=1},
\]
where, whenever the relevant sampling probability is positive, the conditional law given $Z_f=1$ is specified by $W_f=1/\widetilde p_f$ with
\begin{align}
  &\widetilde p_f \sim \mathrm{Unif}\prn*{[p_f(H),(1+\varepsilon/2)p_f(H)]}
  &&\text{under $P$, and}
  \\
  &\widetilde p_f \sim \mathrm{Unif}\prn*{[p_f(H-e),(1+\varepsilon/2)p_f(H-e)]}
  &&\text{under $Q$.}
\end{align}
If either $p_f(H)=0$ or $p_f(H-e)=0$, the coordinate whose sampling probability is zero is deterministically equal to $(0,0)$, and hence
\[
  \mathrm{TV}\!\left(P^{(f)},Q^{(f)}\right)
  =
  \abs{p_f(H)-p_f(H-e)}.
\]
It remains to consider the case $p_f(H)>0$ and $p_f(H-e)>0$.  Applying \cref{lem:tv-decomp,lem:uniform-tv} to this mixture yields
\begin{align}
  &\mathrm{TV}\!\left(P^{(f)},Q^{(f)}\right)
  \\
  \le{}&
  \abs{p_f(H)-p_f(H-e)}
  +
  \min\{p_f(H),p_f(H-e)\}\,
  \mathrm{TV}\!\left(
    P_{W_f\mid Z_f=1},
    Q_{W_f\mid Z_f=1}
  \right)
  \\
  \le{}&
  \abs{p_f(H)-p_f(H-e)}
  +
  \min\{p_f(H),p_f(H-e)\}\,
  \frac{1+\varepsilon/2}{\varepsilon/2}\,
  \abs*{1-\frac{p_f(H-e)}{p_f(H)}}.
  \\
  \le{}&
  \abs{p_f(H)-p_f(H-e)}
  +
  \frac{2+\varepsilon}{\varepsilon}\,
  \abs{p_f(H)-p_f(H-e)}.
  \\
  \le{}&
  \frac{4}{\varepsilon}\,
  \abs{p_f(H)-p_f(H-e)}.
\end{align}
Therefore, for every $f\in E$, we have
$
  \mathrm{TV}\!\left(P^{(f)},Q^{(f)}\right)
  \le
  \frac{4}{\varepsilon}
  \abs{p_f(H)-p_f(H-e)}
$.

\textbf{Bounding total variation over all hyperedges.\;}
Combining the above bound with \eqref{eq:tv-product-first} gives
\begin{equation}\label{eq:tv-sum-bound}
  \mathrm{TV}(P,Q)
  \le
  \frac{4}{\varepsilon}
  \sum_{f\in E}
  \abs{p_f(H)-p_f(H-e)}
  =
  \frac{4}{\varepsilon}
  \prn*{
    \sum_{f \in E,\, f\neq e}
    \abs{p_f(H)-p_f(H-e)}
    +
    p_e(H)
  }.
\end{equation}
By the definition of $p_f(H) = \min\set*{1,\,M s_f(H)}$ and \cref{lem:clip-lip,lem:l1change}, we have
\begin{equation}
  \sum_{f \in E,\, f\neq e}\!\!
    \abs{p_f(H)-p_f(H-e)}
    +
    p_e(H)
  \le
  M
  \prn*{
    \sum_{f \in E,\, f\neq e}\!\!
    \abs{ s_f(H)-s_f(H-e) }
    \!+\!
    s_e(H)
  }
  \!\le
  2 M s_e(H).
\end{equation}
Therefore, taking the average over $e\in E$ and using \eqref{eq:rho-sum-bounds} and $M = \Theta(\varepsilon^{-2}(|U| + \log T))$, we obtain
\begin{equation}
  \frac{1}{|E|} \sum_{e \in E} \mathrm{TV}\prn*{\mathcal{L}(H), \mathcal{L}(H\!-\!e)}
  \!\le\!
  \frac{8 M}{\varepsilon |E|} \sum_{e \in E} s_e(H)
  \!\le\!
  \frac{8 M}{\varepsilon |E|} (|U|^2 + 1)
  \!=\!
  O\prn*{ \frac{\varepsilon^{-3}(|U|^{3}\!+\!|U|^2\log T)}{|E|}}.
\end{equation}
This completes the proof of \cref{prop:submod-avg-sens}, and hence $\varphi(\varepsilon)=C_{\rm sub}\varepsilon^{-3}(n^3+n^2\log T)$.

\subsection{Checking Assumption~\ref{assump:adaptive-pc-stability}}\label[appendix]{subsec:submod-adaptive-pc-proof}
We next verify that the above sparsifier construction, equipped with the weight perturbation, satisfies \cref{assump:adaptive-pc-stability} with $C_{\rm pc}=O(1)$.
\begin{proposition}\label[proposition]{prop:submod-same-hypergraph-pc}
Let $H=(U,E,\set{g_e}_{e\in E})$ be a submodular hypergraph, and let $\mathcal L_\varepsilon(H)$ be the law of the perturbed sparsifier above with parameter $\varepsilon\in(0,1]$.  For $0<\varepsilon\le\eta\le1$, it holds that
\[
\mathrm{TV}\prn*{\mathcal L_\varepsilon(H),\mathcal L_\eta(H)}
\le
O\prn*{
(|U|^3+|U|^2\log T)\varepsilon^{-3}\frac{\eta-\varepsilon}{\varepsilon}
}.
\]
Consequently, for the anchored lifted instance in \cref{subsec:anchored-lift} built from the original $n$-element ground set, the sparsifier-based offline algorithm family satisfies \cref{assump:adaptive-pc-stability} with $C_{\rm pc}=O(1)$.
\end{proposition}
\begin{proof}
For each parameter $\alpha\in(0,1]$, represent the output of the sparsifier construction by the vector $((Z_f^\alpha,W_f^\alpha))_{f\in E}$, where $Z_f^\alpha$ indicates whether $f$ is sampled and $W_f^\alpha=0$ when $Z_f^\alpha=0$.  Conditional on $Z_f^\alpha=1$ and $p_f^\alpha(H)>0$, the algorithm draws $\widetilde p_f^\alpha$ uniformly from $[p_f^\alpha(H),(1+\alpha/2)p_f^\alpha(H)]$ and sets $W_f^\alpha=1/\widetilde p_f^\alpha$, where
\[
p_f^\alpha(H)=\min\{1,M(\alpha)s_f(H)\},
\qquad
M(\alpha)=c_{\rm sp}\alpha^{-2}(|U|+\log T),
\]
for a fixed absolute constant $c_{\rm sp} > 0$.  Since $x\mapsto\min\{1,x\}$ is $1$-Lipschitz, $\sum_f s_f(H)\le |U|^2+1$, and $|M(\varepsilon)-M(\eta)|\le O(M(\varepsilon)(\eta-\varepsilon)/\varepsilon)$ for $0<\varepsilon\le\eta\le1$, we have
\begin{equation}\label{eq:submod-param-p-diff-sum}
\sum_{f\in E}|p_f^\varepsilon(H)-p_f^\eta(H)|
=
O\prn*{(|U|^3+|U|^2\log T)\varepsilon^{-2}\frac{\eta-\varepsilon}{\varepsilon}}.
\end{equation}
Also, it holds that
\begin{equation}\label{eq:submod-param-p-sum}
\sum_{f\in E}p_f^\varepsilon(H)
\le
M(\varepsilon)\sum_{f\in E}s_f(H)
=
O\prn*{(|U|^3+|U|^2\log T)\varepsilon^{-2}}.
\end{equation}
Let $P_f^\alpha$ denote the law of $(Z_f^\alpha,W_f^\alpha)$, and let $U_f^\alpha$ denote the conditional law of $\widetilde p_f^\alpha$ given $Z_f^\alpha=1$, when $p_f^\alpha(H)>0$.
By applying the conditional decomposition in \cref{lem:tv-decomp} with the sampling indicator as the discrete coordinate, and then using the data processing inequality for the map $x\mapsto1/x$, for every $f$ with $p_f^\varepsilon(H)>0$, we obtain
\[
\mathrm{TV}\prn*{P_f^\varepsilon,P_f^\eta}
\le
\abs{p_f^\varepsilon(H)-p_f^\eta(H)}
+
\min\{p_f^\varepsilon(H),p_f^\eta(H)\}
\,
\mathrm{TV}\prn*{U_f^\varepsilon,U_f^\eta}.
\]
Moreover, \cref{lem:uniform-tv-changing-width} gives
\[
\mathrm{TV}\prn*{U_f^\varepsilon,U_f^\eta}
=
O\prn*{
\frac{1}{\varepsilon}
\abs*{1-\frac{p_f^\eta(H)}{p_f^\varepsilon(H)}}
+
\frac{\eta-\varepsilon}{\varepsilon}
}.
\]
Multiplying by $\min\{p_f^\varepsilon(H),p_f^\eta(H)\}\le p_f^\varepsilon(H)$ yields
\[
\mathrm{TV}\prn*{P_f^\varepsilon,P_f^\eta}
=
O\prn*{
\frac{1}{\varepsilon}\abs{p_f^\varepsilon(H)-p_f^\eta(H)}
+
p_f^\varepsilon(H)\frac{\eta-\varepsilon}{\varepsilon}
}.
\]
If $p_f^\varepsilon(H)=0$, then $s_f(H)=0$ and therefore $p_f^\eta(H)=0$, so the same bound is trivial.  Since the coordinates are sampled independently under each parameter, \cref{lem:product-tv-general} gives
\[
\begin{aligned}
\mathrm{TV}\prn*{\mathcal L_\varepsilon(H),\mathcal L_\eta(H)}
&\le
\sum_{f\in E}\mathrm{TV}\prn*{P_f^\varepsilon,P_f^\eta}
\\
&=
O\prn*{
\frac{1}{\varepsilon}
\sum_{f\in E}\abs{p_f^\varepsilon(H)-p_f^\eta(H)}
+
\frac{\eta-\varepsilon}{\varepsilon}
\sum_{f\in E}p_f^\varepsilon(H)
}
\\
&=
O\prn*{
(|U|^3+|U|^2\log T)\varepsilon^{-3}\frac{\eta-\varepsilon}{\varepsilon}
},
\end{aligned}
\]
where the last step uses \eqref{eq:submod-param-p-diff-sum} and \eqref{eq:submod-param-p-sum}, together with $\varepsilon\le1$.

To establish the condition in \cref{assump:adaptive-pc-stability}, let $X$ be a multiset of losses and let $H_X^+$ be its anchored lifted hypergraph.  Define
\[
  \OPT^+(H_X^+)
  \coloneqq
  \min_{S\subseteq V}\mathrm{cut}_{H_X^+}(S\cup\set{r})
  =
  \min_{S\subseteq V}\sum_{\ell\in X}\ell(S),
\]
and set $\varepsilon_X=\varepsilon_{\min}(\OPT^+(H_X^+)+\lambda)$.  For each $x\in X$, let $e_x$ be the corresponding lifted hyperedge and write $Y=X\setminus x$, so that $H_Y^+=H_X^+-e_x$ and $\varepsilon_Y=\varepsilon_{\min}(\OPT^+(H_Y^+)+\lambda)$.
Since the final decision is obtained from the sparsifier by deterministic minimization, the data processing inequality gives\looseness=-1
\[
  \mathrm{TV}\prn*{A_{\varepsilon_X}(Y),A_{\varepsilon_Y}(Y)}
  \le
  \mathrm{TV}\prn*{\mathcal L_{\varepsilon_X}(H_Y^+),\mathcal L_{\varepsilon_Y}(H_Y^+)}.
\]
Since all losses take values in $[0,1]$, we have
\[
  0\le \OPT^+(H_X^+)-\OPT^+(H_Y^+)\le 1.
\]
Indeed, monotonicity follows from non-negativity, and the upper bound follows by evaluating the added loss at an optimizer for $H_Y^+$.  Applying the fixed-hypergraph bound above to $H_Y^+$ with $|U|=|V^+|=n+1$, $\varepsilon=\varepsilon_X$, and $\eta=\varepsilon_Y$, and using \cref{lem:adaptive-rule-calculus} with $\varphi(\varepsilon)= C_{\rm sub}\varepsilon^{-3}(n^3+n^2\log T)$, yield\looseness=-1
\begin{align}
  \mathrm{TV}\prn*{A_{\varepsilon_X}(Y),A_{\varepsilon_Y}(Y)}
  ={}&
  O\prn*{(n^3+n^2\log T)\varepsilon_X^{-3}
  \frac{\varepsilon_Y-\varepsilon_X}{\varepsilon_X}}
  \\
  ={}&
  O\prn*{\varepsilon_X\prn*{\OPT^+(H_X^+)-\OPT^+(H_Y^+)}}.
\end{align}
Averaging this inequality over $x\in X$ gives the condition in \cref{assump:adaptive-pc-stability} for the original submodular function minimization instance, with $C_{\rm pc}=O(1)$.
\end{proof}

\section{\texorpdfstring{Offline algorithm for $\ell_1$ regression}{Offline algorithm for l1 regression}}\label[appendix]{sec:offline-regression}
We present the details of the offline algorithm for $\ell_1$ regression used in \cref{sec:online-regression}.
We use the $\ell_1$ regression algorithm of \citet{Parulekar2021-bm} based on Lewis weight sampling.
Specifically, suppose that we are now at round $t = n$ and have observed $n$ data points.
In this appendix only, let $d_{\rm feat}$ denote the original feature dimension and set $d=d_{\rm feat}+1$ for the augmented matrix used in Lewis-weight sampling, for convenience.
Let $X = \{(a_i^{\rm feat}, b_i)\}_{i=1}^n \subset \mathbb{R}^{d_{\rm feat}} \times \R$ be a dataset, let $A^{\rm feat}\in\mathbb R^{n\times d_{\rm feat}}$ have rows $(a_i^{\rm feat})^\top$, and let $b=(b_1,\dots,b_n)^\top$.
Define the augmented sampling matrix
\[
  A=[A^{\rm feat},-b]\in\mathbb R^{n\times d}.
\]
We write $a_i^\top=((a_i^{\rm feat})^\top,-b_i)$ for the $i$-th row of this augmented matrix, so that $\abs{\inpr{a_i^{\rm feat},\theta}-b_i}=\abs{\inpr{a_i,(\theta,1)}}$.
Rows with $a_i=0$, equivalently $(a_i^{\rm feat},b_i)=(0,0)$, have zero loss for all $\theta$ and may be discarded;
rows with $a_i^{\rm feat}=0$ and $b_i\neq0$ remain non-zero augmented rows.
For notational simplicity, after discarding zero augmented rows, we relabel the remaining rows and write $n$ for their number in the analysis below; the resulting estimate can be converted to an average over the original occurrences.\footnote{
  If zero augmented rows are discarded, deleting such a row does not change the transcript law, and applying the sensitivity bound to the remaining non-zero rows and then averaging over the original~$n$ occurrences gives the $O(\mu_\varepsilon/(\varepsilon n))$ bound in \cref{thm:avg-sensitivity-l1}.
  If no non-zero augmented row remains after discarding them, the empirical loss is identically zero on $\Theta$, and the offline routine returns a fixed deterministic feasible solution.
}
Throughout the rest of this appendix, $A$ denotes this augmented matrix, after zero augmented rows have been discarded when necessary, or a row-deleted version of it, and $d$ denotes its column dimension.  Since $d=d_{\rm feat}+1$, the bounds stated below imply the rates in \cref{sec:online-regression}, after replacing $d_{\rm feat}$ by the symbol $d$ used there for the feature dimension and adjusting absolute constants.

\begin{remark}
  \Citet{Dong2023-yh} also studied online regression, but their argument does not directly apply to the batch-to-online framework.
  Specifically, their offline approximation guarantee is given in terms of the (unsquared) $\ell_2$ error
  based on an $\ell_2$-norm subspace embedding, but this error measure is not additive over
  rounds and therefore cannot be plugged into the batch-to-online framework.
  If one instead considers the squared $\ell_2$ error, the induced loss becomes exp-concave,
  and the standard online Newton method yields a stronger guarantee.
\end{remark}

\subsection{\texorpdfstring{$\ell_1$ Lewis weights and sampling algorithm}{l1 Lewis weights and sampling algorithm}}
We define $\ell_1$ Lewis weights and introduce useful properties.
For a matrix $A\in\mathbb{R}^{n\times d}$, the $\ell_1$ Lewis weights $w(A) = (w_1(A),\dots,w_n(A))$ are defined as the unique solution to the system of equations
\[
  (w_i(A))^2 = a_i^\top \prn*{A^\top W(A)^{-1} A}^{\dagger} a_i,
  \qquad i\in[n],
\]
where we define $W(A) = \mathrm{diag}(w_1(A),\dots,w_n(A))$ and $M^{\dagger}$ denotes the
Moore--Penrose pseudoinverse of a matrix $M$.
Below are basic properties of $\ell_1$ Lewis weights (cf.~\citealt{Cohen2015-te}).
We include a short proof for completeness.
\begin{lemma}[Basic properties of $\ell_1$ Lewis weights]\label[lemma]{lem:lewis-basic}
Let $A\in\mathbb{R}^{n\times d}$ be a matrix with non-zero rows.
The $\ell_1$ Lewis weights $w(A)$ satisfy the following properties:
\begin{enumerate}[leftmargin=*,itemsep=2pt]
  \item Weights are positive: $w_i(A) > 0$ for all $i \in [n]$.
  \item Bound on the sum of weights: $\displaystyle \sum_{i=1}^n w_i(A) = \mathrm{rank}(A) \eqqcolon r \le d$.
  \item Monotonicity: if $A'$ is obtained from $A$ by adding extra rows, then for every original row~$j \in [n]$, we have
    \[
      w_j(A') \le w_j(A).
    \]
\end{enumerate}
\end{lemma}

\begin{proof}
  Let $w = w(A)$ and $W = W(A)$ for brevity.
  We use the fixed-point characterization of $\ell_p$ Lewis weights.
  For $p=1<4$, \citet[Corollary~3.4]{Cohen2015-te} ensures that the system
  $w_i^2 = a_i^\top (A^\top W^{-1}A)^{\dagger}a_i$ admits a unique solution in $\R^n_{>0}$ (by considering projection onto the range of $A^\top$ if necessary; this does not affect the fixed point argument).
  This fact implies the first claim.

  For the second claim, let $B\coloneqq W^{-1/2}A$ and $M \coloneqq A^\top W^{-1}A = B^\top B$.
  Let $b_i^\top$ denote the $i$-th row of $B$, i.e., $b_i = w_i^{-1/2} a_i$.
  Then, from $w_i^2 = a_i^\top M^{\dagger} a_i$, we have
  \[
    \tau_i(B)
    \coloneqq
    b_i^\top (B^\top B)^{\dagger} b_i
    =
    w_i^{-1} a_i^\top M^{\dagger} a_i
    =
    w_i^{-1} w_i^2
    =
    w_i.
  \]
  That is, $w_i$ coincides with the statistical leverage score of the reweighted matrix $B$.
  It holds that
  \[
    \sum_{i=1}^n w_i
    = \sum_{i=1}^n \tau_i(B)
    = \mathrm{tr}\bigl(B(B^\top B)^{\dagger}B^\top\bigr)
    = \mathrm{tr}\bigl((B^\top B)(B^\top B)^{\dagger}\bigr)
    = \mathrm{rank}(B).
  \]
  Finally, since $W^{-1/2}$ is an invertible diagonal matrix, $\mathrm{rank}(B)=\mathrm{rank}(A)=r$, proving the claim.

  For the third claim, monotonicity under row addition holds for $\ell_p$ Lewis weights for all $p\le 2$ (see \citet[Lemma~5.5]{Cohen2015-te}).
  Thus, for our case of $p=1$, adding rows does not increase the Lewis weights of the existing rows.
\end{proof}

The algorithm we analyze builds upon the standard sampling scheme based on Lewis weights, into which we introduce the weight perturbation and Poissonized sample-size technique to control the average sensitivity.

\textbf{Offline $\ell_1$ regression procedure.\;}
Given the augmented matrix $A\in\mathbb{R}^{n\times d}$, the feasible set $\Theta\subseteq\mathbb R^{d_{\rm feat}}$, and a target mean sample size $\mu\ge1$:
\begin{enumerate}[leftmargin=*,itemsep=2pt]
  \item Compute the $\ell_1$ Lewis weights $w(A)$.
  \item Define a probability distribution $p(A)\in\mathbb{R}^n$ by
    \[
      p_i(A) \coloneqq  \frac{w_i(A)}{\sum_{j=1}^n w_j(A)} = \frac{w_i(A)}{r},
      \qquad i\in[n].
    \]
  \item Draw $N_\varepsilon\sim\mathrm{Pois}(\mu)$ and, conditional on $N_\varepsilon$, independently sample $I_1,\dots,I_{N_\varepsilon}$ from $[n]$ according to $p(A)$ (with replacement).
  \item If $N_\varepsilon=0$, return a fixed deterministic feasible solution.  Otherwise, for each sampled occurrence $k\in[N_\varepsilon]$, draw an independent perturbation
    \[
      \widetilde p_k \sim \mathrm{Unif}\prn*{[p_{I_k}(A),(1+\varepsilon/2)p_{I_k}(A)]}.
    \]
    Form the weighted objective
    \[
      \widehat L(\theta)
      =
      \sum_{k=1}^{N_\varepsilon}
      \frac{\abs{\inpr{a_{I_k},(\theta,1)}}}{N_\varepsilon\,\widetilde p_k}.
    \]
  \item Return a deterministic minimizer of $\widehat L$ over $\Theta$, using fixed tie-breaking.
\end{enumerate}
\Citet{Parulekar2021-bm} show that, for any $\varepsilon,\delta\in(0,1]$, a sample size of $m = \Theta \prn*{\tfrac{d}{\varepsilon^2}\log \tfrac{d}{\varepsilon\delta}}$ is sufficient for the unperturbed version of the $\ell_1$ regression algorithm to enjoy a $(1\pm\varepsilon/2)$-approximation guarantee with probability at least $1-\delta$.
Fix a sufficiently large absolute constant $c_{\ell_1}\ge1$ and set
\[
  \mu_\varepsilon
  =
  c_{\ell_1}d\varepsilon^{-2}
  \log\prn*{\mathrm e+\frac{dT}{\varepsilon}}.
\]
Since $d=d_{\rm feat}+1$, this is the same asymptotic dependence on the original feature dimension as the bound stated in \cref{sec:online-regression}, after adjusting absolute constants.
This choice of $\mu_\varepsilon$ is meant to dominate twice the fixed-size threshold $m$ for $\delta=1/T$.
Conditional on every realized count $N_\varepsilon=m'\ge m$, the fixed-size theorem is applied with its sample-size parameter set to $m'$; in particular, each sampled occurrence is weighted by $1/(m'\widetilde p)$ with the corresponding realized perturbation $\widetilde p$.  Hence \cref{lem:fixed-to-poissonized} transfers the approximation guarantee to the Poissonized version.  The additional event $N_\varepsilon<m$ has probability $\mathrm e^{-\Omega(m)}=O(\delta)$ for the above choice of $m$, so the total failure probability remains~$O(\delta)$.
Note that the perturbation step uses the same technique as the one described in \cref{subsec:weight-perturbation}, which only changes the approximation factor from $1\pm\varepsilon/2$ to $1\pm\varepsilon$.
Also, as in \cref{rem:accuracy-translation}, the constant-factor translation between the internal sampling accuracy and the target $(1+\varepsilon)$ guarantee is absorbed into the constants.

\subsection{Change of probability vectors under row deletion}
To analyze the average sensitivity of the above procedure, we examine how the probability vector based on the Lewis weights changes when a row is deleted.

For each $i\in[n]$, let $A^{(-i)}\in\mathbb{R}^{(n-1)\times d}$ denote the matrix obtained from $A$ by deleting the $i$-th row~$a_i^\top$, and let $w(A^{(-i)})$ be the Lewis weight vector of $A^{(-i)}$.
We define $w^{(-i)} \in \R^n$ by inserting a zero in the missing coordinate as follows:
\[
  w^{(-i)} \coloneqq  (w^{(-i)}_1,\dots,w^{(-i)}_n)
  \quad\text{with}\quad
  w^{(-i)}_i \coloneqq  0
  \quad
  \text{and}
  \quad
  w^{(-i)}_j \coloneqq  w_j(A^{(-i)}) \quad \text{for}\; j\neq i.
\]
Note that $A$ can be obtained from $A^{(-i)}$ by adding the
row $a_i^\top$.
Thus, \cref{lem:lewis-basic}~(3) implies
\[
  w_j(A) \le w_j(A^{(-i)}) = w^{(-i)}_j \quad \text{for}\; j\neq i.
\]
From \cref{lem:lewis-basic}~(2), we also have $\sum_{j=1}^n w_j(A) = r$.
Recall that the sampling distribution is
\[
  p(A) = \prn*{p_1(A),\dots,p_n(A)}
  \quad\text{where}\quad
  p_i(A) = \frac{w_i(A)}{r}.
\]
Similarly, for $A^{(-i)}$, define $p(A^{(-i)})$ as the normalized
Lewis weights of $A^{(-i)}$, extended by a zero in the $i$-th coordinate, i.e.,
\[
  p_j(A^{(-i)}) = \frac{w^{(-i)}_j}{r^{(-i)}}
  \quad\text{for}\; j\in[n],
\]
where $r^{(-i)} = \sum_{j=1}^n w^{(-i)}_j = \mathrm{rank}(A^{(-i)}) \le r$.
If $r^{(-i)}=0$, define $p(A^{(-i)})$ as the zero vector on the extended coordinate space.
By using these, we obtain the following lemma on the average change of the probability vectors.
\begin{lemma}\label[lemma]{lem:avg-L1-lewis}
It holds that
\[
  \frac{1}{n}\sum_{i=1}^n
  \norm*{p(A) - p(A^{(-i)})}_1
  \le \frac{2}{n}.
\]
\end{lemma}

\begin{proof}
  Write $w \coloneqq  w(A)$. Fix $i \in [n]$.
  If $r^{(-i)}=0$, then $p(A^{(-i)})$ is the zero vector.  Moreover, $w_j(A)\le w_j(A^{(-i)})=0$ for every $j\neq i$, so $r=\sum_j w_j(A)=w_i(A)$ and hence
  \[
    \norm*{p(A)-p(A^{(-i)})}_1=1\le \frac{2w_i}{r}.
  \]
  Below, we assume $r^{(-i)}>0$. Then, we have
  \begin{align}
    \norm*{p(A) - p(A^{(-i)})}_1
    &= \sum_{j=1}^n \abs*{ \frac{w_j}{r} - \frac{w^{(-i)}_j}{r^{(-i)}} } \\
    &=
    \frac{w_i}{r} +
    \sum_{j\neq i} \prn*{ \frac{w^{(-i)}_j}{r^{(-i)}} - \frac{w_j}{r}} \\
    &=
    \frac{w_i}{r} +
     \frac{\sum_{j\neq i}w^{(-i)}_j}{r^{(-i)}} - \frac{\sum_{j=1}^n w_j - w_i}{r} \\
    &=
    \frac{w_i}{r} + \frac{r^{(-i)}}{r^{(-i)}} - \frac{r - w_i}{r}  \\
    &=
    \frac{2w_i}{r}.
  \end{align}
  Averaging over $i=1,\dots,n$ and using $\sum_{i=1}^n w_i = r$ yield the claim.
\end{proof}

\subsection{Bounding average sensitivity}\label[appendix]{sec:lewis-avg-sen}
We upper bound the average sensitivity of the offline $\ell_1$ regression algorithm.
Fix $\varepsilon\in(0,1]$ and let $p(A)$ denote the single-step sampling distribution on $[n]$.
Conditioned on the Poisson count taking the value $m\in\Z_{>0}$, the algorithm draws $m$ indices i.i.d.\ as
\[
  I_1,\dots,I_m \sim p(A),
\]
and assigns randomly perturbed weights to selected rows.
That is, conditioned on $I_j = k$, we draw
\[
  \widetilde p_{j} \sim \mathrm{Unif}\prn*{[p_k(A),(1+\varepsilon/2)p_k(A)]}
\]
independently and set the weight to
\[
  W_j \coloneqq \frac{1}{m\,\widetilde p_{j}}.
\]
If $p_k(A)=0$, then $\Pr[I_j=k]=0$ holds and the value of $W_j$ on that event is irrelevant.
We write $(I,W)=(I_1,\dots,I_m,W_1,\dots,W_m)$ for the full random output of the sampling procedure, and view the final $\ell_1$-regression step on the sampled weighted subset as a deterministic function of $(I,W)$.

\textbf{Bounding total variation for a single draw $(I_j,W_j)$.\;}
Fix $i\in[n]$ and let $A^{(-i)}$ be the matrix with the $i$-th row removed.
In this paragraph, assume that $p(A^{(-i)})$ is a probability vector.
Let $P$ and~$Q$ denote the laws of the single-draw pair $(I_j,W_j)$ under inputs $A$ and $A^{(-i)}$, respectively; below we omit the subscript~$j$ for brevity.
We apply the conditional decomposition (\cref{lem:tv-decomp}) with $\mathcal Z=[n]$ and $\mathcal W=\R_{\ge0}$:
\begin{equation}
  \mathrm{TV}(P,Q)
  \le
  \mathrm{TV}\!\left(p(A),p(A^{(-i)})\right)
  \!+\!
  \sum_{k\in[n]}
  \min\{p_k(A),p_k(A^{(-i)})\}\,
  \mathrm{TV}\!\left(P_{W\mid I=k},Q_{W\mid I=k}\right),
  \label{eq:l1reg-tv-one-draw-decomp}
\end{equation}
where $P_{W\mid I=k}$ and $Q_{W\mid I=k}$ denote the conditional distributions of $W$ given $I=k$ under $P$ and~$Q$, respectively.
The first term equals $\frac12\|p(A)-p(A^{(-i)})\|_1$.
For the second term, coordinates with $p_k(A)=0$ or $p_k(A^{(-i)})=0$ contribute nothing because the coefficient $\min\{p_k(A),p_k(A^{(-i)})\}$ is zero.  For the remaining coordinates, the conditional laws of $W\mid(I=k)$ are images of the corresponding uniform perturbation intervals under the map $\widetilde p\mapsto 1/(m\widetilde p)$.
Since this deterministic mapping does not increase the total variation distance, \cref{lem:uniform-tv} gives
\[
  \mathrm{TV}\prn*{P_{W\mid I=k},Q_{W\mid I=k}}
  \le
  \frac{1+\varepsilon/2}{\varepsilon/2}\,
  \abs*{1-\frac{p_k(A^{(-i)})}{p_k(A)}}
  \qquad
  \text{for\; $p_k(A)>0$\; and\; $p_k(A^{(-i)})>0$},
\]
and the zero-probability coordinates are covered by the vanishing coefficient noted above.
Therefore, using $\min\{a,b\}\,\bigl|1-b/a\bigr|\le |a-b|$ for the positive-probability coordinates, we obtain
\begin{align}
  \sum_{k\in[n]}
  \min\{p_k(A),p_k(A^{(-i)})\}\,
  \mathrm{TV}\!\left(P_{W\mid I=k},Q_{W\mid I=k}\right)
  &\le
  \frac{2+\varepsilon}{\varepsilon}\,
  \sum_{k\in[n]}\abs{p_k(A)-p_k(A^{(-i)})}
  \nonumber\\
  &=
  \frac{2+\varepsilon}{\varepsilon}\,
  \|p(A)-p(A^{(-i)})\|_1.
\end{align}
Plugging this into \eqref{eq:l1reg-tv-one-draw-decomp} yields
\begin{equation}\label{eq:l1reg-tv-one-draw-final}
  \mathrm{TV}(P,Q)
  \le
  \prn*{\frac12 + \frac{2+\varepsilon}{\varepsilon}}\,
  \|p(A)-p(A^{(-i)})\|_1
  \le
  \frac{4}{\varepsilon}\,\|p(A)-p(A^{(-i)})\|_1,
\end{equation}
where the last inequality uses $\varepsilon\in(0,1]$.

\textbf{From one draw to $m$ draws.\;}
Let $P_{A,m}$ and $P_{A^{(-i)},m}$ denote the laws of the count-$m$ transcripts under inputs $A$ and $A^{(-i)}$, respectively.  When $p(A^{(-i)})$ is a probability vector, the count-$m$ laws are products of the single-draw laws $P$ and $Q$ above, and hence \cref{lem:product-tv-general} and \eqref{eq:l1reg-tv-one-draw-final} give
\[
  \mathrm{TV}\prn*{P_{A,m},P_{A^{(-i)},m}}
  \le
  \frac{4m}{\varepsilon}\,\|p(A)-p(A^{(-i)})\|_1.
\]
If $p(A^{(-i)})$ is the zero vector, the deleted input produces a fixed deterministic transcript; viewing this transcript as an element of the count-$m$ transcript space, we have
\[
  \mathrm{TV}\prn*{P_{A,m},P_{A^{(-i)},m}}
  \le 1
  \le
  \frac{4m}{\varepsilon}\,
  \|p(A)-p(A^{(-i)})\|_1,
\]
because $\|p(A)-p(A^{(-i)})\|_1=\|p(A)\|_1=1$ and $\varepsilon\le1$.  Therefore, for every $m\ge1$, we have
\begin{equation}\label{eq:l1reg-tv-m-draws}
  \mathrm{TV}\prn*{
    P_{A,m},P_{A^{(-i)},m}
  }
  \le
  \frac{4m}{\varepsilon}\,\|p(A)-p(A^{(-i)})\|_1.
\end{equation}
We now bound the average sensitivity of the Poissonized $\ell_1$ regression algorithm.
\begin{proposition}\label{thm:avg-sensitivity-l1}
Let $A\in\mathbb{R}^{n\times d}$, fix $\varepsilon\in(0,1]$, and let the Poissonized algorithm use mean sample count $\mu_\varepsilon$.  Let $\mathcal L_\varepsilon(A)$ be the output law of the resulting offline $\ell_1$ regression algorithm.
Then, we have
\[
  \frac{1}{n}\sum_{i=1}^n
  \mathrm{TV}\prn*{\mathcal L_\varepsilon(A),\mathcal L_\varepsilon(A^{(-i)})}
  =
  O\prn*{\frac{\mu_\varepsilon}{\varepsilon n}}.
\]
\end{proposition}

\begin{proof}
Fix $i\in[n]$.  Since the final optimization step is deterministic given the sampled indices and weights, the data processing inequality reduces the claim to the total variation distance between sampling transcripts.  Apply \cref{lem:poissonized-transcript-tv} with equal Poisson means and with $P_m,Q_m$ equal to the count-$m$ transcript laws for $A$ and $A^{(-i)}$.  Conditional on the common count $m$, \eqref{eq:l1reg-tv-m-draws} applies for~$m\ge1$, while for $m=0$ both procedures produce the same fixed deterministic transcript.
By \cref{lem:poissonized-transcript-tv} with equal Poisson means, averaging over $m\sim\mathrm{Pois}(\mu_\varepsilon)$ gives\looseness=-1
\[
  \mathrm{TV}\prn*{
    \mathcal L_\varepsilon(A),\mathcal L_\varepsilon(A^{(-i)})
  }
  \le
  \frac{4\mu_\varepsilon}{\varepsilon}\,\|p(A)-p(A^{(-i)})\|_1.
\]
Averaging over $i \in [n]$ and using \cref{lem:avg-L1-lewis} completes the proof.
\end{proof}

With the explicit mean $\mu_\varepsilon$ above, fix a sufficiently large constant $C_{\ell_1}\ge c_{\ell_1}$ and use the following admissible choice of $\varphi$ in \cref{assump:approx-avesen}:
\[
\varphi(\varepsilon)
=
C_{\ell_1}\frac{d}{\varepsilon^{3}}\log \prn*{\mathrm e+\frac{d T}{\varepsilon}}.
\]
Since $d=d_{\rm feat}+1$, increasing the absolute constant if necessary gives the $C_{\ell_1}d_{\rm feat}\varepsilon^{-3}\log(\mathrm e+d_{\rm feat}T/\varepsilon)$ rate, which is written with $d$ in the main text.

\subsection{Checking Assumption~\ref{assump:adaptive-pc-stability}}\label[appendix]{sec:check-pc-stab-regression}
We verify that the above offline $\ell_1$ regression algorithm, equipped with the weight perturbation and Poissonized sample-size technique, satisfies \cref{assump:adaptive-pc-stability} with $C_{\rm pc}=O(1)$.
\begin{proposition}\label[proposition]{prop:l1-same-input-pc}
Let $\mathcal L_\varepsilon(A)$ be the output law of the Poissonized Lewis-weight sampling algorithm on a fixed input matrix $A\in\mathbb R^{n\times d}$, with mean sample count $\mu_\varepsilon=c_{\ell_1}d\varepsilon^{-2}\log(\mathrm e+dT/\varepsilon)$.  For $0<\varepsilon\le\eta\le1$, it holds that
\[
\mathrm{TV}\prn*{\mathcal L_\varepsilon(A),\mathcal L_\eta(A)}
=
O\prn*{
\mu_\varepsilon\frac{\eta-\varepsilon}{\varepsilon}
}.
\]
Consequently, with $\varphi(\varepsilon)=C_{\ell_1}d\varepsilon^{-3}\log(\mathrm e+dT/\varepsilon)$ for a sufficiently large constant $C_{\ell_1} > 0$, the Poissonized Lewis-weight sampling family satisfies \cref{assump:adaptive-pc-stability} with $C_{\rm pc}=O(1)$; as above, this implies the feature-dimension rate stated in the main text after adjusting constants.
\end{proposition}
\begin{proof}
For a fixed matrix $A$, the original Lewis sampling distribution $p(A)$ is independent of the accuracy parameter.  Thus changing $\varepsilon$ to $\eta$ affects the transcript only through the Poisson mean and the width of the weight perturbation interval. Viewing $\mu_\varepsilon$ as a function of $\varepsilon$, it is non-increasing and satisfies $|\mu_\varepsilon'|= O\prn*{\frac{\mu_\varepsilon}{\varepsilon}}$.
Therefore, for $0<\varepsilon\le\eta\le1$, the mean-value theorem gives
\[
|\mu_\varepsilon-\mu_\eta|
=
O\prn*{\sup_{x\in[\varepsilon,\eta]}\frac{\mu_x}{x}(\eta-\varepsilon)}
=
O\prn*{\mu_\varepsilon\frac{\eta-\varepsilon}{\varepsilon}}.
\]
Condition on a common count $m$.  For each sampled index $k$ with $p_k(A)>0$, let $U_{k,m}^\alpha$ denote the conditional law of $1/(m\tilde p)$ when $\tilde p\sim\mathrm{Unif}([p_k(A),(1+\alpha/2)p_k(A)])$.  By \cref{lem:uniform-tv-changing-width} and the data processing inequality, for $m\ge1$, we have
\[
\mathrm{TV}\prn*{U_{k,m}^{\varepsilon},U_{k,m}^{\eta}}
=
O\prn*{\frac{\eta-\varepsilon}{\varepsilon}},
\]
uniformly in $k$ and $m$.  Hence \cref{lem:product-tv-general} bounds the count-$m$ transcript total variation by $O(m(\eta-\varepsilon)/\varepsilon)$, while the count-$0$ transcript is common to both laws.  Applying \cref{lem:poissonized-transcript-tv} gives
\[
\mathrm{TV}\prn*{\mathcal L_\varepsilon(A),\mathcal L_\eta(A)}
=
O\prn*{
|\mu_\varepsilon-\mu_\eta|
+\mu_\varepsilon\frac{\eta-\varepsilon}{\varepsilon}
}
=
O\prn*{\mu_\varepsilon\frac{\eta-\varepsilon}{\varepsilon}}.
\]
It remains to convert this same-input bound into \cref{assump:adaptive-pc-stability}.  For a dataset $X$ and $x\in X$, let $Y=X\setminus x$ and set $\varepsilon=\varepsilon_X$ and $\eta=\varepsilon_Y$.  Since losses are non-negative and bounded by one,
\[
0\le \OPT(X)-\OPT(Y)\le1,
\qquad
\varepsilon_X\le\varepsilon_Y.
\]
Using the above total variation bound on the fixed input $Y$ and the chosen constant $C_{\ell_1}$, which makes $\mu_\varepsilon/\varepsilon=O(\varphi(\varepsilon))$, gives
\[
\mathrm{TV}\prn*{\mathcal L_{\varepsilon_X}(Y),\mathcal L_{\varepsilon_Y}(Y)}
=
O\prn*{\varphi(\varepsilon_X)(\varepsilon_Y-\varepsilon_X)}.
\]
Thus, \cref{lem:adaptive-rule-calculus}, applied with $u=\OPT(X)+\lambda$ and $v=\OPT(Y)+\lambda$, gives
$\varphi(\varepsilon_X)(\varepsilon_Y-\varepsilon_X)=O(\varepsilon_X^2(\OPT(X)-\OPT(Y)))$.  Since we have $\varepsilon_X\le1$, this is at most $O(\varepsilon_X(\OPT(X)-\OPT(Y)))$.  Averaging over $x\in X$ proves \cref{assump:adaptive-pc-stability} with $C_{\rm pc}=O(1)$.
\end{proof}

\section{Details of other applications}\label[appendix]{sec:dong-yoshida-applications}
We present the details of the applications of \citet{Dong2023-yh} to online $(k,z)$-clustering and online low-rank matrix approximation.
Below, the assumption that losses lie in $[0,1]$ is imposed on the instance and the feasible decision class.

\textbf{A useful common calculus.\;}
We will use the following elementary calculus to control changes in the Poisson means.
For $a>0$, $B\ge0$, $C\ge1$ and $M(\alpha)=C\alpha^{-a}\log(\mathrm e+B/\alpha)$ on $(0,1]$, the function $M$ is non-increasing and satisfies
\[
  -M'(s)
  \le
  (a+1)\frac{M(s)}{s}
  \qquad (0<s\le1).
\]
Consequently, for $0<\varepsilon\le\eta\le1$, we have
\begin{equation}\label{eq:power-log-mean-lipschitz}
  |M(\varepsilon)-M(\eta)|
  \le
  (a+1)M(\varepsilon)\frac{\eta-\varepsilon}{\varepsilon}.
\end{equation}
Indeed, the derivative bound follows by differentiating $M$, using $\log(\mathrm e+B/s)\ge1$, and bounding the derivative of the logarithmic factor by $1/s$; integrating the derivative from $\varepsilon$ to $\eta$ gives \eqref{eq:power-log-mean-lipschitz}.

\subsection{\texorpdfstring{Online $(k,z)$-clustering}{Online (k,z)-clustering}}\label[appendix]{subsec:online-kz}
\citet{Dong2023-yh} instantiate the batch-to-online framework for the following \emph{online $(k,z)$-clustering} problem \citep{Cohen-Addad2021-fb}.
\begin{remark}
  For this clustering application, the selection rule \eqref{eq:rule} requires access to the exact prefix optimum $\OPT(Y)$, or to an oracle that provides a monotone surrogate satisfying the stability properties used in \cref{assump:adaptive-pc-stability}.
In this paper, the clustering result should therefore be interpreted as an oracle-based refinement of the fixed-$\varepsilon$ guarantee of \citet{Dong2023-yh};
obtaining a fully polynomial-time implementation with approximate prefix optima requires an additional analysis.
\end{remark}

\textbf{Problem setup.\;}
Fix $k\in\mathbb{Z}_{>0}$ and $z\ge1$.  Let $\mathcal K\subset\R^d$ be a known bounded center domain and set $\Theta=\set*{\{c_1,\ldots,c_k\}:c_j\in\mathcal K\ \text{for all }j\in[k]}$.  At each round $t$, the learner plays a set of $k$ centers $\theta_t=\{c_{t,1},\ldots,c_{t,k}\}\in\Theta$ and incurs the loss defined by
\[
  \ell\prn*{\theta_t,x_t}
  =
  \min_{j\in[k]}\norm{x_t-c_{t,j}}_2^z.
\]
Then, the learner observes a data point $x_t\in\R^d$ and updates the centers for the next round.
We assume that the adversarially chosen points and the domain $\mathcal K$ are scaled so that $\ell(\theta,x)\in[0,1]$ for every feasible $\theta\in\Theta$ and every admissible point $x$.

We use a Poissonized version of the two-stage importance-sampling coreset construction of \citet{Huang2020-cb}, in the perturbed form used by \citet{Dong2023-yh}.
The following description highlights the aspects relevant to our analysis; see \citet[Algorithm~7]{Dong2023-yh} for the full details.\looseness=-1

\textbf{Poissonized coreset procedure.\;}
Given a finite multiset $Y$ of data points and an accuracy parameter $\varepsilon\in(0,1]$, the offline algorithm proceeds as follows.
\begin{enumerate}[leftmargin=*,itemsep=2pt]
  \item Compute the preliminary objects used in the Huang--Vishnoi two-stage construction, following the organization of \citet[Algorithm~7, Lines~2--4]{Dong2023-yh}.  Namely, compute a preliminary center set $C_Y^\star$ by the $D^z$-sampling subroutine, assign each $y\in Y$ to its closest center $c_Y^\star(y)\in C_Y^\star$ using fixed tie-breaking, form the cells $Y^c=\set*{y\in Y:c_Y^\star(y)=c}$, and compute the first-stage scores $\sigma_{1,Y}(y)$.  The randomness used to construct $C_Y^\star$ is included in the transcript and is generated independently of the target accuracy parameter.  We use the normalized first-stage distribution $p_{1,Y}(y)=\sigma_{1,Y}(y)/\sum_{y\in Y}\sigma_{1,Y}(y)$.  Any accuracy-dependent scalar that is common to all first-stage scores cancels in this normalization and is absorbed into the sample-size constants.
  \item Let $m_1(\varepsilon)$ and $m_2(\varepsilon)$ denote the deterministic sample-size thresholds for the first and second importance-sampling stages in \citet[Algorithm~7, Lines~5 and~10]{Dong2023-yh}, with constants chosen for failure probability $O(1/T)$.  Choose Poisson means $\mu_j(\varepsilon)$ so that $\mu_j(\varepsilon)/2\ge m_j(\varepsilon)$ for $j=1,2$, and draw independent counts $N_j\sim\mathrm{Pois}(\mu_j(\varepsilon))$.  Concretely, after absorbing the $k$- and $z$-dependent and failure-probability factors into sufficiently large constants $c_{1,k,z},c_{2,k,z}\ge1$, we take
  \[
    \mu_1(\varepsilon)
    =
    c_{1,k,z}\varepsilon^{-(5z+15)}
    \log\prn*{\mathrm e+\frac{T}{\varepsilon}},
    \qquad
    \mu_2(\varepsilon)
    =
    c_{2,k,z}\varepsilon^{-(2z+2)}
    \log\prn*{\mathrm e+\frac{T}{\varepsilon}}.
  \]
  Since $z\ge1$ and $\varepsilon\le1$, their sum is bounded by
  \[
    \mu(\varepsilon)\coloneqq
    \mu_1(\varepsilon)+\mu_2(\varepsilon)
    =
    O_{k,z}\prn*{\varepsilon^{-(5z+15)}
    \log\prn*{\mathrm e+\frac{T}{\varepsilon}}}.
  \]
  \item In the first stage, sample $N_1$ points independently with replacement from $p_{1,Y}$.  For each sampled point, draw the perturbed first-stage weight coordinate $u_Y$ as in \citet[Algorithm~7, Lines~6--7]{Dong2023-yh}.  For the parameter-change comparison, whenever this first-stage weight is used in the center-weight correction of \citet[Algorithm~7, Line~13]{Dong2023-yh}, we record the scaled coordinate
  \[
    r_Y^\varepsilon=(1+10\varepsilon)u_Y
  \]
  instead of the raw coordinate $u_Y$, changing coordinates in the recorded transcript.  Multiplying all first-stage weights by the same positive factor does not change the normalized second-stage sampling distribution, so the second-stage law can be expressed as a deterministic function of the recorded first-stage transcript.
  \item In the second stage, after the first-stage transcript is fixed, sample $N_2$ points independently with replacement from the normalized second-stage distribution and draw the perturbed second-stage weight coordinates as in \citet[Algorithm~7, Lines~10--12]{Dong2023-yh}.  With the scaled first-stage coordinates from Step~3, the center-weight correction in \citet[Algorithm~7, Line~13]{Dong2023-yh} can be written as
  \[
    w_Y(c)
    =
    \sum_{x\in D_c} r_Y^\varepsilon(x)
    -
    \sum_{x\in D^2\cap D_c} w_Y(x).
  \]
Hence this correction is a deterministic affine function of the recorded first-stage coordinates and the perturbed second-stage weights.  All random weight coordinates appearing in this affine formula are included in the transcript.
  \item Return the complete weighted coreset transcript $(D^2\cup C_Y^\star,w)$, as in \citet[Algorithm~7, Lines~14--16]{Dong2023-yh}.  For the stability statement below, we take the downstream solver to be exact deterministic minimization on the coreset with fixed tie-breaking, so the coreset-to-solution map is parameter-independent deterministic post-processing.
\end{enumerate}
The regret guarantee below is stated under access to this deterministic exact solver for the coreset objective.  If it is replaced by an approximation scheme, the solver accuracy must be incorporated into the offline approximation guarantee, or chosen small enough to be absorbed into the stated approximation factor.

Let $\mathcal C_\varepsilon(Y)$ denote the law of the complete Poissonized two-stage coreset transcript on a fixed clustering dataset $Y$ and accuracy parameter $\varepsilon$, and let $\mathcal L_\varepsilon(Y)$ denote the output law of the above procedure.
Let $C_{k,z} > 0$ be a sufficiently large constant depending only on $k$ and $z$ and set
\[
  \varphi(\varepsilon)
  =
  C_{k,z}\varepsilon^{-5z-16}
  \log\prn*{\mathrm e+\frac{T}{\varepsilon}}
  \asymp
  \frac{\mu(\varepsilon)}{\varepsilon}.
\]
The additional $\varepsilon^{-1}$ factor is the cost of comparing the perturbed reweighting coordinates.\footnote{
  In \citet[Appendix~C]{Dong2023-yh}, the corresponding clustering average-sensitivity
  bound is stated with an $\varepsilon^{-5z-15}$ dependence. However, their calculation does not charge the contribution of the random reweighting variables to the total variation. Charging these variables gives the additional $\varepsilon^{-1}$ factor above. Therefore, the larger $\varepsilon^{-5z-16}$ dependence comes from the correct analysis of the sampling law including the rescaling variables.
}

\begin{proposition}[Fixed-accuracy guarantee for clustering]\label[proposition]{prop:clustering-fixed-sensitivity}
For every fixed $\varepsilon\in(0,1]$, the Poissonized two-stage procedure above returns a $(1+\varepsilon)$-approximate solution in expectation, up to an additive $O(t/T)$ term on every prefix of size $t$, and its average sensitivity is at most $\varphi(\varepsilon)/(2t)$.
\end{proposition}
\begin{proof}
First consider the approximation guarantee.  Conditional on $N_j=m_j\ge\mu_j(\varepsilon)/2$ for~$j=1,2$, the deterministic-count coreset guarantee underlying \citet[Algorithm~7 and Theorem~C.1]{Dong2023-yh} is applied with realized stage sizes $m_1,m_2$.  The Poisson means are specified so that these realized sizes exceed the deterministic thresholds required by the two-stage coreset construction of \citet{Huang2020-cb}, with failure probability $O(1/T)$.  Since $N_j\sim\mathrm{Pois}(\mu_j(\varepsilon))$, the probability of $N_j<\mu_j(\varepsilon)/2$ is exponentially small in $\mu_j(\varepsilon)$ and is absorbed into the $O(1/T)$ failure probability.  On the success event, the coreset approximation and deterministic downstream solver give a $(1+\varepsilon)$-approximate solution;
as in \cref{rem:accuracy-translation}, any constant-factor rescaling between the internal coreset accuracy and the target accuracy is absorbed into the constants.
On the failure event, the prefix loss is at most $t$, and thus its contribution to the expected prefix loss is $O(t/T)$.

For average sensitivity, we use the deterministic-count two-stage sensitivity estimate established by \citet[Theorem~C.1]{Dong2023-yh} for the same perturbed coreset construction.  Their proof bounds the two importance-sampling stages separately and includes the perturbed weights from \citet[Algorithm~7, Lines~6 and~11]{Dong2023-yh}.  Thus, for fixed stage sizes $m_1,m_2$ and fixed $\varepsilon$, the average-deletion total variation of the complete perturbed coreset transcript is bounded by
\[
  O_{k,z}\prn*{\frac{m_1+m_2}{\varepsilon t}},
\]
where $O_{k,z}(\cdot)$ hides constants depending only on $k$ and $z$.  The final center computation is deterministic post-processing and therefore does not increase the total variation.  Taking expectation over the independent Poisson counts yields
\[
  \frac1t\sum_{y\in Y}
  \mathrm{TV}\prn*{\mathcal L_\varepsilon(Y),\mathcal L_\varepsilon(Y\setminus y)}
  =
  O_{k,z}\prn*{\frac{\mu(\varepsilon)}{\varepsilon t}}
  \lesssim
  \frac{\varphi(\varepsilon)}{2t},
\]
where the last inequality follows by choosing the constant in $C_{k,z}$ sufficiently large.
\end{proof}
We then verify the parameter-change stability in \cref{assump:adaptive-pc-stability}.
The only parameter-dependent random coordinates in the transcript are the Poisson counts and the one-dimensional perturbation variables. The discrete sampling distributions, after conditioning on the previous-stage transcript, are parameter-independent under our transcript representation.
Formalizing this idea gives the following result.\looseness=-1
\begin{proposition}[Parameter-change stability for clustering]\label[proposition]{prop:clustering-same-input-pc}
  For $0<\varepsilon\le\eta\le1$, it holds that
\[
  \mathrm{TV}\prn*{\mathcal L_\varepsilon(Y),\mathcal L_\eta(Y)}
  \le
  \mathrm{TV}\prn*{\mathcal C_\varepsilon(Y),\mathcal C_\eta(Y)}
  =
  O_{k,z}\prn*{\varphi(\varepsilon)(\eta-\varepsilon)}.
\]
Consequently, the Poissonized clustering coreset family satisfies \cref{assump:adaptive-pc-stability} with $C_{\rm pc}=O_{k,z}(1)$.
\end{proposition}
\begin{proof}
The first inequality follows from the data processing inequality, since the downstream center solver is deterministic post-processing.
We compare the two complete transcripts on the same input~$Y$.  The Step~1 preliminary transcript is generated independently of the target accuracy parameter: it consists of the $D^z$-sampling randomness, the preliminary center set $C_Y^\star$, the nearest-center map, and the induced cell decomposition, as in \citet[Algorithm~7, Lines~2--3]{Dong2023-yh}.
The first-stage scores $\sigma_{1,Y}$ and the corresponding normalized sampling law
$
  p_{1,Y}(y)={\sigma_{1,Y}(y)}/{\sum_{y'\in Y}\sigma_{1,Y}(y')}
$
are formed as in \citet[Algorithm~7, Lines~4--5]{Dong2023-yh}.  Although the unnormalized score formula in \citet[Algorithm~7, Line~4]{Dong2023-yh} contains a global factor depending on the target accuracy, this factor is common to all $y\in Y$ and therefore cancels in the normalized law $p_{1,Y}$.  Thus the preliminary transcript and the first-stage discrete sampling law can be coupled identically under $\varepsilon$ and $\eta$; the target-accuracy-dependent sample sizes and perturbation intervals enter only after this law has been fixed, starting from \citet[Algorithm~7, Lines~5--7]{Dong2023-yh}.

For each stage $j\in\{1,2\}$, the mean function $\mu_j$ has the form covered by \eqref{eq:power-log-mean-lipschitz}: the first stage has exponent $5z+15$, while the second stage has exponent $2z+2$. Therefore, we have
\[
  \abs*{\mu_j(\varepsilon)-\mu_j(\eta)}
  =
  O_{k,z}\prn*{\mu_j(\varepsilon)\frac{\eta-\varepsilon}{\varepsilon}}.
\]
By \cref{lem:poissonized-transcript-tv}, the total variation distance between the two Poisson count laws in stage $j$ is bounded by the right-hand side.

After separating the contribution from the Poisson count laws in Step~2, we compare the remaining conditional laws by an explicit coupling.
Fix a stage $j\in\{1,2\}$.  By \cref{lem:tv-decomp}, the total variation distance between the joint transcript laws is bounded by the total variation distance between the previous-stage transcript laws plus the expected total variation distance between the current-stage conditional laws given a common previous-stage transcript.  Thus, when analyzing stage $j$, we condition on a common recorded transcript from the previous stages.  For $j=1$, this previous transcript is the preliminary transcript from Step~1.  For $j=2$, it also contains the first-stage sampled occurrences and their recorded reweighting coordinates from Step~3.  We then condition further on a common count $m$ in the current stage.  Under this conditioning, the current-stage discrete sampling distribution is the same under $\varepsilon$ and $\eta$, so we couple the $m$ sampled occurrences identically.  Consequently, the sampled occurrences themselves contribute no additional total variation distance, and it remains only to compare the conditional laws of the recorded scalar reweighting coordinates introduced in Steps~3--4.\looseness=-1

The remaining parameter-dependent randomness consists of reweighting coordinates attached to these coupled sampled occurrences.  A first-stage occurrence contributes its perturbed first-stage weight, which we record as the scaled coordinate $r_Y^\alpha=(1+10\alpha)u_Y^\alpha$ from Step~3, where $\alpha\in\{\varepsilon,\eta\}$; a second-stage occurrence contributes its perturbed second-stage weight from Step~4.
Note that the center-weight correction in Step~4 introduces no additional randomness: with the scaled first-stage coordinates recorded, it is the parameter-independent affine function
\[
  w_Y(c)
  =
  \sum_{x\in D_c} r_Y^\alpha(x)
  -
  \sum_{x\in D^2\cap D_c} w_Y(x).
\]
Thus, each sampled occurrence is associated with only $O_{k,z}(1)$ one-dimensional coordinates whose laws depend on the accuracy parameter.  Moreover, once the scaled first-stage coordinates are fixed, the corresponding raw first-stage weights under $\varepsilon$ and $\eta$ are obtained by multiplying the same recorded coordinates by $1/(1+10\varepsilon)$ and $1/(1+10\eta)$, respectively.  This is a common positive rescaling across all first-stage weights in the normalized second-stage scores;
therefore, the conditional second-stage sampling distribution is unchanged.

We now bound the total variation distance arising from these recorded coordinates.  Coordinates with zero base scale are deterministic under both parameters and may be ignored.  Consider a recorded coordinate with positive base scale.  Under accuracy parameter $\alpha\in\{\varepsilon,\eta\}$, its law is obtained from a one-dimensional uniform perturbation interval, followed by a deterministic coordinate map.  For the second-stage weights, the base scale is fixed under the conditioning above and only the perturbation width changes from order $\varepsilon$ to order $\eta$.  For the scaled first-stage coordinate $r_Y^\alpha$, the deterministic multiplier $1+10\alpha$ changes the interval endpoints by a relative $1+O_{k,z}(\eta-\varepsilon)$ factor, while the perturbation width remains of order $\alpha$.  Therefore, \cref{lem:uniform-tv-changing-width}, together with the data processing inequality through the deterministic coordinate maps, implies that the total variation distance between the two conditional laws of each recorded coordinate is at most
\[
  O_{k,z}\prn*{\frac{\eta-\varepsilon}{\varepsilon}}.
\]
Since each coupled sampled occurrence carries only $O_{k,z}(1)$ recorded coordinates, \cref{lem:product-tv-general} implies that the total variation distance between the conditional laws of all recorded coordinates associated with one sampled occurrence is also $O_{k,z}\prn*{\frac{\eta-\varepsilon}{\varepsilon}}$.
Applying \cref{lem:product-tv-general} again over the~$m$ coupled sampled occurrences, the total variation distance between the conditional laws of all current-stage recorded coordinates, given the common previous transcript, the common count, and the common sampled occurrences, is at most
\[
  O_{k,z}\prn*{m\frac{\eta-\varepsilon}{\varepsilon}}.
\]
Taking expectation with respect to $m\sim\mathrm{Pois}(\mu_j(\varepsilon))$ yields
\[
  O_{k,z}\prn*{\mu_j(\varepsilon)\frac{\eta-\varepsilon}{\varepsilon}}
\]
for the current-stage reweighting coordinates.  Adding the Poisson-count bound for the same stage, and summing the two stage-wise bounds using \cref{lem:tv-decomp}, we obtain
\[
  \mathrm{TV}\prn*{\mathcal C_\varepsilon(Y),\mathcal C_\eta(Y)}
  =
  O_{k,z}\prn*{\mu(\varepsilon)\frac{\eta-\varepsilon}{\varepsilon}}.
\]

It remains to verify \cref{assump:adaptive-pc-stability}.  Fix a dataset $X$ and $x\in X$, let $Y=X\setminus x$, and set $\varepsilon=\varepsilon_X$ and $\eta=\varepsilon_Y$.  Since the losses are non-negative and bounded by one, we have
\[
  0\le\OPT(X)-\OPT(Y)\le1,
  \qquad
  \varepsilon_X\le\varepsilon_Y.
\]
Using the above bound for fixed $Y$, the data processing inequality, and the chosen constant $C_{k,z}$, which makes $\mu(\varepsilon)/\varepsilon=O_{k,z}(\varphi(\varepsilon))$, we obtain
\[
  \mathrm{TV}\prn*{\mathcal L_{\varepsilon_X}(Y),\mathcal L_{\varepsilon_Y}(Y)}
  \le
  \mathrm{TV}\prn*{\mathcal C_{\varepsilon_X}(Y),\mathcal C_{\varepsilon_Y}(Y)}
  =
  O_{k,z}\prn*{\varphi(\varepsilon_X)(\varepsilon_Y-\varepsilon_X)}.
\]
By applying \cref{lem:adaptive-rule-calculus} with $u=\OPT(X)+\lambda$ and $v=\OPT(Y)+\lambda$, the right-hand side is $O_{k,z}(\varepsilon_X^2(\OPT(X)-\OPT(Y)))$, and this is at most $O_{k,z}(\varepsilon_X(\OPT(X)-\OPT(Y)))$ as $\varepsilon_X\le1$.  Averaging over $x\in X$ proves \cref{assump:adaptive-pc-stability}.
\end{proof}

\textbf{Regret guarantee.\;}
By \cref{prop:clustering-fixed-sensitivity,prop:clustering-same-input-pc}, the offline algorithm family satisfies \cref{assump:approx-avesen} with $\varphi=
  C_{k,z}\varepsilon^{-5z-16}
  \log\prn*{\mathrm e+\frac{T}{\varepsilon}}$ and \cref{assump:adaptive-pc-stability} with $C_{\rm pc}=O_{k,z}(1)$.  Applying \cref{cor:simple-form} with $\lambda=H_T^{-(q+1)/q}$, $q=5z+16$, $C_1=C_{k,z}$, $C_2=T$, and $C_{\rm pc}=O_{k,z}(1)$ yields
\[
  \Reg_T
  =
  \tilde O\prn*{
  C_{k,z}
  +
  C_{k,z}^{\tfrac{1}{5z+17}}
  \prn*{1+\OPT_T^{\tfrac{5z+16}{5z+17}}}
  }.
\]
The guarantee is under access to a deterministic exact solver for the coreset objective with fixed tie-breaking. If this solver is replaced by an approximation scheme, the regret statement should be interpreted with the corresponding approximation factor, or the solver accuracy must be chosen small enough and treated as part of the offline approximation guarantee.

\subsection{Online low-rank matrix approximation}\label[appendix]{sec:online-lrma}
We then discuss online low-rank matrix approximation, which is applied to, for example, recommendation systems and experimental design \citep{Warmuth2008-ol,Nie2016-bj}.

\textbf{Problem setup.\;}
At each round $t$, the learner plays an orthogonal matrix $Z_t\in\R^{d\times k}$, whose columns form an orthonormal basis of a $k$-dimensional subspace, and incurs the loss of
\[
  \ell\prn*{Z_t,a_t}
  =
  \norm{a_t-Z_tZ_t^\top a_t}_2^2,
\]
where $a_t\in\R^d$ is the $t$-th column vector.
Then, the learner receives $a_t$ as feedback.
We assume $\norm{a_t}_2^2\le1$ for every $t$, which ensures $\ell(Z,a_t)\in[0,1]$ for every feasible rank-$k$ projection subspace~$Z$.
For a prefix $Y=\{a_1,\ldots,a_m\}$, write $A_Y=[a_1,\ldots,a_m]\in\R^{d\times m}$, and define
\[
  \OPT(Y)
  =
  \min_{Z\in\R^{d\times k}:\,Z^\top Z=I_k}
  \sum_{a\in Y}\norm{a-ZZ^\top a}_2^2.
\]
We give the details of the low-rank approximation application based on ridge-leverage sampling, following the offline algorithm used by \citet{Dong2023-yh}.

\textbf{Poissonized ridge-leverage sampling procedure.\;}
Given a prefix matrix $A=[a_1,\ldots,a_m]\in\R^{d\times m}$ and an accuracy parameter $\varepsilon\in(0,1]$, the offline algorithm proceeds as follows.
\begin{enumerate}[leftmargin=*,itemsep=2pt]
  \item Compute the ridge leverage scores of the columns of $A$ with respect to the target rank $k$.
  Specifically, writing $A_k$ for a best rank-$k$ approximation to $A$ and
  $\lambda_k(A)=\norm{A-A_k}_F^2/k$, set
  \[
    \tau_i(A)
    =
    a_i^\top\prn*{AA^\top+\lambda_k(A)I_d}^{\dagger}a_i,
    \qquad i\in[m].
  \]
  If $\sum_i\tau_i(A)=0$, terminate and return a fixed deterministic fallback subspace.
  Otherwise define $p_i(A)=\tau_i(A)/\sum_j\tau_j(A)$.\footnote{
    The projection-cost-preserving sampling theorem of \citet{Cohen2017-th} uses probabilities
    proportional to these ridge leverage scores.  For the above ridge parameter, the total ridge
    leverage score satisfies $\sum_i \tau_i(A)=O(k)$; hence sampling from the normalized distribution
    $p_i(A)\propto \tau_i(A)$ only changes the required fixed-size sample count by constant factors.
  }
  This distribution depends on $A$ and $k$, but not on $\varepsilon$.
  \item Draw $N_\varepsilon\sim\mathrm{Pois}(\mu(\varepsilon))$, where
  \[
    \mu(\varepsilon)
    =
    c_{\rm lr}k\varepsilon^{-2}\log\prn*{\mathrm e+\frac{kT}{\varepsilon}}
  \]
  for a sufficiently large absolute constant $c_{\rm lr}\ge1$.
  Conditional on $N_\varepsilon=s$, sample $I_1,\ldots,I_s$ independently from $[m]$ according to $p(A)$.
  \item If $s=0$, return the fixed fallback subspace.  Otherwise, for each sampled index $I_h=i$, draw an independent perturbation
  \[
    \widetilde p_{h}
    \sim
    \mathrm{Unif}\prn*{[p_i(A),(1+\varepsilon/2)p_i(A)]}
  \]
  and insert the rescaled column $a_i/\sqrt{s\widetilde p_h}$ into matrix $B$.
  \item Return a minimizer of $Z\mapsto\norm{B-ZZ^\top B}_F^2$, equivalently the top-$k$ left singular subspace of $B$, using SVD and a fixed deterministic tie-breaking rule.
\end{enumerate}
The Poisson mean $\mu(\varepsilon)$ is so that $\mu(\varepsilon)/2$ dominates the fixed-size threshold in the projection-cost-preserving sampling theorem \citep{Cohen2017-th} with failure probability $O(1/T)$.  Conditional on $N_\varepsilon$ exceeding this threshold, the fixed-size theorem is applied with sample-size parameter $N_\varepsilon$ and gives, after absorbing the perturbation into constants, for all $Z\in\R^{d\times k}$ with $Z^\top Z=I_k$,
\[
  (1-\varepsilon)\norm{A-ZZ^\top A}_F^2
  \le
  \norm{B-ZZ^\top B}_F^2
  \le
  (1+\varepsilon)\norm{A-ZZ^\top A}_F^2.
\]
Thus, the deterministic SVD post-processing returns a $(1+\varepsilon)$-approximate $k$-dimensional subspace after adjusting constants;
as in \cref{rem:accuracy-translation}, the constant-factor translation from projection-cost-preservation accuracy to target offline accuracy is absorbed into the constants.
Also, the perturbation in Step~3 only changes each sampled column norm by a factor in $[1/\sqrt{1+\varepsilon/2},1]$ relative to the unperturbed rescaling, equivalently each squared projection-cost contribution by a factor in $[1/(1+\varepsilon/2),1]$; hence the two-sided projection-cost-preserving inequalities are preserved after the same constant rescaling of the target accuracy.

For a prefix $Y$, let $\mathcal R_\varepsilon(Y)$ be the law of the complete Poissonized ridge-leverage transcript constructed from $A_Y$, and let $\mathcal L_\varepsilon(Y)$ denote the output law after the deterministic SVD post-processing.  When the input matrix is denoted by $A$, we use the equivalent notation $\mathcal R_\varepsilon(A)$ and $\mathcal L_\varepsilon(A)$.
We also set
\[
  \varphi(\varepsilon)
  =
  C_{\rm lr}k\varepsilon^{-3}\log\prn*{\mathrm e+\frac{kT}{\varepsilon}}
  \asymp
  \frac{\mu(\varepsilon)}{\varepsilon},
\]
for a sufficiently large absolute constant $C_{\rm lr}\ge c_{\rm lr}$, where the additional $\varepsilon^{-1}$ factor comes from comparing the perturbed reweighting coordinates.\footnote{
  In \citet[Appendix~D]{Dong2023-yh}, the corresponding low-rank average-sensitivity
  bound is stated with an $\varepsilon^{-2}$ dependence. However, their calculation does not charge the contribution of the random reweighting variables to the total variation. Charging these variables gives the additional $\varepsilon^{-1}$ factor above. Therefore, the larger $\varepsilon^{-3}$ dependence comes from the correct analysis of the sampling law including the rescaling variables.
}

For later use, we present an average-sensitivity bound for fixed-size ridge-leverage sampling.
\begin{lemma}[Fixed-size ridge-leverage sensitivity]\label[lemma]{lem:fixed-size-ridge-leverage-sensitivity}
  Fix a prefix $Y$ with $\abs{Y}=t$ and a deterministic sample size $s\ge0$.  Let $P_{Y,s}^{\rm sel}$ be the law of the sequence of $s$ selected occurrences drawn from the normalized ridge-leverage distribution of $A_Y$.  For the input $Y\setminus a$, view the law $P_{Y\setminus a,s}^{\rm sel}$ on the same occurrence labels as $Y$ by assigning zero probability to the removed occurrence.  Then, we have
  \[
    \frac1t\sum_{a\in Y}
    \mathrm{TV}\prn*{
      P_{Y,s}^{\rm sel},
      P_{Y\setminus a,s}^{\rm sel}
    }
    =
    O\prn*{\frac{s}{t}}.
  \]
\end{lemma}
\begin{proof}
  This follows from the same one-draw total-variation calculation as in \citet[Lemma~5.2]{Dong2023-yh}: deleting a uniformly random column changes the normalized ridge-leverage distribution by $O(1/t)$ on average.  Applying \cref{lem:product-tv-general} to the product of the $s$ independent draws gives the stated $O(s/t)$ bound.
\end{proof}
We are ready to establish the approximation and average-sensitivity guarantees.
\begin{proposition}[Fixed-accuracy guarantee for low-rank approximation]\label[proposition]{prop:lra-fixed-sensitivity}
For every fixed $\varepsilon\in(0,1]$, the Poissonized perturbed ridge-leverage procedure returns a $(1+\varepsilon)$-approximate $k$-dimensional subspace in expectation, up to an additive $O(t/T)$ term on every prefix of size $t$, and its average sensitivity is at most $\varphi(\varepsilon)/(2t)$.
\end{proposition}
\begin{proof}
The approximation guarantee follows from
the projection-cost-preserving theorem of ridge-leverage sampling \citep{Cohen2017-th}.
By the Poisson lower-tail bound, the event that the realized count is below half its mean has probability exponentially small in $\mu(\varepsilon)$ and is absorbed into the $O(1/T)$ failure probability; on the complementary event we apply the fixed-size theorem with the realized sample size.  The failure contribution to expected loss is $O(t/T)$ because the losses are bounded by one.

For average sensitivity, fix a prefix $Y$ with $\abs{Y}=t$ and fix a deterministic
sample size $s$.
Let $P_{Y,s}^{\rm sel}$ be the law of the sequence of selected
occurrences $(I_1,\ldots,I_s)$ generated in Step~2, as in \cref{lem:fixed-size-ridge-leverage-sensitivity}.
Let $P_{Y,s}^{\rm wt}$ denote the law of the random variables $(I_h,\widetilde p_h)_{h=1}^s$ generated in Steps~2--3.  The matrix $B$ and the final SVD output are deterministic functions of these variables.
By \cref{lem:fixed-size-ridge-leverage-sensitivity}, we have
\[
  \frac1t\sum_{a\in Y}
  \mathrm{TV}\prn*{
    P_{Y,s}^{\rm sel},
    P_{Y\setminus a,s}^{\rm sel}
  }
  =
  O\prn*{\frac{s}{t}}.
\]
Passing from $P^{\rm sel}$ to $P^{\rm wt}$ adds the random variables
$\widetilde p_h$, which determine the rescaling $a_{I_h}/\sqrt{s\widetilde p_h}$.
For each draw, this is controlled by the same conditional-decomposition argument
as in \cref{subsec:weight-perturbation}: after conditioning on the selected
occurrence, the comparison of the two one-dimensional uniform perturbations costs
an additional factor $O(1/\varepsilon)$, at the scale of the same change in sampling
probabilities. Thus, it holds that
\[
  \frac1t\sum_{a\in Y}
  \mathrm{TV}\prn*{
    P_{Y,s}^{\rm wt},
    P_{Y\setminus a,s}^{\rm wt}
  }
  =
  O\prn*{\frac{s}{\varepsilon t}}.
\]
Since the deterministic SVD post-processing step cannot increase the total variation distance, averaging over $s=N_\varepsilon\sim\mathrm{Pois}(\mu(\varepsilon))$ gives
\[
  \frac1t\sum_{a\in Y}
  \mathrm{TV}\prn*{\mathcal L_\varepsilon(Y),\mathcal L_\varepsilon(Y\setminus a)}
  =
  O\prn*{\frac{\mu(\varepsilon)}{\varepsilon t}}
  \lesssim
  \frac{\varphi(\varepsilon)}{2t},
\]
obtaining the desired average sensitivity bound.
\end{proof}
We then establish the parameter-change stability guarantee.
\begin{proposition}[Parameter-change stability for low-rank approximation]\label[proposition]{prop:lr-same-input-pc}
For $0<\varepsilon\le\eta\le1$, it holds that
\[
  \mathrm{TV}\prn*{\mathcal L_\varepsilon(A),\mathcal L_\eta(A)}
  \le
  \mathrm{TV}\prn*{\mathcal R_\varepsilon(A),\mathcal R_\eta(A)}
  =
  O\prn*{\mu(\varepsilon)\frac{\eta-\varepsilon}{\varepsilon}}.
\]
Consequently, the Poissonized ridge-leverage procedure satisfies \cref{assump:adaptive-pc-stability} with $C_{\rm pc}=O(1)$.
\end{proposition}
\begin{proof}
For a fixed input matrix $A$, the ridge-leverage distribution $p(A)$ is independent of the accuracy parameter; only the Poisson mean and the perturbation intervals depend on the parameter.
The mean function $\mu$ has the form covered by \eqref{eq:power-log-mean-lipschitz} with $a=2$ and $B=kT$, hence
\[
  \abs*{\mu(\varepsilon)-\mu(\eta)}
  =
  O\prn*{\mu(\varepsilon)\frac{\eta-\varepsilon}{\varepsilon}}.
\]
By \cref{lem:poissonized-transcript-tv}, the change in the Poisson count contributes this amount.

Conditional on a common count $s$ and common sampled indices, the remaining accuracy parameter dependence is in the independent perturbations $\widetilde p_h$.
For each sampled index $i_h$, applying \cref{lem:uniform-tv-changing-width} to the two laws of
$\widetilde p_h\sim\mathrm{Unif}([p_{i_h}(A),(1+\alpha/2)p_{i_h}(A)])$, with
$\alpha=\varepsilon,\eta$, yields an
$O((\eta-\varepsilon)/\varepsilon)$ bound on their total variation distance.
Since the rescaled column is a deterministic measurable function of $\widetilde p_h$, the same bound holds after rescaling by the data processing inequality.
Applying \cref{lem:product-tv-general} over the $s$ sampled indices, the total variation distance between the conditional laws of the remaining reweighting part of the transcript, given the common count and sampled indices, is
\[
  O\prn*{s\frac{\eta-\varepsilon}{\varepsilon}}.
\]
Taking expectation over $s\sim\mathrm{Pois}(\mu(\varepsilon))$, adding the Poisson-count contribution above, and using the data processing inequality for the deterministic SVD post-processing yield
\[
  \mathrm{TV}\prn*{\mathcal L_\varepsilon(A),\mathcal L_\eta(A)}
  \le
  \mathrm{TV}\prn*{\mathcal R_\varepsilon(A),\mathcal R_\eta(A)}
  =
  O\prn*{\mu(\varepsilon)\frac{\eta-\varepsilon}{\varepsilon}}.
\]

For \cref{assump:adaptive-pc-stability}, fix a dataset $X$ and $x\in X$, let $Y=X\setminus x$, and set $\varepsilon=\varepsilon_X$ and $\eta=\varepsilon_Y$.
Since the losses are bounded and non-negative, we have $0\le\OPT(X)-\OPT(Y)\le1$ and $\varepsilon_X\le\varepsilon_Y$.
By using the above total variation bound for fixed $Y$ and the chosen constant $C_{\rm lr}$, which makes $\mu(\varepsilon)/\varepsilon=O(\varphi(\varepsilon))$, we obtain
\[
  \mathrm{TV}\prn*{\mathcal L_{\varepsilon_X}(Y),\mathcal L_{\varepsilon_Y}(Y)}
  =
  O\prn*{\varphi(\varepsilon_X)(\varepsilon_Y-\varepsilon_X)}.
\]
By \cref{lem:adaptive-rule-calculus}, the right-hand side is $O(\varepsilon_X^2(\OPT(X)-\OPT(Y)))$;
since $\varepsilon_X\le1$, this is at most $O(\varepsilon_X(\OPT(X)-\OPT(Y)))$.  Averaging over $x\in X$ yields the condition with $C_{\rm pc}=O(1)$.
\end{proof}

\textbf{Regret guarantee.\;}
By \cref{prop:lra-fixed-sensitivity,prop:lr-same-input-pc}, the low-rank approximation offline algorithm family satisfies \cref{assump:approx-avesen} with $\varphi(\varepsilon)=C_{\rm lr}k\varepsilon^{-3}\log(\mathrm e+kT/\varepsilon)$ and \cref{assump:adaptive-pc-stability} with $C_{\rm pc}=O(1)$.
Thus, applying \cref{cor:simple-form} with $\lambda=H_T^{-(q+1)/q}$, $q=3$, $C_1=C_{\rm lr}k$, $C_2=kT$, and $C_{\rm pc}=O(1)$ gives
\[
  \Reg_T
  =
  \tilde O\prn*{
  k+
  k^{1/4}\prn*{1+\OPT_T^{3/4}}
  }.
\]

\section{\texorpdfstring{$\Omega(\sqrt{n\OPT_T})$ lower bound for online submodular function minimization}{Ω(√nOPT\_T) lower bound for online submodular function minimization}}\label[appendix]{asec:proof-sqrt-lower-bound}
We present the formal statement and proof of the $\Omega(\sqrt{n\OPT_T})$ lower bound for online submodular function minimization mentioned in \cref{sec:conclusion}.
Below, for a fixed multiset $M$ of $T$ submodular loss functions, write
\[
  \Reg_T(M)
  \coloneqq
  \E_{\pi,\mathsf{Alg}}\!\left[\sum_{t=1}^T \ell_{\pi(t)}(S_t)\right]
  -
  \OPT(M),
  \qquad
  \text{where}
  \qquad
  \OPT(M)\coloneqq \min_{S\subseteq V}\sum_{\ell_t\in M}\ell_t(S).
\]
\begin{restatable}{theorem}{sqrtlowerboundtheorem}\label{thm:sqrt-lower-bound}
  There exists a universal constant $c>0$ such that, for every randomized online learner for online submodular function minimization on a ground set of size $n\le T$, there is a fixed multiset~$M$ of $T$ submodular losses such that, when $M$ is revealed in a uniformly random order, it holds that\looseness=-1
\[
  \Reg_T(M)\ge c\sqrt{n\OPT(M)}.
\]
\end{restatable}
\wraprestatablewithlabelrestore{sqrtlowerboundtheorem}
\begin{proof}
We use the following instance construction of \citet[Theorem~14]{Hazan2012-hc} for establishing the $\Omega(\sqrt{nT})$ lower bound for online submodular function minimization.

Assume first that $n$ divides $T$, and put $m=T/n$; the general case follows by using $\lfloor T/n\rfloor$ and $\lceil T/n\rceil$ occurrences for each coordinate, which changes only constants.
Let $V$ be a ground set of size~$n \le T$.
For each coordinate $i\in[n]$ and each copy $j\in[m]$, draw an independent Rademacher sign $\sigma_{i,j}\in\set*{-1,+1}$ and define the submodular loss
\[
  \ell_{i,j}(S)
  =
  \begin{cases}
    \tfrac12(1+\sigma_{i,j}), & i\in S,\\
    \tfrac12(1-\sigma_{i,j}), & i\notin S.
  \end{cases}
\]
Let $M_\sigma$ be the balanced multiset containing these $m$ losses for every coordinate, and reveal $M_\sigma$ in a uniformly random order.
For any online learner, the sign of the next revealed loss is independent of the learner's current choice, and therefore
\[
  \E_{\sigma,\pi,\mathsf{Alg}}\brc*{\sum_{t=1}^T\ell_t(S_t)}
  =
  \frac{T}{2},
\]
where $\pi$ denotes the random order and $\mathsf{Alg}$ denotes the learner's internal randomness.
On the other hand, for each fixed sign realization $\sigma$, by including $i$ in $S$ if and only if $\sum_{j=1}^m\sigma_{i,j}\le0$, we can achieve
\[
  \OPT(M_\sigma)
  =
  \frac{1}{2}\prn*{
    T - \sum_{i=1}^n \abs*{\sum_{j=1}^m \sigma_{i,j}}
  },
\]
and thus $\OPT(M_\sigma) \le T/2$ always holds.
By Khinchine's inequality, we have
\[
  \E_{\sigma}\brc*{\Reg_T(M_\sigma)}
  =
  \E_{\sigma,\pi,\mathsf{Alg}}\brc*{\sum_{t=1}^T \ell_t(S_t)}
  -
  \E_{\sigma}\brc*{\OPT(M_\sigma)}
  =
  \frac{1}{2} \sum_{i=1}^n \E_{\sigma}\brc*{\abs*{\sum_{j=1}^m \sigma_{i,j}}}
  =
  \Omega\prn*{
    \sqrt{nT}
  }.
\]
Since $\sqrt{T} \ge \sqrt{2\OPT(M_\sigma)}$ holds pointwise, we have $\sqrt{T} \ge \E_\sigma[\sqrt{2\OPT(M_\sigma)}]$, hence
\[
  \E_{\sigma}\brc*{\Reg_T(M_\sigma)}
  =
  \Omega\prn*{\E_\sigma\brc*{\sqrt{n\OPT(M_\sigma)}}}.
\]
Let $c_0>0$ be the universal constant implicit in the preceding lower bound, and choose $c<c_0$. If every realization of the balanced multiset satisfied
\[
  \Reg_T(M_\sigma)<c\sqrt{n\OPT(M_\sigma)},
\]
then averaging over $\sigma$ would contradict the preceding lower bound.
Hence, for every learner, some fixed multiset $M_\sigma$ satisfies
\[
  \Reg_T(M_\sigma)\ge c\sqrt{n\OPT(M_\sigma)}.
\]
Therefore, for every learner, some fixed realization $M_\sigma$ satisfies the claimed lower bound.
\end{proof}
\end{document}